\documentclass[journal]{IEEEtran}
\usepackage{amsmath,amsfonts}
\usepackage{algorithmic}
\usepackage{algorithm}
\usepackage{array}
\usepackage[caption=false,font=normalsize,labelfont=sf,textfont=sf]{subfig}
\usepackage{textcomp}
\usepackage{stfloats}
\usepackage{url}
\usepackage{verbatim}
\usepackage{graphicx}
\graphicspath{{figures/}}
\usepackage{cite}
\usepackage{amssymb}
\usepackage{booktabs}
\usepackage{multirow}
\usepackage{enumitem}

\usepackage[table,dvipsnames]{xcolor}
\usepackage{placeins}
\usepackage{amsthm} 
\theoremstyle{plain}
\newtheorem{theorem}{Theorem}
\newtheorem{remark}{Remark}
\newtheorem{lemma}{Lemma}
\newtheorem{definition}{Definition}
\usepackage{soul}

\begin{document}

\title{Multimodal Functional Maximum Correlation \\for Emotion Recognition}

\author{Deyang Zheng, Tianyi~Zhang*~\IEEEmembership{Senior Member,~IEEE,}
        Wenming Zheng~\IEEEmembership{Senior Member,~IEEE,}%
        \\
        and Shujian Yu*~\IEEEmembership{Member,~IEEE}%

\thanks{Tianyi Zhang and Wenming Zheng are with the Key Laboratory of Child Development and Learning Science (Ministry of Education), School of Biological Sciences and Medical Engineering, Southeast University, Nanjing, China. Emails: \{t.zhang, wenming\_zheng\}@seu.edu.cn}%

\thanks{D. Zheng is with the Department of Artificial Intelligence, Westlake University, Hangzhou, China (e-mail: zhengdeyang@westlake.edu.cn).}%

\thanks{S. Yu is with the Department of Artificial Intelligence, Vrije Universiteit Amsterdam, Amsterdam, The Netherlands. (e-mail: s.yu3@vu.nl).}%

\thanks{*Corresponding authors.}
}
\markboth{IEEE TRANSACTIONS ON AFFECTIVE COMPUTING}%
{}

\maketitle

\begin{abstract}
Emotional states manifest as coordinated and heterogeneous physiological responses across central and autonomic systems, posing a fundamental challenge for multimodal representation learning in affective computing.
Learning such joint dynamics is further complicated by scarce and subjective affective annotations, motivating the use of self-supervised learning (SSL).
However, most existing SSL approaches rely on pairwise alignment objectives, which are insufficient to characterize the joint dependencies among more than two modalities and fail to capture higher-order interactions arising from coordinated brain–autonomic responses.
To move beyond this limitation, we propose Multimodal Functional Maximum Correlation (MFMC), a principled SSL framework that maximizes higher-order multimodal dependence via a Dual Total Correlation (DTC) objective.
By deriving a tight sandwich bound and optimizing it with a functional maximum correlation analysis (FMCA)-based trace surrogate, MFMC directly captures joint multimodal interactions without resorting to pairwise contrastive losses.
Experiments on three public affective computing benchmarks demonstrate that MFMC consistently achieves state-of-the-art or competitive performance under both subject-dependent and subject-independent protocols, highlighting its robustness to inter-subject variability. In particular, MFMC achieves substantial gains on CEAP-360VR, improving subject-dependent accuracy from 78.9\% to 86.8\% and subject-independent accuracy from 27.5\% to 33.1\% using the EDA signal alone, and remains highly competitive within 0.8 percentage points of the best-performing method on the most challenging EEG subject-independent split of MAHNOB-HCI. Our code is available at \url{https://github.com/DY9910/MFMC}.



\end{abstract}
 

\section{Introduction}
\IEEEPARstart{R}{ecognizing} emotions is a critical yet challenging task in affective brain–computer interfaces, enabling systems to perceive and adapt to users’ internal states for neural rehabilitation, mental health, and more sympathetic human–machine interaction \cite{song2018eeg,garcia2017efficacy,gumuslu2020emotion,liu2025divergent}. Emotion-recognition methods are commonly grouped into two categories. 
The first relies on non-physiological cues such as facial images \cite{anderson2006realtime,huang2016spontaneous,huang2015microexpr,xiaohua2017discriminative,liu2016mdmof}, body gestures \cite{yan2014integrating}, and speech/voice \cite{ang2002prosody}. The second leverages physiological signals, including electroencephalography (EEG) \cite{zheng2017gscca}, electromyography (EMG) \cite{cheng2008emg}, and electrocardiogram (ECG) \cite{agrafioti2012ecg}.
Since facial and vocal behavior can be consciously modulated, whereas physiological responses are largely involuntary, physiology provides a more reliable basis for affect inference \cite{lopez2023hypercomplex}. 
Moreover, physiological signals capture both synchronous and asynchronous dynamics across multiple layers of the nervous system, including the central, peripheral, and autonomic subsystems, which collectively encode and regulate affective states. Jointly modeling multiple modalities (e.g., EEG together with peripheral measures such as ECG, EDA, or BVP) can therefore provide complementary views of the same emotional episode.
As a result, effective affective computing critically depends on jointly modeling multiple physiological modalities in order to mitigate modality-specific noise and capture the complex, coordinated brain–body dynamics that underlie emotional experience.


\IEEEpubidadjcol
Traditional physiological emotion recognition systems are predominantly supervised and therefore rely on large collections of labeled affective data. In practice, ground truth labels are typically obtained through self-report instruments such as the Self-Assessment Manikin or dimensional scales based on the circumplex and PAD frameworks~\cite{Bradley1994SAM,Russell1980Circumplex,Posner2005Circumplex,Mehrabian1996PAD}. These ratings are labor-intensive to collect, require carefully controlled experimental protocols, and remain inherently subjective and context-dependent~\cite{Calvo2010TACReview,deap,ceap,MirandaCorrea2018AMIGOS}. Participants may further modulate or reinterpret their emotional reports across trials, sessions, and studies, introducing label noise and reducing the comparability of datasets. As a result, supervised models are often trained on relatively small and idiosyncratic corpora with limited standardization, which leads to poor generalization across datasets, recording setups, and populations~\cite{Saganowski2022TACWear,Apicella2024Neurocomputing,Ma2024PeerJCS}.

A further limitation lies in how existing models exploit multimodal physiological dynamics. Although emotional episodes can elicit coordinated responses across central and peripheral systems~\cite{Kreibig2010BiolPsych}, many approaches still rely on unimodal encoders~\cite{Liu2024eeg} or simple feature-level fusion strategies, such as concatenation~\cite{ramadan2024mm} or shallow late fusion~\cite{he2020advances,Saganowski2022TACWear}. While recent studies have explored more sophisticated multimodal architectures using attention mechanisms~\cite{lopez2023hypercomplex} or hypercomplex representations~\cite{lopez2024,zhang2019}, these methods often operate on pre-engineered features and primarily model pairwise correlations~\cite{yang2021b}. As a result, they fail to explicitly capture the synchronous and asynchronous dependencies between central neural activity and peripheral autonomic responses that characterize affective processing~\cite{pillalamarri2025multimodal}. Consequently, much of the higher-order structure inherent in multichannel physiological data remains underutilized.

These challenges collectively motivate learning frameworks that reduce the reliance on explicit labels and can directly capture the intrinsic relationships among multiple heterogeneous modalities. Contrastive self-supervised learning (SSL) has therefore emerged as a promising direction for physiological representation learning. By treating different views of the same physiological event as positives, contrastive objectives encourage invariance to nuisance factors such as subject identity and session-specific noise~\cite{zbontar2021barlow,bardes2022vicreg}. In multimodal settings, contrasting matched and mismatched central-peripheral segments further promotes cross-modal alignment~\cite{cui2025physiosync}.


Despite these advances, most existing SSL frameworks for physiological computing remain fundamentally limited by their reliance on pairwise contrastive formulations. General contrastive methods such as SimCLR~\cite{chen2020simple}, CPC~\cite{oord2018representation}, and CLIP~\cite{radford2021learning}, as well as their multimodal extensions including VATT~\cite{akbari2021vatt}, ImageBind~\cite{girdhar2023imagebind}, Symile~\cite{saporta2024contrasting}, and Gramian-based models~\cite{cicchetti2025gramian}, were primarily developed for vision-language settings and typically assume dual-modality alignment. When applied to physiological signals, these frameworks are commonly instantiated as pairwise contrastive objectives between two modalities or between original and augmented views of a single modality. As a consequence, the learned representations are biased toward lower-order dependencies (e.g., between EEG and ECG segments), while overlooking both the intrinsic multichannel structure within each modality and the higher-order interactions that emerge when three or more physiological modalities jointly encode affective dynamics~\cite{del2023applications,wang2023self,yang2023self,zhang2024self}.

Recent physiological SSL approaches, such as PhysioSync~\cite{cui2025physiosync} and cross-modal ECG--EEG alignment~\cite{wu2025crossmodal}, demonstrate that contrastive and generative self-supervision can substantially improve emotion recognition performance. However, these methods either operate strictly in a pairwise setting or treat each modality as a single aggregated unit, without explicitly modeling how informative patterns are distributed across heterogeneous channels and modalities over time. 
Consequently, current SSL pipelines remain poorly aligned with the structured, multilevel dependence that characterizes multimodal physiological responses to emotion, leaving substantial room for approaches that can directly capture higher-order correlations across modalities.

Motivated by this observation, we propose a principled multimodal SSL framework that goes beyond pairwise alignment and explicitly maximizes the total dependence among multiple physiological modalities. Our approach does not rely on positive-negative sample construction or handcrafted data augmentations, and is designed to capture higher-order interactions that emerge when emotional stimuli jointly modulate central and peripheral physiological systems.

To summarize, our main contributions include:
\begin{itemize}
\item We introduce the first SSL framework for physiological emotion recognition that explicitly models higher-order multimodal dependence beyond pairwise contrastive objectives. By grounding multimodal alignment in dual total correlation (DTC)~\cite{sun1975linear,yu2021measuring}, our framework naturally generalizes from tri-modal to arbitrarily many physiological modalities.
\item We derive a principled and tractable optimization objective based on functional maximum correlation analysis (FMCA)~\cite{hu2022normalized}, which enables stable estimation of joint multimodal dependence and avoids the numerical instability associated with eigenvalue decomposition and lower-bound-based contrastive objectives.
\item Extensive experiments on three public benchmarks demonstrate that our approach consistently achieves state-of-the-art or competitive performance under both subject-dependent and subject-independent protocols, highlighting its robustness to inter-subject variability and its effectiveness in modeling complex physiological dynamics.
\end{itemize}

\section{Related Work}
In this section, we first provide a brief review of existing supervised learning and contrastive SSL approaches for emotion recognition. We then discuss prior efforts in the vision and language domains that attempt to extend bi-modal contrastive SSL to multiple modalities, although none of these have yet been evaluated on physiological signals.

\subsection{EEG Feature Extraction and Supervised Deep Learning for Emotion Recognition}


EEG is the most widely used modality for physiological emotion recognition due to its ability to capture neural dynamics associated with affective processing~\cite{wang2023self,zhang2024self}. Traditional EEG-based emotion recognition pipelines typically rely on handcrafted time-domain~\cite{Hjorth1970,Petrantonakis2010HigherOrder} or frequency-domain~\cite{Shi2013DE,Li2009GammaBand,Davidson2004FrontalAsymmetry} features combined with supervised classifiers~\cite{song2018eeg}. While effective in controlled settings, such approaches depend heavily on expert-designed features and often struggle to generalize across subjects, sessions, and datasets.

Building on handcrafted features, early physiological emotion recognition pipelines typically employ classifiers such as support vector machines, $k$-nearest neighbors, or shallow neural networks~\cite{Calvo2010TACReview,he2020advances}. More recent studies have shifted toward deep representation learning, using convolutional, recurrent, or graph-based encoders to learn features end-to-end from raw or lightly preprocessed signals~\cite{keelawat2021,iwana2021,Liu2024eeg,cheng2024emotion}. In multimodal settings, these architectures have been extended to incorporate peripheral signals (e.g., EDA, BVP, ECG, EOG) through feature concatenation, attention-based fusion, or hypercomplex representations~\cite{lopez2023hypercomplex,lopez2024,pillalamarri2025multimodal}. While these approaches improve flexibility and performance, they typically focus on pairwise fusion mechanisms and do not explicitly model higher-order dependencies across multiple physiological modalities.

\subsection{Contrastive SSL for Physiological Signals}\label{sec:2.2}
Self-supervised learning (SSL) alleviates the need for explicit labels by leveraging surrogate objectives that exploit the intrinsic structure in the data~\cite{jaiswal2020survey}. 
For time series and physiological signals, contrastive SSL has become a dominant paradigm~\cite{del2023applications,wang2023self,yang2023self,zhang2024self} enabling applications such as fMRI-based mental-disorder diagnosis~\cite{gryshchuk2025contrastive}, EEG-based sleep staging~\cite{yang2023self}, and affect recognition~\cite{del2023applications,wang2023self}. 

Current approaches often follow established SSL pipelines developed for vision tasks, such as SimCLR~\cite{chen2020simple} and VICReg~\cite{bardes2022vicreg}, and optimize contrastive losses like the InfoNCE~\cite{oord2018representation}, defined as:
\begin{equation}
\mathcal{L}_{\text{InfoNCE}} = -\log \frac{\exp(\text{sim}(z_i, z_j)/\tau)}{\sum_{k=1}^{N} \exp(\text{sim}(z_i, z_k)/\tau)},
\end{equation}
where $z_i$ and $z_j$ are positive sample embeddings, $z_k$ denotes embeddings of negative samples in the batch of size $N$, $\text{sim}(\cdot, \cdot)$ is a similarity function (e.g., cosine similarity), and $\tau$ is a temperature parameter.


Unimodal contrastive SSL on physiological time series typically instantiates one of three families of collapse-avoidance principles: (i) contrastive objectives with negatives (e.g., SimCLR), which maximize agreement between two augmented views of the same sample \cite{chen2020simple}; (ii) redundancy-reduction via cross-correlation matching (Barlow Twins) \cite{zbontar2021barlow}; and (iii) variance–invariance–covariance regularization without negatives (VICReg) \cite{bardes2022vicreg}. In EEG and related biosignals, these objectives are adapted with modality-aware augmentations or pretext design to respect temporal and multichannel structure, yielding strong representations for downstream tasks such as sleep staging and affect recognition \cite{del2023applications,yang2023self,wang2023self}. For instance, GANSER couples adversarial augmentation with self-supervision to improve EEG emotion recognition \cite{zhang2022ganser}, while VICReg-style training has been shown effective for EMG representation learning \cite{raghu2025self}.

In physiological emotion recognition, contrastive SSL is particularly appealing because multimodal recordings naturally provide intrinsic supervision signals~\cite{jaiswal2020survey}. Temporal correspondences within a modality and synchronous correspondences across modalities (e.g., EEG-ECG or EEG-EDA) can be leveraged to define positive pairs without requiring explicit affect labels~\cite{chen2020simple,oord2018representation,radford2021learning}. 


Beyond label efficiency, SSL enables representation learning directly from raw or lightly processed physiological signals, reducing reliance on handcrafted feature engineering~\cite{del2023applications,wang2023self,yang2023self,zhang2024self}. Pretraining on large-scale unlabeled recordings encourages the learned representations to capture structures that are stable across subjects, sessions, and acquisition conditions, thereby improving robustness to noise and inter-subject variability~\cite{Ma2024PeerJCS,Apicella2024Neurocomputing}. Moreover, contrastive objectives can promote invariance to nuisance factors while encouraging alignment between central and peripheral modalities, making SSL a natural fit for modeling multimodal physiological dynamics during emotional processing~\cite{zbontar2021barlow,bardes2022vicreg,cui2025physiosync}.

Recent studies in physiological emotion recognition provide concrete evidence of this potential: Wu et al.~\cite{wu2025crossmodal} align ECG with EEG derived features using sequence and patch level InfoNCE and report binary arousal/valence accuracies of 0.892/0.879 on DREAMER~\cite{dreamer} and 0.849/0.834 on AMIGOS~\cite{MirandaCorrea2018AMIGOS}, exemplifying a typical bi-modal, pairwise setup. Similarly, Zhang et al. propose GANSER~\cite{zhang2023ganser}, a GAN-based self-supervised data augmentation framework for EEG-based emotion recognition, which attains state-of-the-art performance of 93.52\%/94.21\% accuracy on the binary valence/arousal classification tasks of the DEAP dataset~\cite{deap}. These approaches demonstrate that contrastive and generative self-supervision can exploit physiological recordings as a rich supervision source and substantially reduce annotation requirements. However, they still focus on pairwise modality alignment or single-modality augmentation and do not explicitly model the intrinsic multichannel dependencies among heterogeneous signals.

\subsection{Multimodal Contrastive Learning Beyond Two Modalities}
Multimodal contrastive learning has achieved remarkable success in vision-language domains. For a batch of paired samples $(x_i, y_i)$ from two modalities, where $x_i$ and $y_i$ are encoded, respectively, by networks $f_\theta$ and $g_\phi$, methods like CLIP~\cite{radford2021learning} aim to maximize similarity between corresponding pairs (positive samples) and minimize similarity with all other pairs (negative samples). The CLIP objective can be expressed as a symmetric InfoNCE-like loss. When $x$ is treated as the anchor modality, the loss is:
\begin{equation}
\mathcal{L}_{\text{CLIP}}^{x \rightarrow y} = -\frac{1}{N} \sum_{i=1}^{N} \log \frac{\exp(\text{sim}(f_\theta(x_i), g_\phi(y_i))/\tau)}{\sum_{j=1}^{N} \exp(\text{sim}(f_\theta(x_i), g_\phi(y_j))/\tau)},
\end{equation}
and similarly for the reverse direction $y \rightarrow x$. The final loss is the average of the two directions:
\begin{equation}
\mathcal{L}_{\text{CLIP}} = \frac{1}{2} \left( \mathcal{L}_{\text{CLIP}}^{x \rightarrow y} + \mathcal{L}_{\text{CLIP}}^{y \rightarrow x} \right).
\end{equation}

Although highly effective for bi-modal alignment, this paradigm fundamentally decomposes multimodal learning into collections of pairwise objectives. In affective computing, CLIP-style losses have been applied to align EEG with peripheral signals such as ECG or EOG~\cite{cai2023emotion,lopez2023hypercomplex,lopez2024,liu2025multi}, and similar formulations have been extended to additional modalities in vision-language settings, including tri-modal video-audio-text models (VATT)~\cite{akbari2021vatt} and broad cross-modal embedding spaces (ImageBind)~\cite{girdhar2023imagebind}.

In practice, these approaches typically optimize sums of pairwise losses across modality pairs, implicitly assuming that higher-order multimodal structure can be recovered from pairwise alignment alone~\cite{akbari2021vatt,girdhar2023imagebind}. For example, in vision, language, and other multimedia domains, extending CLIP to three modalities is commonly implemented by summing all pairwise CLIP objectives~\cite{akbari2021vatt,girdhar2023imagebind,ruan2023accommodating}, leading to:
\begin{equation}
\mathcal{L}^{(x,y,z)}_{\text{CLIP}} =
\mathcal{L}^{(x,y)}_{\text{CLIP}} +
\mathcal{L}^{(y,z)}_{\text{CLIP}} +
\mathcal{L}^{(x,z)}_{\text{CLIP}} .
\end{equation}


Such pairwise decompositions miss higher-order dependencies that emerge only when multiple modalities are considered jointly. From an information-theoretic perspective, this limitation is fundamental: pairwise criteria may fail to capture synergistic interactions that arise only when three or more modalities are jointly observed~\cite{mcgill1954multivariate,watanabe1960information,timme2014synergy}. This insight has motivated recent efforts toward multi-way alignment objectives, including SymILE~\cite{saporta2024contrasting}, Gramian volume-based alignment~\cite{cicchetti2025gramian}, and analyses of what to align in multimodal contrastive learning~\cite{dufumier2025what}. Nevertheless, existing approaches still lack a principled and tractable objective for modeling higher-order dependence across heterogeneous modalities. To our knowledge, no prior work has systematically addressed the alignment of more than two physiological modalities for emotion recognition.

To address these limitations, we introduce a principled alternative that explicitly models higher-order multimodal dependence by analyzing joint interactions among multiple modalities through the lens of dual total correlation (DTC). To make this formulation tractable in practice, we further leverage functional maximum correlation analysis (FMCA)~\cite{hu2022normalized} to enable efficient and scalable estimation of such joint dependence.

\section{Multimodal FMCA for Emotion Recognition}

We consider the setting where \( M \) physiological modalities \( X_1, X_2, \dots, X_M \) are available, and the goal is to learn generalizable representations that reside in a shared latent space and can effectively transfer to downstream tasks. In emotion recognition, typical physiological modalities include EEG, ECG, and peripheral signals such as skin temperature.

As discussed in Section~\ref{sec:2.2}, most existing contrastive SSL objectives rely on mutual information (MI) estimation. From an information-theoretic perspective, the widely used InfoNCE-based estimator suffers from inherent limitations, including its lower-bound nature~\cite{oord2018representation} and unfavorable sample complexity in high-dimensional settings~\cite{mcallester2020formal}. To address these issues, our framework introduces a new MI estimator derived from a density-ratio decomposition perspective. For readers who may be unfamiliar with information theory, we first review the FMCA, which serves as the mathematical foundation of our proposed framework.

\subsection{Density Ratio Decomposition and FMCA}\label{sec:FMCA_background}


Given any two random processes \(X\) and \(Y\) with joint distribution \(p(X,Y)\) and marginal product \(p(X)p(Y)\), their statistical dependence can be characterized through an orthonormal decomposition of the density ratio~\cite{huang21universal,hu2022normalized}:
\begin{equation}\label{eq:density_ratio}
\rho := \frac{p(X,Y)}{p(X)p(Y)} = \sum_{k=1}^{\infty} \sqrt{\sigma_k} \, \phi_k(X) \psi_k(Y),
\end{equation}
where \(\{\phi_k\}\) and \(\{\psi_k\}\) form orthonormal systems in
\(L^2(p(X))\) and \(L^2(p(Y))\), respectively, satisfying:
\begin{equation}
\mathbb{E}_X \left[ \phi_i(X)\phi_{j}(X) \right] = 
\mathbb{E}_Y \left[ \psi_i(Y)\psi_{j}(Y) \right] = 
\begin{cases}
1, & i = j \\
0, & i \ne j
\end{cases}
\end{equation}

The coefficients $\sigma_1 \ge \sigma_2 \ge \cdots \ge 0$, with $\sigma_k \in [0,1)$, can be interpreted as nonlinear correlation strengths between $X$ and $Y$. The largest coefficient satisfies $\sigma_1 = 1$. 
If $X$ and $Y$ are statistically independent, then $\sigma_k = 0$ for all $k \ge 2$; conversely, larger values of $\sigma_k$ indicate stronger dependence captured at increasingly higher orders.

This spectral view motivates defining a total statistical dependence (TSD) measure as an additive functional of \(\{\sigma_k\}\). FMCA quantifies this dependence as:
\begin{equation}
T := -\sum_{i=1}^{\infty} \log(1 - \sigma_i),
\end{equation}
which increases monotonically with each \(\sigma_i\) and reduces to zero only when \(X\) and \(Y\) are independent. The formulation requires no parametric assumptions and naturally generalizes classical linear correlation.

Intuitively, each non-trivial coefficient $\sigma_i$ can be regarded as one nonlinear dependence mode shared by two random variables. If two physiological modalities are independent, these non-trivial shared modes vanish; if they contain coordinated information, more modes become active and their corresponding coefficients increase. TSD therefore measures dependence by accumulating multiple nonlinear shared modes, rather than relying on a single linear correlation coefficient. This is particularly suitable for physiological emotion recognition, where cross-modal relationships are often nonlinear and distributed across several latent physiological factors.

\subsubsection{Neural Networks Implementation}
Since true density functions are unavailable in practice, FMCA optimizes a neural surrogate of \(T\) using paired projection networks \(f_\theta : \mathcal{X}\!\to\!\mathbb{R}^K\) and \(g_\phi : \mathcal{Y}\!\to\!\mathbb{R}^K\).  Define the marginal autocorrelation and cross-correlation matrices:
\begin{align}
\mathbf{R}_X &= \mathbb{E}[f_\theta(X) f_\theta(X)^\top], \quad
\mathbf{R}_Y = \mathbb{E}[g_\phi(Y) g_\phi(Y)^\top], \\
\mathbf{P}_{XY} &= \mathbb{E}[f_\theta(X) g_\phi(Y)^\top],\quad
\mathbf{R}_{XY} =
\begin{bmatrix}
\mathbf{R}_X & \mathbf{P}_{XY} \\
\mathbf{P}^\top_{XY} & \mathbf{R}_Y
\end{bmatrix}.
\end{align}

FMCA minimizes the following log-determinant objective:
\begin{equation}\label{eq:FMCA_obj}
\min_{\theta,\phi} \; r(\theta,\phi)
:= \log\det\mathbf{R}_{XY}
   - \log\det\mathbf{R}_X
   - \log\det\mathbf{R}_Y.
\end{equation}

After optimization, \( f_\theta \) and \( g_\phi \) approximate the top eigenfunctions of the density ratio \( \rho \). Applying singular value decomposition to \( \mathbf{R}_{XY} = U S^{1/2} V \), where \( S = \mathrm{diag}(\lambda_1, \ldots, \lambda_K) \) approximates the top-\(K\) eigenvalues of \( \rho \), we can estimate \( \rho \) as \( \hat{\rho} = f_\theta^\top S^{1/2} g_\phi \).

We note that the FMCA objective resembles the mutual information formula for jointly Gaussian variables:
\begin{equation}
I(X; Y) = \frac{1}{2} \left( \log \det R_X + \log \det R_Y - \log \det \Sigma \right),
\end{equation}
where \( \Sigma = \begin{bmatrix} R_X & R_{XY} \\ R_{YX} & R_Y \end{bmatrix} \) is the sample joint-covariance matrix. This resemblance partially motivates Eq.~(\ref{eq:FMCA_obj}), but FMCA remains fully
non-parametric.

\begin{lemma}[\cite{hu2022normalized}]\label{lemma_1}
Minimizing Eq.~(\ref{eq:FMCA_obj}) yields
\(r^*_{\mathcal{L}}(\theta,\phi)=\sum_{i=1}^K \log(1-\sigma_i)\),
where \(\sigma_i\in[0,1)\) are the top \(K\) eigenvalues in the
decomposition \eqref{eq:density_ratio}.
\end{lemma}

Thus, minimizing Eq.~(\ref{eq:FMCA_obj}) is equivalent to
maximizing a truncated version of \(T\), and hence the nonlinear
dependence between \(X\) and \(Y\).

While theoretically principled, the log-determinant objective
requires repeated matrix inversion and eigenvalue computation,
which may become numerically unstable during minibatch
training. Small ridge terms \(\varepsilon \mathbf{I}\) are often added
for stability \cite{hu2022normalized}, motivating the development
of an alternative formulation that remains faithful to the
spectral interpretation yet is more efficient and robust for
large-scale optimization.

\subsection{Aligning Multimodal Representation with DTC}


Motivated by CLIP, a natural approach is to maximize the total dependence among all modalities. 
From an information-theoretic view, total dependence among \(M\) random variables can be quantified using total correlation (TC)~\cite{watanabe1960information} and dual total correlation (DTC)~\cite{sun1975linear}. Let \([M] := \{1, 2, \dots, M\}\) denote the index set, and \([M] \setminus \{i\}\) denote the set excluding \(i\). For convenience, we write \(X_{[M]}\) to denote the tuple \((X_1, \dots, X_M)\), and \(X_{[M] \setminus \{i\}}\) to denote all variables except \(X_i\). Then, TC and DTC are defined as:
\begin{align}
\mathrm{TC}(X_{[M]}) &= \sum_{i=1}^{M} H(X_i) - H(X_{[M]}), \\
\mathrm{DTC}(X_{[M]}) &= H(X_{[M]}) - \sum_{i=1}^{M} H(X_i \mid X_{[M] \setminus \{i\}}).
\end{align}

From a set theory perspective in Fig.~\ref{fig:TC_DTC}, TC and DTC differ in how they count the central three-way interaction (shaded area). Both TC and DTC capture conditional cross-variable interactions; however, TC double-counts the three-way interaction, thereby overestimating the true total dependence.

TC has recently been applied to multimodal learning via CLIP-like lower bounds~\cite{saporta2024contrasting}. However, TC is known to overestimate redundancy due to repeated counting of shared information. For example, in the case of three variables, it double-counts the interaction information~\cite{mcgill1954multivariate}:
\[
\Pi(X_1,X_2,X_3) = I(X_1;X_2) + I(X_1;X_3) - I(X_1;X_2,X_3),
\]
which may cause misleading dependence estimates (see Fig.~\ref{fig:TC_DTC} for an illustration). In contrast, DTC avoids this issue by emphasizing unique and synergistic information, making it better suited for multimodal alignment~\cite{timme2014synergy,austin2018multi}.

More broadly, the shared regions in Fig.~\ref{fig:TC_DTC} represent cross-modal information across physiological systems. In emotion recognition, they may reflect common affective factors, such as arousal responses jointly expressed in neural, cardiac, and peripheral signals. This motivates preserving both shared and conditional cross-modal information, rather than only redundant common variation.

\begin{figure}
    \centering
    \includegraphics[width=0.7\linewidth]{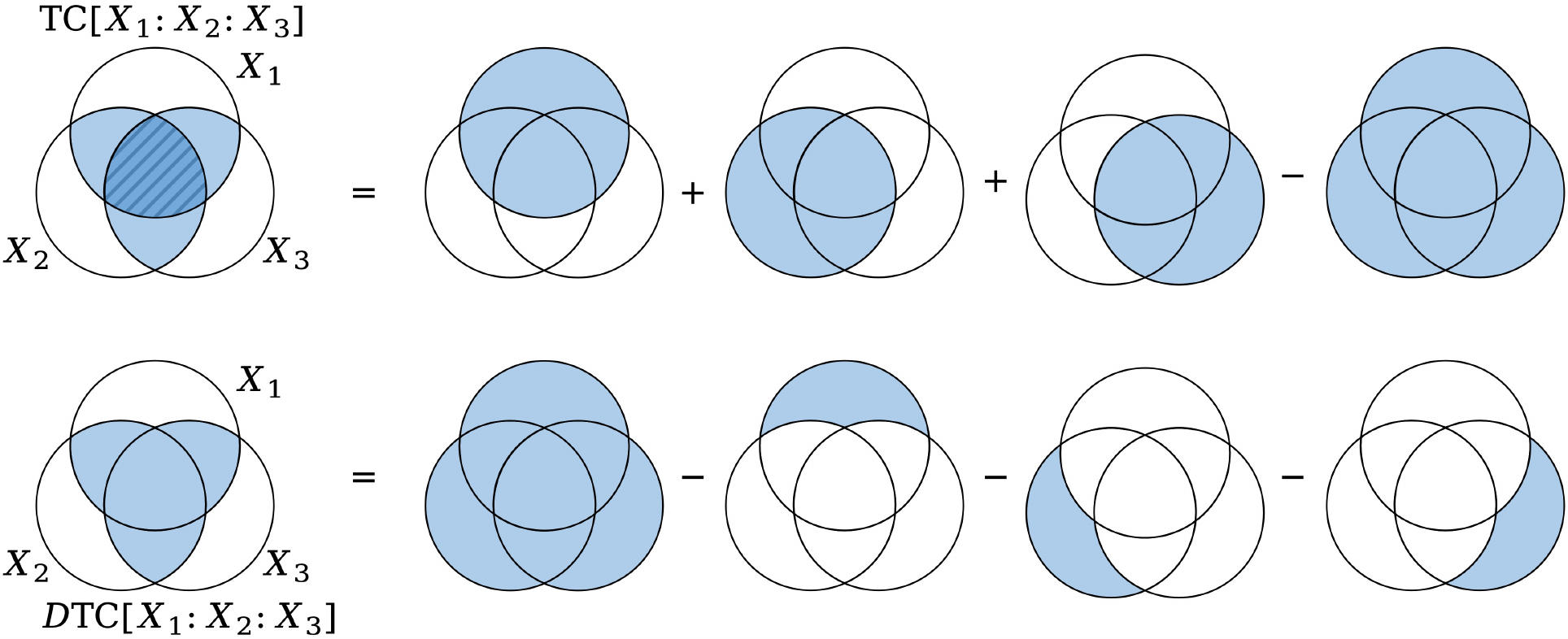}
    \caption{Venn diagram of TC and DTC for three variables. The top row shows that TC is obtained by summing the three marginal entropies and subtracting the joint entropy, which leaves the pairwise shared regions once but counts the central three-way shared region (shaded area) twice. The bottom row shows that DTC subtracts the three conditional unique components from the joint entropy, thereby retaining shared and conditional cross-modal information while counting the central three-way shared region only once.}
    \label{fig:TC_DTC}
\end{figure}

Beyond its theoretical advantages over TC, DTC also highlights a key limitation of prior approaches that naively extend CLIP by aggregating contrastive losses across all modality pairs. For instance, in the case of three variables \( X_1, X_2, X_3 \), DTC decomposes as (proof in Appendix~A):
\begin{equation}\label{eq:DTC_decompose}
\begin{split}
&\mathrm{DTC}(X_1, X_2, X_3) = \frac{1}{3} \big[I(X_1; X_2) + I(X_2; X_3) + I(X_1; X_3) \big] \\
&+ \frac{2}{3}\big[ I(X_1; X_2 \mid X_3) + I(X_2; X_3 \mid X_1) + I(X_1; X_3 \mid X_2) \big].
\end{split}
\end{equation}

Eq.~(\ref{eq:DTC_decompose}) shows that the pairwise strategy, though intuitive, fails to capture all conditional higher-order dependencies. This limitation motivates the cyclic pair-to-third FMCA objective introduced next. 



\subsection{Approximating DTC with FMCA}


We first present a \textit{sandwich bound} that approximates the DTC using more tractable mutual information terms.

\begin{theorem}\label{theorem_three_variable}
For three random variables \( X_1, X_2, X_3 \), the dual total correlation satisfies the following bound:
\begin{equation}
\begin{aligned}
\frac{1}{3} \sum_{\text{cyc}} I(\text{pair}; \text{third}) & \le \mathrm{DTC}(X_1, X_2, X_3)  \\
& \le \frac{2}{3} \sum_{\text{cyc}} I(\text{pair}; \text{third}).
\end{aligned}
\end{equation}
\end{theorem}

\begin{proof}
    Complete proofs are provided in Appendix~A, and we also present empirical evidence supporting the tightness of the proposed bound in Appendix~B.
\end{proof}

Here, \( \sum_{\text{cyc}} I(\text{pair}; \text{third}) \) denotes the cyclic sum over all permutations of the three variables, where in each term, two variables are grouped as a joint input and the third as the target. That is,
\begin{equation}
\begin{aligned}
& \sum_{\text{cyc}} I(\text{pair}; \text{third}) 
= I(X_1, X_2; X_3) \\
& \qquad + I(X_2, X_3; X_1) + I(X_3, X_1; X_2).
\end{aligned}
\end{equation}


Theorem~\ref{theorem_three_variable} suggests that optimizing the cyclic sum \( \sum_{\text{cyc}} I(\text{pair}; \text{third}) \) serves as a reliable surrogate for DTC. By the chain rule,
\begin{equation}
\begin{aligned}
I(X_1, X_2; X_3)
&= I(X_1; X_3) + I(X_2; X_3 \mid X_1) \\
&= I(X_2; X_3) + I(X_1; X_3 \mid X_2).
\end{aligned}
\end{equation}
Thus, each pair-to-third term contains both a pairwise dependence component and a conditional dependence component. Optimizing the cyclic sum in Eq.~(16) therefore encourages the encoders to capture information that cannot be fully represented by independent pairwise alignments alone, consistent with the DTC decomposition in Eq.~(\ref{eq:DTC_decompose}).

This sandwich bound can be extended to settings with \( M > 3 \) modalities (see Theorem~\ref{thm:generalization}), though in this study we focus on the case of three modalities.

\begin{theorem}\label{thm:generalization}
For \( M \geq 3 \) random variables \( X_1, X_2, \dots, X_M \), the dual total correlation satisfies the following sandwich bound:
\begin{align}
\frac{1}{M} \sum_{i=1}^{M} I\left( X_{[M] \setminus \{i\}} ; X_i \right)
& \leq \mathrm{DTC}(X_1, \dots, X_M) \\
& \leq \frac{M-1}{M} \sum_{i=1}^{M} I\left( X_{[M] \setminus \{i\}} ; X_i \right), \nonumber
\end{align}
where \( [M] \setminus \{i\} \) refers to the set of all indices except \(i\). 
\end{theorem}

Theorem~2 also provides a practical leave-one-out route to settings with more than three modalities. For each target modality $X_i$, the remaining modalities $X_{[M]\setminus\{i\}}$ can be encoded and fused into a joint representation, which is then aligned with the target modality representation. This gives an $M$-term objective rather than requiring an enumeration of all higher-order subsets.

\subsubsection{Estimation of joint mutual information}

\begin{figure*}
    \centering
    \includegraphics[width=1\linewidth]{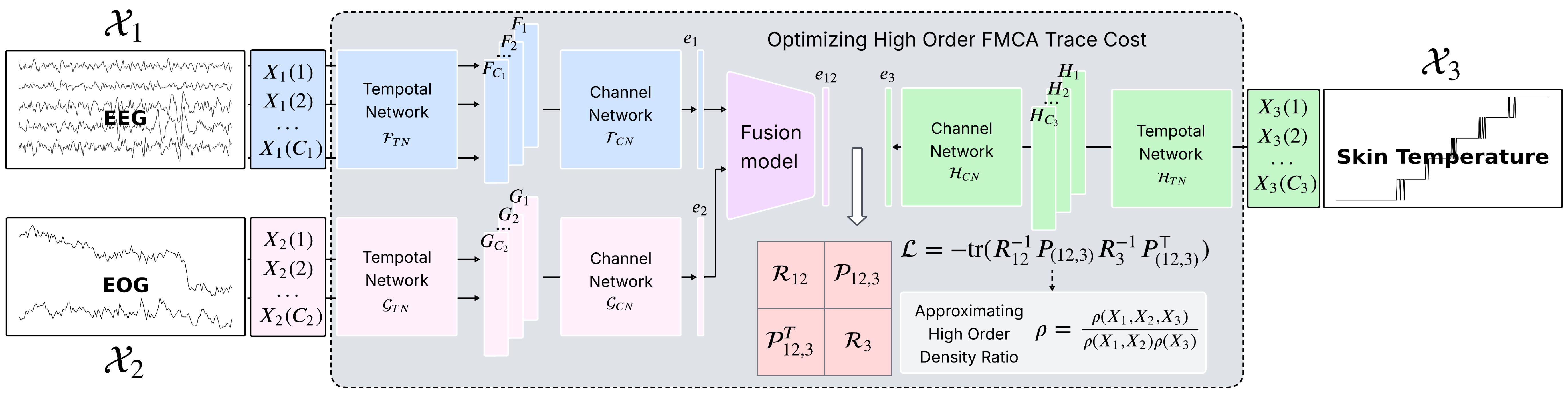}
    \caption{Diagram illustrating the maximization of total dependence between \{EEG, EOG\} and skin temperature by decomposing the density ratio \( \frac{p(X_1, X_2, X_3)}{p(X_1, X_2)p(X_3)} \). Networks \( f_\theta \) ($C_1$ channels) and \( g_\phi \) ($C_2$ channels) are applied to \( X_1 \) and \( X_2 \), respectively, to extract EEG and EOG embeddings \( e_1 \) and \( e_2 \), which are then fused via a fusion network to obtain the joint embedding \( e_{12} \). Similarly, \( h_\psi \) encodes \( X_3 \) ($C_3$ channels) into \( e_{3} \). Finally, \( e_{12} \) and \( e_{3} \) are used to optimize the trace loss $\mathrm{tr}(R_{12}^{-1} P_{12,3} R_3^{-1} P_{12,3}^\top)$, where $R$ and $P$ denote the sample auto- and cross-covariance matrices, respectively. The optimization converges to $- \sum_{i=1}^K \sigma_i$.}
    \label{fig:FMFC_framework}
\end{figure*}

With \( \sum_{\text{cyc}} I(\text{pair}; \text{third}) \) as a surrogate for DTC, the key challenge is estimating each joint mutual information term.

Rather than relying on InfoNCE- or CLIP-like objectives that provide only lower bounds, we aim to directly optimize the true dependence values using FMCA. However, as discussed earlier, FMCA suffers from high computational cost and numerical instability due to its reliance on eigenvalue decomposition. Additionally, it is not straightforward to extend FMCA to the joint mutual information term \( I(\text{pair}; \text{third}) \), which involves three variables instead of two.

To address both issues, we first introduce a new FMCA objective that eliminates the need for eigenvalue decomposition. We then present a neural framework for estimating \( I(\text{pair}; \text{third}) \) based on this formulation.

By Lemma~\ref{lemma_1}, the original FMCA objective essentially maximizes a truncated total statistical dependence (TSD), defined as \( T = -\sum_{i=1}^K \log(1 - \sigma_i) \). In fact, TSD can be generalized beyond the log-based formulation. Since \( \{ \sigma_i \} \) lie in \([0, 1)\), any convex and monotonically increasing function \( \zeta(\cdot) \) can be applied to define a generalized TSD measure:
\begin{equation}
T_\zeta := \sum_{i=1}^{K} \zeta(\sigma_i).
\end{equation}
This observation motivates our adoption of a flexible, function-based formulation of statistical dependence.


\begin{definition}[Generalized definition of TSD]
Given the eigenspectrum \( \{ \sigma_i \}_{i=1}^{\infty} \) and any monotonically increasing convex function \( \zeta(\cdot) : [0,1] \to [0, \infty) \) satisfying \( \zeta(0) = 0 \), we define the generalized form of the truncated total statistical dependence (TSD) as:
\begin{equation}
T(\sigma_1, \sigma_2, \dots; \zeta) = \sum_{i=1}^{K} \zeta(\sigma_i).
\end{equation}
\end{definition}

Notably, setting \( \zeta(\sigma) = \log \frac{1}{1 - \sigma} \) recovers the TSD formulation proposed in \cite{hu2022normalized}.
Alternatively, by choosing a much simpler function \( \zeta(\lambda) = \lambda \), we obtain an alternative of TSD, defined as $T := \sum_{i=1}^{K} \sigma_i$.

This motivates a new objective for the FMCA:
\begin{equation} \label{eq:FMCA_our}
\min_{\theta, \phi} \; r^{\dagger}(f_\theta, g_\phi) := - \sum_{i=1}^K \sigma_i.
\end{equation}
Minimizing Eq.(\ref{eq:FMCA_our}) effectively maximizes \( T:=  \sum_{i=1}^{K} \sigma_i \), thus maximizing the dependence for two variables.

Compared to the original objective \( r(f_\theta, g_\phi) \) in Eq.~(\ref{eq:FMCA_obj}), the new objective avoids explicit eigenvalue decomposition and can be efficiently computed as a trace term:
\begin{equation}\label{eq:FMCA_final_loss}
 \sum_{i=1}^K \sigma_i =  \mathrm{tr}(VV^\top) =  \mathrm{tr}(R_X^{-1} R_{XY} R_Y^{-1} R_{XY}^\top),
\end{equation}
where \( V = R_X^{-1/2} R_{XY} R_Y^{-1/2} \) is the normalized cross-covariance operator.

\begin{remark}[Numerical stabilization]
Although the trace-based FMCA surrogate avoids the log-determinant and eigendecomposition operations required by the original FMCA formulation, the empirical mini-batch autocorrelation matrices in Eq.~(22) can still be ill-conditioned. In implementation, before evaluating the trace objective, we add a small diagonal ridge term to each autocorrelation matrix:
\[
\widetilde{R}_X = R_X + \epsilon I,\qquad
\widetilde{R}_Y = R_Y + \epsilon I.
\]
The trace score is then evaluated with $\widetilde{R}_X$ and $\widetilde{R}_Y$ in place of $R_X$ and $R_Y$. We use $\epsilon=10^{-6}$ in the main experiments. Appendix~C-F reports a sensitivity analysis over $\epsilon \in \{0,10^{-6},10^{-5},10^{-4},10^{-3},10^{-2}\}$, showing that ridge regularization improves matrix conditioning while downstream accuracy and macro-F1 remain robust.
\end{remark}

\begin{lemma}[First-order approximation] 
Let \( \sigma_i \in (0, 1) \) denote the eigenvalues of density ratio $\rho$. By applying the first-order Taylor approximation to \( \log(1 - x) \), i.e.,
\[
\log(1 - x) \approx -x - \frac{x^2}{2} - \cdots  \Rightarrow  \sum_{i=1}^K \log(1 - \sigma_i) \approx - \sum_{i=1}^K \sigma_i.
\]
\end{lemma}

This shows that the trace-based objective in Eq.~(\ref{eq:FMCA_final_loss}) can be interpreted as a first-order approximation of the original log-determinant-based dependence measure. Since the eigenvalues \( \sigma_i \) are often small in practice (especially in high-dimensional or weakly correlated modalities), higher-order terms in the expansion decay rapidly, making the approximation sufficiently accurate while avoiding the numerical instability of computing log-determinants and eigenvalues.

\subsubsection{Three Modality Framework}
After introducing the trace-based loss, we detail a neural implementation for estimating \( I(\text{pair}; \text{third}) \) using FMCA. As an illustrative example, Fig.~\ref{fig:FMFC_framework} shows the estimation of I((EEG, EOG); Temperature).

EEG and EOG signals are first encoded by two modality-specific encoders
$f_\theta$ and $g_\phi$, producing embeddings $e_1, e_2 \in \mathbb{R}^K$, respectively. These embeddings are then fused through a fusion network to form a joint representation $e_{12} \in \mathbb{R}^K$, which serves as the representation of the paired modalities (EEG, EOG). 
In our default implementation, the fusion network is instantiated as a lightweight pairwise MLP projection head. Each modality encoder outputs a fixed 128-dimensional embedding. For each cyclic pair, two modality embeddings are concatenated into a 256-dimensional vector and projected back to the shared 128-dimensional space by a two-layer MLP. The MLP consists of Linear(256, 512), Batch Normalization, ReLU, Linear(512, 128), and Batch Normalization, with no dropout. The three fused pair representations, denoted as $e_{12}$, $e_{13}$, and $e_{23}$, are produced by projection heads with the same architecture but separate parameters, enabling each modality pair to learn pair-specific interactions. These fusion heads are used only during MFMC pretraining to construct paired representations for estimating the pair-to-third dependence terms, and are discarded during downstream inference, where only the single-modality encoders are retained.

We adopt this lightweight MLP design because the fusion head is not intended to serve as a downstream multimodal classifier. Instead, its role is to provide a compact parameterization of paired modality embeddings within the cyclic DTC-grounded FMCA objective. The main representation learning capacity comes from the modality encoders and the cyclic dependence maximization across all pair-to-third terms. Therefore, the MLP fusion head provides a sufficient and parameter-efficient way to model pairwise cross-modality interactions in our setting, while avoiding unnecessary architectural complexity. We further evaluate lightweight attention-based fusion variants in Appendix~C-C.

In parallel, the temperature signal is encoded by a third encoder $h_\psi$ to produce $e_3 \in \mathbb{R}^K$. The dependence between the paired representation $e_{12}$ and the third modality representation $e_3$ is then maximized using the trace-based FMCA objective in Eq.~(\ref{eq:FMCA_final_loss}), where the auto-covariance matrices $R$ and cross-covariance matrix $P$ are estimated empirically from mini-batch samples of $\{e_{12}, e_3\}$.

The same procedure is applied cyclically to estimate \(I((\mathrm{EEG}, \mathrm{Temperature}); \mathrm{EOG})\) and \(I((\mathrm{EOG}, \mathrm{Temperature}); \mathrm{EEG})\).
Consequently, the final tri-modal training objective is given by:
\begin{equation}
\begin{aligned}
\mathcal{L}_{\theta,\phi,\psi} &=
\mathrm{tr}(R_{12}^{-1} P_{12,3} R_3^{-1} P_{12,3}^\top)
+ \mathrm{tr}(R_{13}^{-1} P_{13,2} R_2^{-1} P_{13,2}^\top) \\
&\quad + \mathrm{tr}(R_{23}^{-1} P_{23,1} R_1^{-1} P_{23,1}^\top),
\end{aligned}
\end{equation}
In implementation, each autocorrelation matrix \(R\) in Eq.~(23) is replaced by its ridge-regularized version \(\widetilde{R}\). After pretraining, the fusion heads are discarded and only the modality-specific encoders \(f_\theta\), \(g_\phi\), and \(h_\psi\) are retained for downstream emotion recognition. Since each encoder is optimized through cyclic pair-to-third dependence maximization, it learns cross-modality-informed representations and can be used independently at inference time. Thus, MFMC supports single-modality deployment under missing-modality or resource-constrained conditions, without positive-negative pair construction or modality-specific augmentations.

Algorithm~1 summarizes the full MFMC training and inference workflow, including the generalized leave-one-out cyclic dependence maximization and the downstream use of single-modality encoders.

\begin{algorithm}[t]
\caption{Generalized MFMC training and inference workflow}
\label{alg:mfmc_general}
\begin{algorithmic}[1]
\STATE \textbf{Input:} synchronized mini-batches \(\{X_i\}_{i=1}^{M}\) with batch size \(B\), modality encoders \(\{f_i\}_{i=1}^{M}\), fusion module \(F\), and shared embedding dimension \(K\)
\FOR{each training mini-batch}
\STATE Compute \(e_i=f_i(X_i)\in\mathbb{R}^{B\times K}\) for \(i=1,\ldots,M\)
\STATE Choose \(S\subseteq\{1,\ldots,M\}\); use \(S=\{1,\ldots,M\}\) for full-sum training
\FOR{each \(i\in S\)}
\STATE Fuse \(\{e_j\}_{j\neq i}\) into \(e_{[M]\setminus\{i\}}=F(\{e_j\}_{j\neq i})\in\mathbb{R}^{B\times K}\)
\STATE Estimate \(R_{[M]\setminus\{i\}}\), \(R_i\), and \(P_{[M]\setminus\{i\},i}\) from the mini-batch
\STATE Compute \(\ell_i=\mathrm{tr}\!\left(\widetilde{R}_{[M]\setminus\{i\}}^{-1}P_{[M]\setminus\{i\},i}\widetilde{R}_i^{-1}P_{[M]\setminus\{i\},i}^{\top}\right)\)
\ENDFOR
\STATE Set \(\mathcal{J}_{\mathrm{MFMC}}=\frac{M}{|S|}\sum_{i\in S}\ell_i\) and update the encoder and fusion parameters to maximize it
\ENDFOR
\STATE Discard \(F\), retain the modality encoders, and train the downstream classifier on top of any single-modality encoder
\end{algorithmic}
\end{algorithm}

\noindent\textbf{Computational complexity.}
The computational cost follows directly from Algorithm~1. For one leave-one-out term, forming the sample autocorrelation and cross-covariance matrices with batch size $B$ and embedding dimension $K$ costs $\mathcal{O}(BK^2)$, and the matrix inverse or linear solve in the trace objective costs $\mathcal{O}(K^3)$. Full-sum MFMC therefore scales as $\mathcal{O}(M(BK^2+K^3))$ per mini-batch. The dominant memory cost scales as $\mathcal{O}(MBK+MK^2)$, accounting for modality embeddings and covariance matrices. In the tri-modal setting used in our main experiments, all cyclic terms are summed exactly. For larger $M$, the same leave-one-out objective can also be approximated by stochastic term sampling, where only a subset $S$ of target modalities is evaluated in each mini-batch. Additional implementation details and the DEAP scalability study up to five modalities, including a 5-modality sampled variant, are provided in Appendix~C-E.

\section{Experiments and Results}
\label{sec:experiments}

\subsection{Datasets}
We use three public affective computing datasets, each containing at least three simultaneously recorded physiological signal modalities. A generic procedure is applied to select the most informative channels from the desired modalities in each dataset (see Sec.~\ref{sec:impl}). We exclude the widely used SEED~\cite{seed} and DREAMER~\cite{dreamer}, as they contain only two signal modalities.

\textbf{DEAP-Emotion}~\cite{deap} includes 32 EEG and 8 peripheral channels recorded from 32 participants (16 male, 16 female) watching 40 one-minute music videos, each rated on 1--9 Self-Assessment Manikin (SAM) valence and arousal scales. 
Signals were acquired in a controlled laboratory setting using a BioSemi ActiveTwo system at 512\,Hz (downsampled to 128\,Hz), with 32 active Ag/AgCl electrodes placed according to the international 10--20 system and peripheral sensors capturing EOG, EMG, GSR, respiration, blood volume pulse (BVP), and skin temperature; frontal face video was additionally recorded for most subjects~\cite{deap}. The 40 stimuli span all four valence--arousal quadrants, and each trial followed a standardized baseline--video--rating protocol with self-reports of valence, arousal, dominance, liking, and familiarity. DEAP is a canonical EEG-based affective-computing benchmark widely used by deep and self-supervised emotion-recognition models~\cite{song2018eeg,zhang2022ganser,wang2023self}. 

We focus on central (EEG) and ocular/thermal (EOG, skin temperature) channels that are most relevant to affective processing (see Sec.~\ref{sec:impl_modality_selection}).

\textbf{CEAP-360VR}~\cite{ceap} provides only peripheral signals sampled at $\sim$30\,Hz from 32 participants viewing eight 360\textdegree{} videos. Available channels include blood volume pulse (BVP), electrodermal activity (EDA), skin temperature (SKT), heart rate (HR), inter-beat interval (IBI), and three-axis acceleration. 
Participants viewed the clips in VR while continuous valence--arousal annotations were collected with a joystick, followed by within-VR SAM ratings after each clip~\cite{ceap}. Physiological responses were recorded with an Empatica E4 wristband on the non-dominant hand, providing synchronized heart rate, EDA, and SKT measurements alongside head- and eye-movement streams. CEAP-360VR has been used for valence/arousal classification, correlation-based feature extraction, and multimodal physiological emotion recognition in immersive media~\cite{zhang2019,yang2021b,ramadan2024mm}.

In our experiments we focus on BVP, EDA, and SKT as the most reliable and emotion relevant modalities.

\textbf{MAHNOB-HCI}~\cite{Soleymani2012MAHNOB} records EEG, ECG, galvanic skin response (GSR/EDA), respiration, SKT, audio, video, and eye-tracking from 30 participants during movie viewing. 
Data were collected with synchronized video, audio, eye tracking, 32-channel EEG, and peripheral sensors including ECG, GSR, respiration, and SKT~\cite{Soleymani2012MAHNOB}. In the emotion-recognition experiment considered here, 27 valid participants watched 20 emotional film clips and, after each clip, reported their felt emotion using dimensional ratings (valence, arousal, dominance, predictability) and free-form emotion keywords, enabling both dimensional and categorical analyses. MAHNOB-HCI is a key multimodal affect benchmark and has been used to evaluate hypercomplex fusion models on EEG and peripheral physiological signals~\cite{lopez2023hypercomplex,lopez2024}.

We use EEG together with a subset of peripheral channels (cardiac and electrodermal signals) that capture autonomic arousal; see Sec.~\ref{sec:impl_modality_selection} for the exact trimodal configuration.




Across all datasets, valence and arousal ratings are binarized at the midpoint (5), forming four emotional quadrants:
\textsc{LAHV} (low arousal, high valence),
\textsc{HALV} (high arousal, low valence),
\textsc{LALV} (low arousal, low valence),
and \textsc{HAHV} (high arousal, high valence),
yielding a 4-class classification task.

\subsection{Implementation Details}
\label{sec:impl}

\subsubsection{\textbf{Temporal windowing}}
\label{sec:impl_window}

\begin{figure}[t]
    \centering
    \includegraphics[width=1\linewidth]{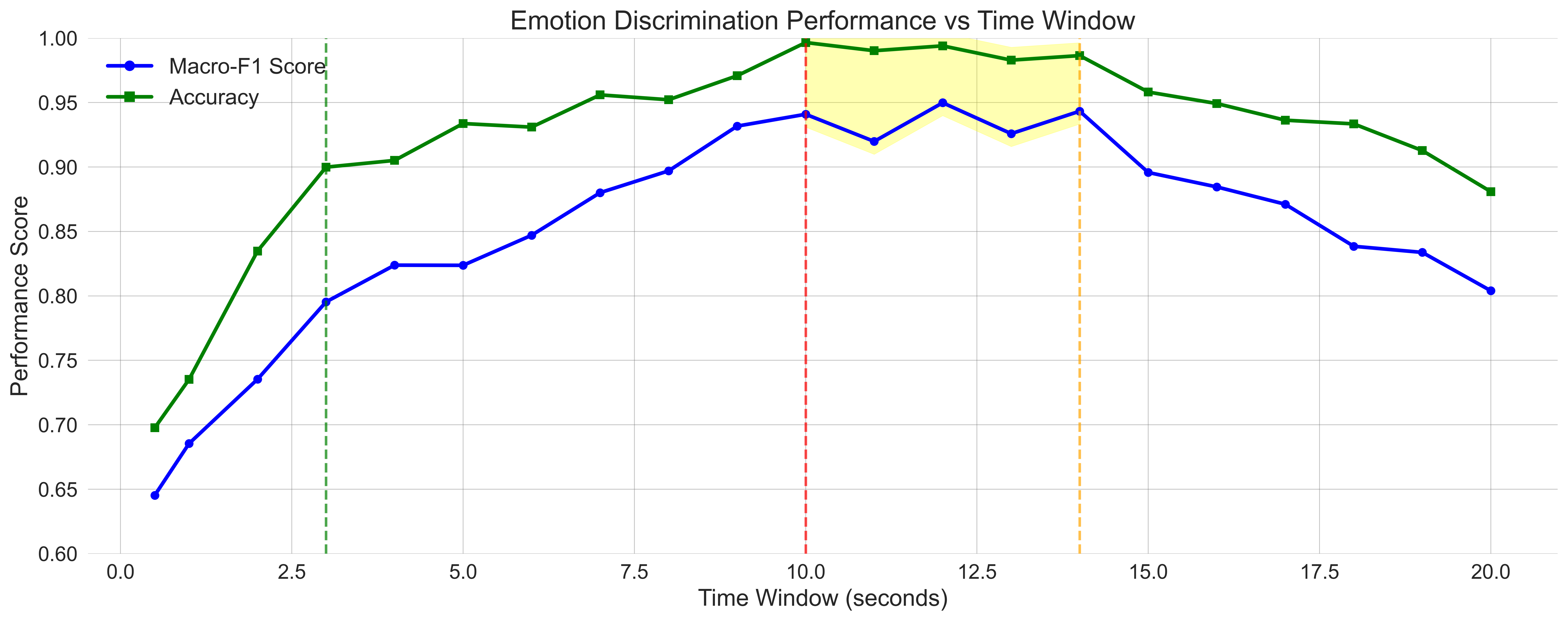}
    \caption{Emotion-discrimination performance versus window length.
    Both macro-$F_{1}$ (blue) and accuracy (green) peak at 10 s, remain
    flat up to 14 s, and deteriorate beyond that.
    Very short windows ($<3$ s) perform markedly worse.}
    \label{fig:performaneVSTime}
\end{figure}

Selecting appropriate temporal windows is critical in multimodal physiological tasks to balance temporal context, sample count, and nonstationarity. In a DEAP pre-experiment, we first performed a window-size sweep over durations from 0.5 s to 20 s. For shorter windows ($\leq 5$\,s) we used nonoverlapping segmentation to avoid excessive redundancy and highly correlated samples; for longer windows ($> 5$\,s) we adopted a stride overlap of 40\% (i.e., stride $= 0.6 \times$ window length) to maintain an adequate sample count while preventing overly high correlation between adjacent segments. Empirical results (see Fig.~\ref{fig:performaneVSTime}) show an increase in performance up to about 10\,s window length, a plateau from 10--14\,s, and a decrease beyond that point, indicating that $\sim$10\,s captures sufficient temporal context without incurring excessive nonstationarity or diluted temporal resolution. This is consistent with prior work showing that window size has a marked effect on emotion recognition performance~\cite{keelawat2021}.

Unless otherwise noted, all subsequent experiments use 10 s windows, but the final preprocessing is dataset specific. For DEAP, windows start 3 s after stimulus onset and are extracted with a fixed stride of 0.4 s (51 samples at 128 Hz). EEG, EOG, and skin temperature (SKT) are rescaled by modality-specific constants computed once from the first available DEAP window, and a window is discarded if any normalized EEG, EOG, or SKT value exceeds ±5. For CEAP-360VR, physiological streams are first interpolated to the annotation timeline, then segmented into 10 s windows with a 0.5 s stride (15 frames at approximately 30 Hz). ACC, BVP, EDA, SKT, HR, and IBI are normalized by modality-specific constants computed from the first available window, and windows with normalized values outside ±5 are discarded. For MAHNOB-HCI, we discard the initial 30 s neutral baseline, segment the remaining recording into 10 s windows with a 0.5 s stride (128 samples at 256 Hz), normalize each modality by the magnitude of its first available window, and discard windows containing values outside ±5 after normalization. This procedure yields 20,097 windows for DEAP, 25,216 segments for CEAP-360VR, and 105,000 segments for MAHNOB-HCI.

We further conducted a one-factor DEAP sensitivity study over window length, stride, normalization, and outlier threshold. For efficiency, it uses one fixed subject-dependent split and one fixed subject-independent split, reporting the best test accuracy and macro-F1 during training; details are provided in Appendix~C-A as qualitative trend diagnostics rather than replacements for the official 5-fold results.

\subsubsection{\textbf{Modality selection}}
\label{sec:impl_modality_selection}

\begin{figure}
    \centering
    \includegraphics[width=\linewidth]{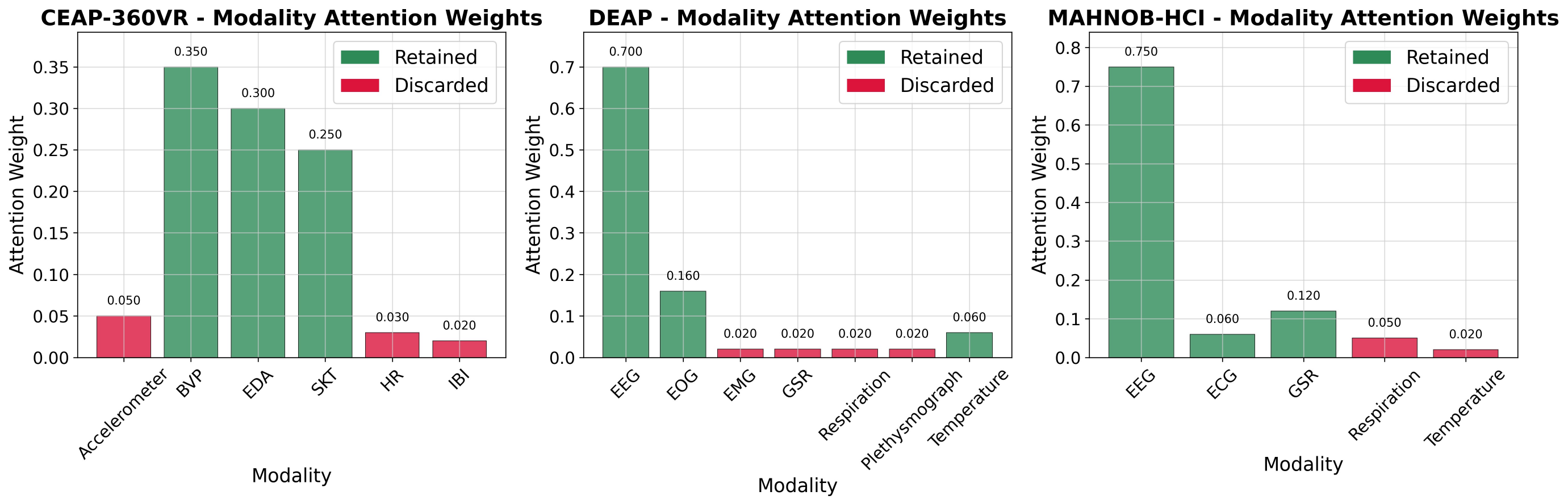}
    \caption{Learnable attention weights for physiological signal selection across CEAP-360VR, DEAP, and MAHNOB-HCI datasets. Green bars indicate retained modalities, red bars indicate discarded modalities based on learned importance scores. Retained modalities are selected by the top 3 ranking rule rather than by a fixed absolute threshold.}
    \label{fig:sig_selection}
\end{figure}

To select informative modalities and suppress noisy or redundant ones, we implement a learnable attention mask at the modality-fusion layer. Each modality embedding is associated with a softmax-normalized scalar weight that is trained jointly with the classification network. After convergence, we rank modalities by the final attention weights and retain the highest-scoring subset for each dataset. Specifically, we do not use a manually specified absolute numerical threshold. Since all downstream MFMC experiments are conducted with three active modalities, the retained subset is determined by a top 3 ranking rule: after convergence, the three modality groups with the largest learned softmax weights are selected. The same gate definition, training objective, and selection rule are used for DEAP, CEAP-360VR, and MAHNOB-HCI, with the candidate modality groups defined separately for each dataset. Additional implementation details for the modality-selection protocol are provided in Appendix~C-A. This follows recent multimodal literature showing that attention at the modality level allows the model to emphasize informative modalities and down-weight uninformative ones~\cite{mamieva2023multimodal}.

Guided by this attention analysis, we fix the trimodal inputs per dataset as follows and use these configurations for all subsequent experiments: 
\begin{itemize}
    \item \textbf{DEAP}: \textit{EEG}, \textit{EOG}, and skin temperature (\textit{SKT});
    \item \textbf{CEAP-360VR}: \textit{EDA}, \textit{BVP}, and \textit{SKT};
    \item \textbf{MAHNOB-HCI}: \textit{EEG}, \textit{ECG}, and \textit{GSR}.
\end{itemize}

\subsubsection{\textbf{Encoder architecture}}
\label{sec:impl_encoder}
To ensure fairness, all self-supervised and supervised methods share the same modality-specific backbone. For multichannel signals such as EEG or EOG, the encoder applies four \emph{depth-wise} 1D convolutional blocks along the time axis, reusing the same kernel across channels. Each block ends with a stride-4 max-pooling operation, resulting in a total temporal downsampling factor of $4^4 = 256$. At 128\,Hz with 10\,s windows, each segment thus contains 1{,}280 time samples per channel, which are progressively reduced to compact feature maps.

The resulting per-channel features are flattened, concatenated, and passed through a lightweight fusion MLP (with one hidden layer) to model cross-channel interactions, yielding a 128-dimensional modality embedding suitable for multimodal alignment. For univariate signals with a single channel, such as skin temperature, we use the same four-block temporal CNN (omitting the cross-channel fusion step) and project the flattened features directly into the shared 128-dimensional space.

\subsubsection{\textbf{Baselines and training protocol}}
\label{sec:impl_baselines}
We term our approach the Multimodal Functional Maximum Correlation (MFMC) algorithm. We compare it against three families of baselines.

\paragraph{\textbf{Unimodal self-supervised learning}}
We first benchmark three \emph{uni-modal} SSL methods applied independently to each modality: SimCLR~\cite{chen2020simple}, Barlow Twins~\cite{zbontar2021barlow}, and VICReg~\cite{bardes2022vicreg}. Each method forms positive pairs from two augmented views of the same signal window using Gaussian noise, temporal shifts, and channel dropout~\cite{iwana2021}. The rest of the mini-batch serves as negatives when required.

\paragraph{\textbf{Multimodal self-supervised learning}}
For the bi-modal setting ($K=2$), we evaluate FMCA~\cite{hu2022normalized} and CLIP~\cite{radford2021learning}. For the tri-modal setting ($K=3$), we evaluate SymILE~\cite{saporta2024contrasting}, a recent SOTA method that maximizes a CLIP-like lower bound on total correlation (TC), and our CLIP++ baseline, which aggregates all pairwise CLIP losses by summing three InfoNCE terms ($x_1{\leftrightarrow}x_2$, $x_1{\leftrightarrow}x_3$, $x_2{\leftrightarrow}x_3$). Both SymILE and CLIP++ approximate multiway structure through pairwise or TC-based bounds, unlike the DTC-grounded MFMC objective.

\paragraph{\textbf{Supervised baselines}}
Finally, we compare against three \emph{supervised} approaches: EEGNet~\cite{lawhern2018eegnet}, a vanilla CNN with the same encoder backbone, and HyperFuseNet~\cite{lopez2023hypercomplex}, a SOTA multimodal emotion recognition model that fuses modality-specific representations using a hypercomplex neural network.

\paragraph{\textbf{Training protocol}}
For each contrastive SSL method, we first perform self-supervised pretraining using the full training set of unlabeled windows. After pretraining, modality encoders are frozen, and a lightweight three-layer MLP classifier is trained on top of any encoder for quadrants classification. Supervised baselines are trained end-to-end under the same windowing, modality configurations, and training budget. To isolate the effect of the objective and modality scope, all methods share the same encoders, projection heads, classifier, preprocessing, and augmentation recipes whenever applicable; only the loss function and active modality set differ.

\begin{itemize}[nosep,leftmargin=*]
  \item Ours (MFMC). Objective: maximizes joint dependence via a DTC-grounded surrogate across \(K \geq 3\) modalities, without positives/negatives or CLIP-style bounds. Mechanics: trace-based FMCA for stability and efficiency. What it tests: whether higher-order conditional interactions improve generalization beyond pairwise criteria.
  \item \textbf{Unimodal SSL: SimCLR, Barlow Twins, VICReg.}
        \emph{Objective families:} contrastive with negatives (SimCLR), redundancy reduction (Barlow Twins), and variance--invariance--covariance regularization (VICReg).
        \emph{What they test:} whether strong within-modality representation learning alone can explain our gains, disentangling the role of negatives versus redundancy/variance regularizers.
  \item \textbf{Bimodal SSL: CLIP, FMCA.}
        \emph{Objective:} pairwise alignment either via InfoNCE-style bounds (CLIP) or direct dependence maximization (FMCA).
        \emph{What they test:} whether two-way alignment between modalities is sufficient.
  \item \textbf{Trimodal (pairwise extensions): CLIP++ (sum of pairs).}
        \emph{Objective:} aggregates all pairwise CLIP losses.
        \emph{Limitation:} cannot represent synergy that emerges only when all three modalities are considered jointly.
  \item \textbf{Trimodal (multiway bound): SymILE.}
        \emph{Objective:} CLIP-like lower bound targeting TC across multiple modalities.
        \emph{Trade-offs:} retains positives/negatives and TC overcounting of redundancy, lacking DTC's correction.
  \item \textbf{Supervised: EEGNet, vanilla CNN, HyperFuseNet.} Objective: label-supervised references with task-specific inductive biases. What they test: the gap between self-supervision and full supervision under subject-dependent and subject-independent protocols.
\end{itemize}

For reproducibility, script-level hyperparameters for the self-supervised and multimodal self-supervised baselines are summarized in Appendix~C-B, Table~\ref{tab:ssl_baseline_hyperparams}; supervised baselines follow separate supervised training protocols.

\subsubsection{\textbf{Train--test splits and evaluation protocols}}
\label{sec:impl_splits}
We evaluate performance under two data-partitioning schemes. In the \emph{subject-dependent} setting, trials from all participants are pooled and stratified 5-fold cross-validation is applied, with 80\% of windows used for training and 20\% for testing in each fold. In the \emph{subject-independent} setting, the test set comprises entirely unseen participants: we adopt a 5-fold leave-group-out protocol, training on 80\% of subjects and testing on the remaining 20\%. Subject splits are as follows: DEAP (15/4 train/test subjects), MAHNOB-HCI (21/6), and CEAP-360VR (26/6).

We use classification accuracy as the primary metric and report mean $\pm$ standard deviation across the 5 folds in both subject-dependent and subject-independent settings. For subject-dependent evaluation, predictions are made at the window level. For subject-independent evaluation, all windows from held-out subjects are aggregated into a single test set per fold.

\subsubsection{\textbf{Computing environment}}
\label{sec:impl_hardware}

All experiments were conducted on a Linux server equipped with a single NVIDIA L40S GPU (48\,GB VRAM), dual Intel Xeon Gold 6330 CPUs, and 256\,GB RAM. 
Training MFMC on DEAP for five subject-dependent cross-validation folds required approximately 5 hours wall-clock time, while the subject-independent setup took about 6 hours. 
All models were implemented in Python using PyTorch~\cite{paszke2019pytorch}.
Additional preprocessing, modality-selection, and baseline-training details are provided in Appendix~C-A and Appendix~C-B, and a GitHub repository is made available for reproducibility at \url{https://github.com/DY9910/MFMC}.
A DEAP scalability study with 3-, 4-, and 5-modality MFMC settings, including a 5-modality sampled variant, is reported in Appendix~C-E.

\subsection{Overall Results}

\subsubsection{\textbf{Subject-dependent performance}}
Tables~\ref{tab:eeg_eog_deap}--\ref{tab:eeg_ecg_mahnob} summarize classification accuracies under both subject-dependent and subject-independent settings across the three datasets. In the subject-dependent regime, MFMC achieves strong performance that is consistently competitive with the fully supervised HyperFuseNet while outperforming all other self-supervised baselines.

On \textbf{DEAP}, MFMC attains 0.987 accuracy with EEG and 0.925 with EOG, close to HyperFuseNet's 0.995 and clearly above all other SSL methods (Table~\ref{tab:eeg_eog_deap}). On \textbf{CEAP-360VR}, MFMC improves over all self-supervised baselines and approaches HyperFuseNet for both EDA and BVP (Table~\ref{tab:eda_bvp_ceap}), indicating that tri-modal alignment is beneficial even for purely peripheral signals. On \textbf{MAHNOB-HCI}, MFMC nearly matches HyperFuseNet in the EEG case (0.953 vs.\ 0.955) and substantially outperforms the remaining SSL methods for both EEG and ECG (Table~\ref{tab:eeg_ecg_mahnob}). Overall, MFMC is the only SSL method that consistently performs competitively with a strong supervised SOTA across all datasets and modalities.

\subsubsection{\textbf{Subject-independent performance}}
The subject-independent setting is considerably more challenging: accuracies are substantially lower than in the subject-dependent case, highlighting the difficulty of cross-subject generalization in physiological emotion recognition. Nevertheless, MFMC establishes a new SOTA among all considered methods in 5 out of 6 subject-independent scenarios across datasets and modalities.

\begin{table}[t]
\centering
\caption{Quantitative subject-dependent to subject-independent accuracy degradation of MFMC.}
\begin{tabular}{lcccc}
\toprule
Dataset & Modality & Subj.-Dep. & Subj.-Indep. & Absolute drop \\
\midrule
DEAP & EEG & 0.987 & 0.346 & 0.641 \\
DEAP & EOG & 0.925 & 0.343 & 0.582 \\
CEAP-360VR & EDA & 0.868 & 0.331 & 0.537 \\
CEAP-360VR & BVP & 0.820 & 0.311 & 0.509 \\
MAHNOB-HCI & EEG & 0.953 & 0.442 & 0.511 \\
MAHNOB-HCI & ECG & 0.910 & 0.441 & 0.469 \\
\bottomrule
\end{tabular}
\label{tab:subject_independent_drop}

\vspace{2pt}
Absolute drop is computed as subject-dependent accuracy minus subject-independent accuracy.
\end{table}

Table~\ref{tab:subject_independent_drop} quantifies the degradation from subject-dependent to subject-independent evaluation. The absolute accuracy drop ranges from 0.469 to 0.641 across datasets and target modalities. This drop is expected because the subject-independent protocol evaluates entirely unseen participants with different physiological baselines and response patterns. Nevertheless, MFMC remains competitive, suggesting that it preserves transferable cross-modal information rather than merely memorizing subject-specific patterns.

For example, MFMC achieves the best or statistically tied best subject-independent performance on DEAP EEG, CEAP EDA/BVP, and MAHNOB ECG, while remaining competitive on MAHNOB EEG. Importantly, in many cases MFMC surpasses supervised baselines (including HyperFuseNet and EEGNet), demonstrating that a DTC-grounded self-supervised objective can yield robust and scalable representations that generalize better across subjects than label-intensive training.

\begin{table*}[t]
\centering
\caption{Subject-dependent (left) and subject-independent (right) accuracy (\%) on \textbf{DEAP} dataset using EEG and EOG.}
\begin{minipage}{0.48\textwidth}
\centering
\begin{tabular}{llcc}
\toprule
\textbf{Category} & \textbf{Method} &
\textbf{EEG} Subj.-Dep. & \textbf{EOG} Subj.-Dep. \\
\midrule
\rowcolor{blue!10}
\cellcolor{white}\multirow{3}{*}{\textbf{Tri-modal}}
  & MFMC   & \underline{0.987}$\pm$0.009 & \underline{0.925}$\pm$0.005 \\
  & Symile & 0.978$\pm$0.006             & 0.912$\pm$0.006 \\
  & CLIP++ & 0.972$\pm$0.004             & 0.923$\pm$0.006 \\
\midrule
\multirow{2}{*}{\textbf{Bi-modal}}
  & FMCA   & 0.872$\pm$0.005             & 0.810$\pm$0.007 \\
  & CLIP   & 0.960$\pm$0.008             & 0.895$\pm$0.009 \\
\midrule
\multirow{3}{*}{\textbf{Uni-modal}}
  & Barlow Twins & 0.944$\pm$0.008 & 0.880$\pm$0.010 \\
  & SimCLR       & 0.975$\pm$0.011 & 0.910$\pm$0.012 \\
  & VICReg       & 0.953$\pm$0.009 & 0.890$\pm$0.011 \\
\midrule
\multirow{3}{*}{\textbf{Supervised}}
  & HyperFuseNet & \textbf{0.995}$\pm$0.005 & \textbf{0.995}$\pm$0.005 \\
  & EEGNet       & 0.958$\pm$0.007          & 0.890$\pm$0.008 \\
  & CNN          & 0.942$\pm$0.010          & 0.910$\pm$0.009 \\
\bottomrule
\end{tabular}
\end{minipage}
\hfill
\begin{minipage}{0.48\textwidth}
\centering
\begin{tabular}{llcc}
\toprule
\textbf{Category} & \textbf{Method} &
\textbf{EEG} Subj.-Indep. & \textbf{EOG} Subj.-Indep. \\
\midrule
\rowcolor{blue!10}
\cellcolor{white}\multirow{3}{*}{\textbf{Tri-modal}}
  & MFMC   & \textbf{0.346}$\pm$0.030 & \textbf{0.343}$\pm$0.028 \\
  & Symile & 0.341$\pm$0.032          & 0.314$\pm$0.025 \\
  & CLIP++ & 0.307$\pm$0.031          & 0.272$\pm$0.050 \\
\midrule
\multirow{2}{*}{\textbf{Bi-modal}}
  & FMCA   & \underline{0.343}$\pm$0.035 & 0.279$\pm$0.029 \\
  & CLIP   & 0.311$\pm$0.028             & 0.269$\pm$0.025 \\
\midrule
\multirow{3}{*}{\textbf{Uni-modal}}
  & Barlow Twins & 0.243$\pm$0.023 & 0.251$\pm$0.026 \\
  & SimCLR       & 0.299$\pm$0.022 & 0.263$\pm$0.031 \\
  & VICReg       & 0.327$\pm$0.026 & 0.275$\pm$0.027 \\
\midrule
\multirow{3}{*}{\textbf{Supervised}}
  & HyperFuseNet & 0.341$\pm$0.030          & \underline{0.341}$\pm$0.030 \\
  & EEGNet       & 0.275$\pm$0.025          & 0.266$\pm$0.022 \\
  & CNN          & 0.256$\pm$0.027          & 0.263$\pm$0.027 \\
\bottomrule
\end{tabular}
\end{minipage}
\label{tab:eeg_eog_deap}
\end{table*}

\begin{table*}[t]
\centering
\caption{Subject-dependent (left) and subject-independent (right) accuracy (\%) on \textbf{CEAP-360VR} dataset using EDA and BVP.}
\begin{minipage}{0.48\textwidth}
\centering
\begin{tabular}{llcc}
\toprule
\textbf{Category} & \textbf{Method} &
\textbf{EDA} Subj.-Dep. & \textbf{BVP} Subj.-Dep. \\
\midrule
\rowcolor{blue!10}
\cellcolor{white}\multirow{3}{*}{\textbf{Tri-modal}}
  & MFMC   & \underline{0.868}$\pm$0.003 & \underline{0.820}$\pm$0.004 \\
  & Symile & 0.789$\pm$0.006             & 0.780$\pm$0.007 \\
  & CLIP++ & 0.853$\pm$0.005             & 0.743$\pm$0.003 \\
\midrule
\multirow{2}{*}{\textbf{Bi-modal}}
  & FMCA   & 0.845$\pm$0.005             & \underline{0.795}$\pm$0.005 \\
  & CLIP   & 0.843$\pm$0.003             & 0.790$\pm$0.006 \\
\midrule
\multirow{3}{*}{\textbf{Uni-modal}}
  & Barlow Twins & 0.638$\pm$0.008 & 0.610$\pm$0.009 \\
  & SimCLR       & 0.655$\pm$0.011 & 0.630$\pm$0.011 \\
  & VICReg       & 0.482$\pm$0.009 & 0.550$\pm$0.010 \\
\midrule
\multirow{3}{*}{\textbf{Supervised}}
  & HyperFuseNet & \textbf{0.887}$\pm$0.005 & \textbf{0.887}$\pm$0.005 \\
  & EEGNet       & 0.850$\pm$0.006          & 0.750$\pm$0.006 \\
  & CNN          & 0.720$\pm$0.009          & 0.650$\pm$0.008 \\
\bottomrule
\end{tabular}
\end{minipage}
\hfill
\begin{minipage}{0.48\textwidth}
\centering
\begin{tabular}{llcc}
\toprule
\textbf{Category} & \textbf{Method} &
\textbf{EDA} Subj.-Indep. & \textbf{BVP} Subj.-Indep. \\
\midrule
\rowcolor{blue!10}
\cellcolor{white}\multirow{3}{*}{\textbf{Tri-modal}}
  & MFMC   & \textbf{0.331}$\pm$0.020 & \textbf{0.311}$\pm$0.028 \\
  & Symile & 0.275$\pm$0.024          & 0.270$\pm$0.024 \\
  & CLIP++ & 0.292$\pm$0.028          & \underline{0.303}$\pm$0.030 \\
\midrule
\multirow{2}{*}{\textbf{Bi-modal}}
  & FMCA   & 0.291$\pm$0.025 & 0.280$\pm$0.024 \\
  & CLIP   & 0.302$\pm$0.030 & 0.292$\pm$0.032 \\
\midrule
\multirow{3}{*}{\textbf{Uni-modal}}
  & Barlow Twins & 0.283$\pm$0.028 & 0.270$\pm$0.028 \\
  & SimCLR       & \underline{0.328}$\pm$0.030 & 0.300$\pm$0.031 \\
  & VICReg       & 0.291$\pm$0.026 & 0.282$\pm$0.026 \\
\midrule
\multirow{3}{*}{\textbf{Supervised}}
  & HyperFuseNet & 0.291$\pm$0.030 & 0.291$\pm$0.030 \\
  & EEGNet       & 0.295$\pm$0.023 & 0.285$\pm$0.024 \\
  & CNN          & 0.270$\pm$0.028 & 0.275$\pm$0.026 \\
\bottomrule
\end{tabular}
\end{minipage}
\label{tab:eda_bvp_ceap}
\end{table*}

\begin{table*}[t]
\centering
\caption{Subject-dependent (left) and subject-independent (right) accuracy (\%) on \textbf{MAHNOB HCI-Tagging} dataset using EEG and ECG.}
\begin{minipage}{0.48\textwidth}
\centering
\begin{tabular}{llcc}
\toprule
\textbf{Category} & \textbf{Method} &
\textbf{EEG} Subj.-Dep. & \textbf{ECG} Subj.-Dep. \\
\midrule
\rowcolor{blue!10}
\cellcolor{white}\multirow{3}{*}{\textbf{Tri-modal}}
  & MFMC   & \underline{0.953}$\pm$0.003 & \underline{0.910}$\pm$0.004 \\
  & Symile & 0.882$\pm$0.006             & 0.850$\pm$0.006 \\
  & CLIP++ & 0.853$\pm$0.009             & 0.842$\pm$0.002 \\
\midrule
\multirow{2}{*}{\textbf{Bi-modal}}
  & FMCA   & 0.889$\pm$0.005             & 0.820$\pm$0.005 \\
  & CLIP   & 0.927$\pm$0.015             & 0.880$\pm$0.010 \\
\midrule
\multirow{3}{*}{\textbf{Uni-modal}}
  & Barlow Twins & 0.935$\pm$0.008 & 0.870$\pm$0.009 \\
  & SimCLR       & 0.931$\pm$0.011 & 0.860$\pm$0.011 \\
  & VICReg       & 0.822$\pm$0.009 & 0.780$\pm$0.010 \\
\midrule
\multirow{3}{*}{\textbf{Supervised}}
  & HyperFuseNet & \textbf{0.955}$\pm$0.005 & \textbf{0.955}$\pm$0.005 \\
  & EEGNet       & 0.920$\pm$0.007          & 0.855$\pm$0.007 \\
  & CNN          & 0.902$\pm$0.010          & 0.800$\pm$0.009 \\
\bottomrule
\end{tabular}
\end{minipage}
\hfill
\begin{minipage}{0.48\textwidth}
\centering
\begin{tabular}{llcc}
\toprule
\textbf{Category} & \textbf{Method} &
\textbf{EEG} Subj.-Indep. & \textbf{ECG} Subj.-Indep. \\
\midrule
\rowcolor{blue!10}
\cellcolor{white}\multirow{3}{*}{\textbf{Tri-modal}}
  & MFMC   & \underline{0.442}$\pm$0.024 & \textbf{0.441}$\pm$0.022 \\
  & Symile & \textbf{0.450}$\pm$0.034    & \underline{0.430}$\pm$0.025 \\
  & CLIP++ & 0.432$\pm$0.031             & 0.352$\pm$0.040 \\
\midrule
\multirow{2}{*}{\textbf{Bi-modal}}
  & FMCA   & 0.435$\pm$0.025 & 0.410$\pm$0.026 \\
  & CLIP   & 0.441$\pm$0.029 & 0.425$\pm$0.028 \\
\midrule
\multirow{3}{*}{\textbf{Uni-modal}}
  & Barlow Twins & 0.404$\pm$0.028 & 0.380$\pm$0.025 \\
  & SimCLR       & 0.348$\pm$0.031 & 0.350$\pm$0.020 \\
  & VICReg       & 0.431$\pm$0.026 & 0.400$\pm$0.028 \\
\midrule
\multirow{3}{*}{\textbf{Supervised}}
  & HyperFuseNet & 0.320$\pm$0.026 & 0.320$\pm$0.026 \\
  & EEGNet       & 0.410$\pm$0.026 & 0.395$\pm$0.027 \\
  & CNN          & 0.390$\pm$0.028 & 0.370$\pm$0.030 \\
\bottomrule
\end{tabular}
\end{minipage}
\label{tab:eeg_ecg_mahnob}
\end{table*}

\subsection{Comparison with SOTA methods}
\subsubsection{\textbf{Comparison with task-specific SOTA}}
HyperFuseNet~\cite{lopez2023hypercomplex} and related supervised fusion models are strong task-specific SOTA baselines on multimodal physiological emotion recognition. Our results show that MFMC, despite being self-supervised and using an augmentation-light objective, matches or closely approaches HyperFuseNet in the subject-dependent setting and surpasses it in most subject-independent scenarios. Compared to previously reported results on DEAP, CEAP-360VR, and MAHNOB-HCI under similar label definitions and modality subsets, MFMC therefore offers a competitive and often superior alternative while requiring no affect labels during pretraining. This suggests that DTC-grounded multimodal self-supervision can serve as a powerful drop-in replacement for heavily engineered, label-intensive pipelines.

\begin{table}[t]
\centering
\caption{Comparison between MFMC (self-supervised) and the best supervised baseline on the representative target modality for each dataset for 4-class valence/arousal classification.}
\begin{tabular}{llccc}
\toprule
\textbf{Protocol} & \textbf{Dataset} & \textbf{MFMC} & \textbf{Best Sup.} & $\Delta$ (MFMC--Sup.) \\
\midrule
\multirow{3}{*}{Subj.-Dep.}
  & DEAP        & 0.987 & 0.995 & -0.008 \\
  & CEAP-360VR  & 0.868 & 0.887 & -0.019 \\
  & MAHNOB-HCI  & 0.953 & 0.955 & -0.002 \\
\midrule
\multirow{3}{*}{Subj.-Indep.}
  & DEAP        & 0.346 & 0.341 & +0.005 \\
  & CEAP-360VR  & 0.331 & 0.295 & +0.036 \\
  & MAHNOB-HCI  & 0.442 & 0.410 & +0.032 \\
\bottomrule
\end{tabular}
\label{tab:mfmc_vs_sup}
\end{table}

\subsubsection{\textbf{Comparison with supervised learning}}
\label{sec:mfmc_vs_supervised}

Table~\ref{tab:mfmc_vs_sup} compares MFMC with the best supervised baseline using the representative target modality for each dataset, namely EEG for DEAP and MAHNOB-HCI and EDA for CEAP-360VR. In the subject-dependent setting, MFMC attains accuracies that are consistently within 0.2--1.9 percentage points of the strongest supervised model (e.g., 0.987 vs.\ 0.995 on DEAP and 0.953 vs.\ 0.955 on MAHNOB-HCI), indicating that self-supervised pretraining recovers almost all of the subject-specific discriminative power of fully supervised training. In the more challenging subject-independent setting, MFMC not only closes this gap but surpasses supervised learning on every dataset, reaching 0.346 vs.\ 0.341 on DEAP, 0.331 vs.\ 0.295 on CEAP-360VR, and 0.442 vs.\ 0.410 on MAHNOB-HCI. Similar trends are observed for peripheral modalities (EOG, EDA/BVP, ECG) in Tables~\ref{tab:eeg_eog_deap}--\ref{tab:eeg_ecg_mahnob}, where MFMC remains close to supervised performance on subject-dependent splits while providing consistently stronger generalization across subjects.

\subsection{Ablation Studies}

To assess the importance of the MFMC objective, we keep the entire framework unchanged---encoders, projection heads, optimizer settings, training loop and substitute the loss term by two alternatives.

\begin{itemize}[nosep,leftmargin=*]
  \item \textbf{High-order InfoNCE}  
        We use exactly the same architecture as MFMC, which essentially maximizes the sum 
        \( I(x_1, x_2; x_3) + I(x_1, x_3; x_2) + I(x_2, x_3; x_1) \). 
        However, instead of our proposed trace-based FMCA objective, we use InfoNCE to estimate the three joint mutual information terms.
  \item \textbf{FMCA-LogDet}  
        We replace the trace formulation (e.g., Eq.~(19) in the main manuscript) with the log-determinant surrogate~\cite{hu2022normalized} (e.g., Eq.~(10) in the main manuscript), with a ridge term \( \varepsilon=10^{-5} \) for stability.
\end{itemize}

Results in Table~\ref{tab:ablation_accuracy} support our claim that the DTC-grounded trace surrogate yields smoother optimization dynamics and superior downstream performance compared with high-order InfoNCE and log-determinant FMCA objectives. On the DEAP dataset, MFMC outperforms the strongest high-order InfoNCE variant by 1.1 percentage points and the FMCA-LogDet surrogate by 12.4 percentage points in the subject-dependent setting (Table~\ref{tab:ablation_accuracy}), while exhibiting smooth and reliable convergence (Fig.~\ref{fig:abl_curves}). To verify that this stability pattern is not specific to DEAP, we further provide a representative CEAP-360VR subject-dependent comparison in Appendix~C-D (Fig.~S4), where the same qualitative ordering is observed. Although InfoNCE converges stably, it saturates at a lower performance level, consistent with theoretical results showing that mutual-information lower-bound estimators tend to flatten as the true mutual information increases~\cite{mcallester2020formal}. The FMCA-LogDet objective initially follows MFMC but becomes numerically unstable as training progresses, owing to ill-conditioned covariance matrices, which causes the loss to diverge and downstream performance to collapse. In contrast, the proposed trace surrogate optimizes a simple sum of eigenvalues and avoids explicit determinants or eigendecompositions. We further provide a ridge coefficient sensitivity diagnostic in Appendix~C-F, showing that the trace-based FMCA objective is insensitive to \(\varepsilon\) in downstream accuracy and macro-F1, while ridge regularization substantially reduces the condition numbers of the empirical autocorrelation matrices.


\begin{figure}[t]
    \centering
    \includegraphics[width=\linewidth]{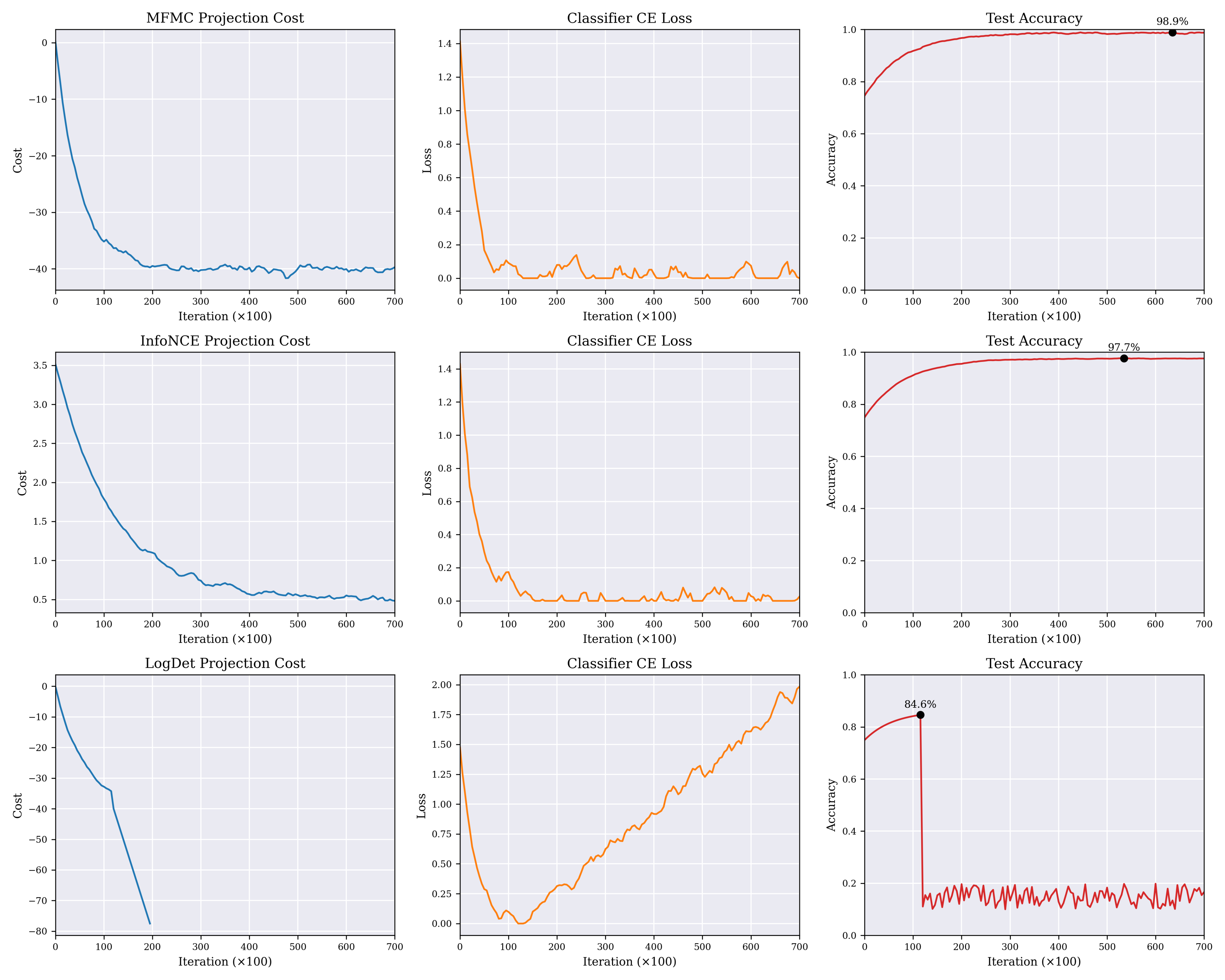}
    \caption{Validation curves on the DEAP subject-dependent split.
    MFMC converges smoothly and attains the highest accuracy.
    High-order InfoNCE plateaus early; LogDet diverges after 13\,k iterations.}
    \label{fig:abl_curves}
\end{figure}

\begin{table}[t]
\centering
\caption{Validation accuracy (\%) on DEAP (subject-dependent) for the
three objectives. All runs share identical hyper-parameters.}
\label{tab:ablation_accuracy}
\setlength{\tabcolsep}{6pt}
\begin{tabular}{lc}
\toprule
\textbf{Objective} & \textbf{Best Acc.} \\
\midrule
MFMC (trace surrogate) & \textbf{98.7} \\
High-order InfoNCE     & 97.6 \\
FMCA-LogDet            & 86.3 \\
\bottomrule
\end{tabular}
\end{table}

\section{Discussion}

Our empirical results and ablation studies support two main observations about multimodal self-supervision for physiological emotion recognition. 
First, the comparison between tri-modal, bi-modal, and uni-modal methods highlights the importance of explicitly modeling higher-order dependence across central and peripheral signals. Across all three datasets, MFMC consistently outperforms pairwise CLIP-style objectives (CLIP, CLIP++) and TC-based SymILE in both subject-dependent and subject-independent regimes (Tables~\ref{tab:eeg_eog_deap}--\ref{tab:eeg_ecg_mahnob}). These gains are particularly pronounced in the more difficult subject-independent setting, where MFMC often matches or surpasses supervised baselines such as EEGNet and HyperFuseNet, despite using no affect labels during pretraining. This pattern suggests that DTC-based training encourages encoders to capture structured higher-order interactions that emerge when EEG and peripheral signals are considered jointly, rather than over-emphasizing redundant pairwise correlations. The fact that MFMC remains augmentation-light and does not require positive/negative pair construction further indicates that much of the supervision signal is intrinsic to the multimodal physiological structure itself, rather than being injected through heavy augmentation or handcrafted contrastive schemes.

To further examine what higher-order dependencies are learned by MFMC, we provide an additional diagnostic analysis in Appendix~E. Specifically, we analyze the normalized contributions of the three cyclic pair-to-third dependence terms in Eq. (23), and further compare each fused pair-to-third dependence score with the stronger corresponding single-modality-to-third dependence score. The results show that MFMC learns distributed joint-over-single dependence in the subject-dependent setting, while the subject-independent setting exhibits more direction-specific and partially redundant cross-subject dependence. This supports the interpretation that MFMC does not merely collapse to a single pairwise alignment path, but uses the cyclic DTC-grounded objective to learn structured multimodal dependencies.

The degradation analysis in Table~\ref{tab:subject_independent_drop} further suggests that the main challenge in the subject-independent setting is not optimization failure, but reduced transferability of higher-order physiological coupling across participants. The dependence diagnostics in Appendix~E support this interpretation. In subject-dependent evaluation, the pair fusion gains are positive for all cyclic directions, indicating that fused modality pairs contain additional joint information beyond the stronger individual modality. In subject-independent evaluation, these gains become modest or partially negative, suggesting that the transferable dependence becomes more direction specific and partially redundant. Thus, the subject-independent drop reflects the weakening of cross-subject joint over single dependence, while the DTC-grounded objective still helps preserve useful higher-order structure.

Second, our results shed light on the relative contribution of architecture versus objective in physiological SSL. All methods share the same temporal CNN backbone and projection head, and differ only in loss function and modality configuration. Yet MFMC substantially narrows the gap to, or even exceeds, task-specific supervised models under cross-subject evaluation. This suggests that for noisy, label-scarce physiological data, improving the alignment objective and dependence estimator may yield larger generalization benefits than further increasing model capacity. In particular, MFMC’s reliance on global covariance structure rather than instance-wise negatives may help it average out subject-specific artifacts and label noise, which are known challenges in affective computing.



\section{Conclusions and Future Work}
This paper presents a dual total correlation (DTC)-grounded self-supervised learning framework for modeling more than two physiological modalities in emotion recognition. 
The proposed Multimodal Functional Maximum Correlation (MFMC) framework is theoretically grounded, practical without positive/negative sample construction or handcrafted augmentations, and generalizable across diverse physiological signals.
Our findings also highlight the potential of incorporating multiple modalities beyond the common bi-modal setting to enhance performance in affective BCI tasks. In future work, we plan to extend MFMC to other downstream BCI applications, such as depression severity assessment. Furthermore, we aim to develop formal methods for quantifying the contribution of each modality, as well as their higher-order interactions, to the final decision.

\section*{Acknowledgment}
This work was supported by the National Natural Science Foundation of China (Grant No. 62506070), the Natural Science Foundation of Jiangsu Province (Grants No BK20251348), the Research Project of Humanities and Social Sciences of the Ministry of Education (No. 25YJCZH372) and Joint Open Project of Southeast University-Jiangsu Province Hospital.


\appendices
\onecolumn

\setcounter{figure}{0}
\renewcommand{\thefigure}{S\arabic{figure}}
\setcounter{table}{0}
\renewcommand{\thetable}{S\arabic{table}}

\section{Proofs}
\subsection{Proof of DTC decomposition (Equation (14) in the main manuscript)}
We first present Lemma~\ref{lemma:dtc_decomposition_appendix}, adapted from Lemma 4.5 in the work of \cite{austin2018multi}.

\begin{lemma}[\cite{austin2018multi}]\label{lemma:dtc_decomposition_appendix}
If $M \geq 3$, then DTC satisfies
\begin{equation}
\mathrm{DTC}(X_1; \ldots; X_M) = \mathrm{DTC}(X_1; \ldots; X_{M-1}) + \sum_{i=1}^{M-1} I(X_i ; X_M \mid X_{[M-1] \setminus \{i\}}).
\end{equation}
\end{lemma}

When $M = 3$, we obtain:
\begin{equation}\label{eq:1}
\mathrm{DTC}(X_1, X_2, X_3) = I(X_1; X_2) + I(X_1; X_3 \mid X_2) + I(X_2; X_3 \mid X_1) 
\end{equation}

Note that DTC is symmetric and invariant under permutation. Therefore,
\begin{equation}\label{eq:2}
\mathrm{DTC}(X_1, X_2, X_3) = I(X_1; X_3) + I(X_1; X_2 \mid X_3) + I(X_3; X_2 \mid X_1) 
\end{equation}
and
\begin{equation}\label{eq:3}
\mathrm{DTC}(X_1, X_2, X_3) = I(X_2; X_3) + I(X_2; X_1 \mid X_3) + I(X_3; X_1 \mid X_2) 
\end{equation}

Combine Eqs.~(\ref{eq:1})--(\ref{eq:3}):
\begin{equation}
\begin{aligned}
3\,\mathrm{DTC}(X_1, X_2, X_3) &= I(X_1; X_2) + I(X_1; X_3) + I(X_2; X_3) \\
&\quad + 2I(X_1; X_2 \mid X_3) + 2I(X_1; X_3 \mid X_2) + 2I(X_2; X_3 \mid X_1)
\end{aligned}
\end{equation}

That is,
\begin{equation}
\begin{aligned}
\mathrm{DTC}(X_1, X_2, X_3) &= \frac{1}{3} \left( I(X_1; X_2) + I(X_1; X_3) + I(X_2; X_3) \right) \\
&\quad + \frac{2}{3} \left( I(X_1; X_2 \mid X_3) + I(X_1; X_3 \mid X_2) + I(X_2; X_3 \mid X_1) \right)
\end{aligned}
\end{equation}

\subsection{Proof of Theorem 1}

Theorem 1 is a special case of Theorem 2 with $M = 3$; we therefore focus on the proof of Theorem 2.

\subsection{Proof of Theorem 2}

We begin with the definition of TC and DTC.

\begin{equation}\label{eq:normal_TC}
\mathrm{TC}(X_1, X_2, \ldots, X_M) = \sum_{i=1}^M H(X_i) - H(X_1, X_2, \ldots, X_M),
\end{equation}

\begin{equation}\label{eq:normal_DTC}
\mathrm{DTC}(X_1, X_2, \ldots, X_M) = H(X_1, X_2, \ldots, X_M) - \sum_{i=1}^M H(X_i \mid X_1, \ldots, X_{i-1}, X_{i+1}, \ldots, X_M)
\end{equation}

Combine Eq.~(\ref{eq:normal_TC}) and Eq.~(\ref{eq:normal_DTC}), we obtain:
\begin{equation}
\mathrm{TC}(X_1, X_2, \ldots, X_M) + \mathrm{DTC}(X_1, X_2, \ldots, X_M) = \sum_{i=1}^M \left[ H(X_i) - H(X_i \mid X_1, \ldots, X_{i-1}, X_{i+1}, \ldots, X_M) \right] 
\end{equation}

That is,
\begin{equation}
\label{eq:sum_TC_DTC}
\begin{aligned}
\mathrm{TC}(X_1, X_2, \ldots, X_M) + \mathrm{DTC}(X_1, X_2, \ldots, X_M) & = \sum_{i=1}^M I(X_i ; X_1, \ldots, X_{i-1}, X_{i+1}, \ldots, X_M) \\
& = \sum_{i=1}^M I(X_i ; X_{[M] \setminus \{i\}}),
\end{aligned}
\end{equation}
where $[M] \setminus \{i\}$ denotes the full set $[M] := \{1, 2, \ldots, M\}$ excluding $i$.

We now present Lemma~\ref{lemma_2}, adapted from Lemma 4.13 in~\cite{austin2018multi}.

\begin{lemma}[\cite{austin2018multi}]\label{lemma_2}
The $\mathrm{TC}$ and $\mathrm{DTC}$ of $X_1, X_2, \ldots, X_M$ both lie between
\[
\max_i I(X_i ; X_{[M] \setminus \{i\}}) \quad \text{and} \quad (M-1) \cdot \max_i I(X_i ; X_{[M] \setminus \{i\}})
\]
In particular,
\begin{align}
\mathrm{DTC} &\leq (M - 1)\, \mathrm{TC} \label{eq:first}  \\
\mathrm{TC} &\leq (M - 1)\, \mathrm{DTC} \label{eq:second}
\end{align}
and the quantities $\mathrm{TC}$, $\mathrm{DTC}$ and $\max_i I(X_i ; X_{[M] \setminus \{i\}})$ are either all finite or all infinite.
\end{lemma}

Combine Eq.~(\ref{eq:sum_TC_DTC}) with Eq.~(\ref{eq:first}), we obtain:
\begin{equation}
\mathrm{DTC} \leq \frac{M-1}{M} \sum_{i=1}^M I(X_i ; X_{[M] \setminus \{i\}}) 
\end{equation}

Combine Eq.~(\ref{eq:sum_TC_DTC}) with Eq.~(\ref{eq:second}), we obtain:
\begin{equation}
\mathrm{DTC} \geq \frac{1}{M} \sum_{i=1}^M I(X_i ; X_{[M] \setminus \{i\}}) 
\end{equation}

\section{Validating DTC and its Approximation}

We first justify our theoretical results through two simulations. Our first goal is to verify the underlying intuition and assess the tightness of the sandwich bound described in Theorem 1. We begin with a Gaussian setting, assuming \( X_1, X_2, X_3 \sim \mathcal{N}(0, \Sigma) \), where \( \Sigma \) has ones on the diagonal and all the off-diagonal elements equal to \( \rho \). 

\paragraph{Closed-form expression of Gaussian DTC} For three Gaussian variables $x_1, x_2, x_3 \sim \mathcal{N}(0, \Sigma)$, with the covariance matrix
\[
    \Sigma = 
    \begin{bmatrix}
    1 & \rho & \rho \\
    \rho & 1 & \rho \\
    \rho & \rho & 1
    \end{bmatrix}.
\]

Each bivariate marginal (e.g., $H(x_1, x_2)$) has:
\[
\Sigma_{\text{pair}} = 
\begin{bmatrix}
1 & \rho \\
\rho & 1
\end{bmatrix}, \quad 
|\Sigma_{\text{pair}}| = 1 - \rho^2
\]

Let $\mathbf{x} \in \mathbb{R}^d$ be a multivariate Gaussian random vector with mean $\boldsymbol{\mu}$ and covariance matrix $\Sigma$, i.e., $\mathbf{x} \sim \mathcal{N}(\boldsymbol{\mu}, \Sigma)$. The joint entropy of $\mathbf{x}$ is given by:
\[
H(\mathbf{x}) = \frac{1}{2} \log \left( (2\pi e)^d \, |\Sigma| \right).
\]

Then:
\begin{equation}
H(x_i, x_j) = \frac{2}{2} \log(2\pi e) + \frac{1}{2} \log(1 - \rho^2) 
= \log(2\pi e) + \frac{1}{2} \log(1 - \rho^2)
\end{equation}

We also have:
\[
|\Sigma| = 1 + 2\rho^3 - 3\rho^2
\]

So:
\begin{equation}
H(x_1, x_2, x_3) = \frac{3}{2} \log(2\pi e) + \frac{1}{2} \log(1 + 2\rho^3 - 3\rho^2)
\end{equation}

Therefore:
\begin{equation}
\begin{aligned}
DTC(x_1,x_2,x_3) & = H(x_1, x_2) + H(x_1, x_3) + H(x_2, x_3) - 2H(x_1, x_2, x_3) \\
& = \left[3 \log(2\pi e) + \frac{3}{2} \log(1 - \rho^2)\right] 
- \left[3 \log(2\pi e) + \log(1 + 2\rho^3 - 3\rho^2)\right] \\
& = \frac{3}{2} \log(1 - \rho^2) - \log(1 + 2\rho^3 - 3\rho^2).
\end{aligned}
\end{equation}

With this setup, both the DTC and \( I(\text{pair}; \text{third}) \), can be computed in closed form:
\begin{align*}
\text{DTC} &= 1.5 \log(1 - \rho^2) - \log(1 + 2\rho^3 - 3\rho^2) \\
I(\text{pair};\text{third}) &= 0.5\log(1 - \rho^2) - 0.5\log(1 + 2\rho^3 - 3\rho^2).
\end{align*}

Fig.~\ref{fig:synthetic}(a) illustrates that both upper and lower bounds on DTC are valid and increase with $\rho$. 

We then evaluate the validity of our generalized bound in Theorem 2 under two challenging scenarios. For \emph{Data~A}: $X_1$ is functionally related to the other dimensions by $X_1 = \left(\frac{1}{M-1} \sum_{i=2}^{M} X_i\right)^2$, where ${X_2,\dots,X_M}$ are independently and uniformly distributed. As $M$ increases, the total dependence weakens. For \emph{Data~B}: $X_1$ is uniformly distributed in $[0,1]$, and $X_i = (X_1)^2+X_1$ for $i = 2, \dots, M$. The total dependence is close to a constant, as ${X_2, \dots, X_M}$ are nonlinear functions of $X_1$. Since closed-form expressions for DTC and $I(X_{[M]\setminus i}; X_i)$ are unavailable, we estimate both terms using a modern matrix-based entropy functional~\cite{yu2021measuring}. 

\paragraph{Estimator in \emph{Data A} and \emph{Data B}}
For \emph{Data A} and \emph{Data B}, no closed-form expression for the DTC is available. Therefore, we adopt a modern sample-based estimator. Specifically, we estimate the dependence among the $M$-dimensional components of the random variable $\mathbf{x} = [x_1; x_2; \cdots; x_M] \in \mathbb{R}^M$, based solely on $N$ i.i.d.\ samples from $\mathbf{y}$, i.e., $\{\mathbf{x}^i\}_{i=1}^N$ (We use superscripts to denote sample indices and subscripts to denote variable indices, e.g., $x_1^2$ refers to the 2nd observation of the 1st dimension of $\mathbf{x}$).

In our experiment, $N=500$.

The DTC can be empirically estimated as~\cite{yu2021measuring}:
\begin{equation}\label{eq:NDTC_renyi}
DTC_{\alpha} (\mathbf{x}) =
\left[\sum_{i=1}^{M}S_\alpha\left(A_{[M]\setminus i}\right)\right]-(M-1)S_{\alpha}\left(A_{[M]}\right),
\end{equation}
where
\begin{equation}\label{eq:renyi_joint_full}
\mathbf{S}_\alpha(A_{[M]})=\mathbf{S}_\alpha\left(\frac{A_1\circ A_2\circ\cdots\circ A_L}{\mathrm{tr}(A_1\circ A_2\circ\cdots\circ A_L)}\right),
\end{equation}
and
\begin{equation}\label{eq:renyi_joint_marginal}
\mathbf{S}_\alpha(A_{[M]\setminus i})=\mathbf{S}_\alpha\left(\frac{A_1\circ \cdots A_{i-1}\circ A_{i+1}\circ \cdots \circ A_M}{\mathrm{tr}(A_1\circ \cdots A_{i-1}\circ A_{i+1}\circ \cdots \circ A_M)}\right),
\end{equation}
with
$(A_1)_{ij}=\kappa(x_1^i,x_1^j)$, $(A_2)_{ij}=\kappa(x_2^i,x_2^j)$, $\cdots$, $(A_M)_{ij}=\kappa(x_M^i,x_M^j)$,
$\kappa:\mathcal{Y}\times\mathcal{Y}\mapsto\mathbb{R}$ is a real-valued positive definite kernel such as Gaussian and $\circ$ denotes the Hadamard product.

Here, $S_\alpha(A)$ is referred to as the matrix-based entropy functional~\cite{giraldo2014measures}, defined as:
\begin{equation}
    S_\alpha(A) = \frac{1}{1-\alpha} \log_2 \left(\text{tr}(A^\alpha) \right) = \frac{1}{1-\alpha} \log_2 \left( \sum_{i=1}^N \lambda_i^\alpha \right),
\end{equation}
where $\lambda_i$ denotes the $i$-th eigenvalue of matrix $A$ (after trace normalization, i.e., $A=A/\text{tr}(A)$). 

Results in Figs.~\ref{fig:synthetic}(b) and (c) show that our bound is tight for $M=3$ (three modalities) but grows with increasing $M$. The lower bound consistently tracks the DTC, which is desirable as our experiments focus on three modalities and aim to maximize DTC. 

In our final simulation, we evaluate the strength of dependence between \( X_1 \) and \( X_2 \) using FMCA and InfoNCE separately. The data are drawn from a 20-dimensional Gaussian with zero mean and correlation coefficient \( \rho \), where the true mutual information is given by \( -10 \log(1 - \rho^2) \), which increases logarithmically with \( \rho \). As shown in Fig.~\ref{fig:synthetic}(d), FMCA accurately captures the increasing dependence as \( \rho \) grows. In contrast, InfoNCE estimates tend to saturate as mutual information increases, a known limitation of lower-bound-based estimators~\cite{mcallester2020formal}.

\begin{figure}[!t]
\centering
\subfloat[Gaussian variable.]{%
  \includegraphics[width=0.45\linewidth]{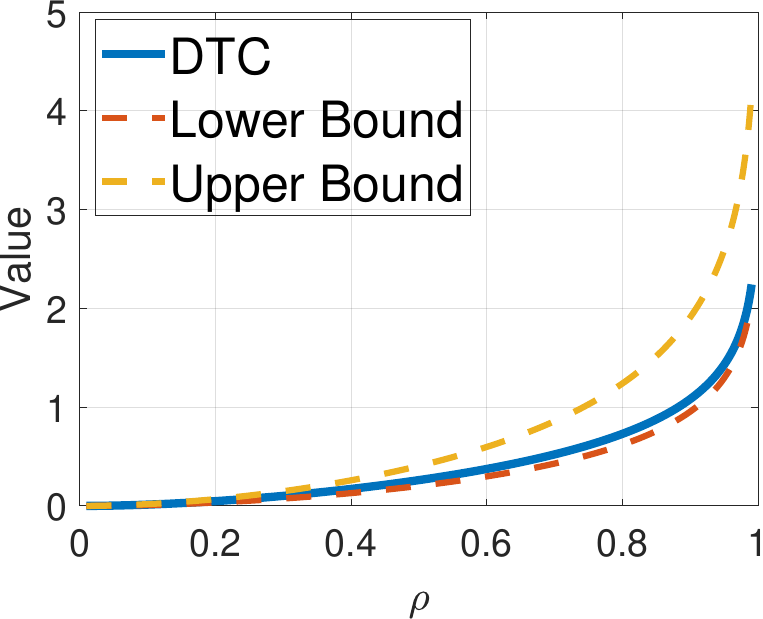}%
  \label{fig:gaussian}}
\hfil
\subfloat[\emph{Data~A}.]{%
  \includegraphics[width=0.45\linewidth]{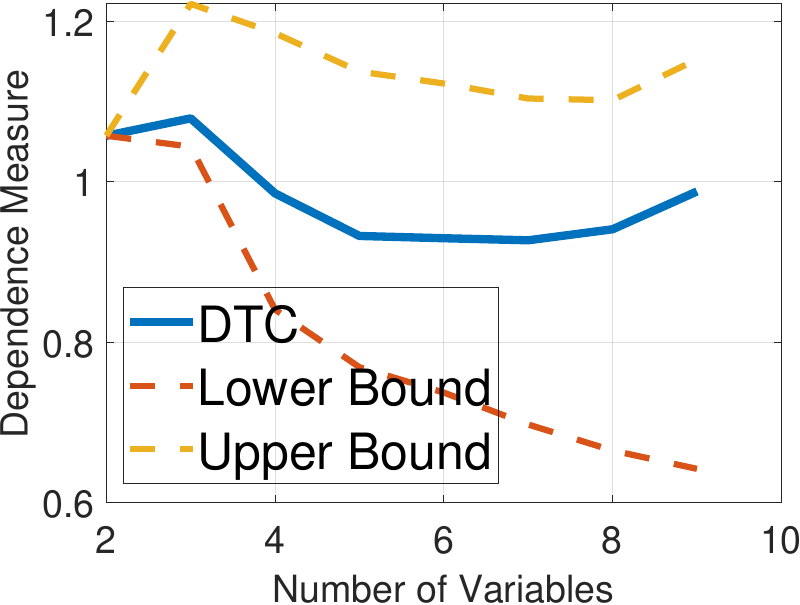}%
  \label{fig:dataA}}
\par\medskip
\subfloat[\emph{Data~B}.]{%
  \includegraphics[width=0.45\linewidth]{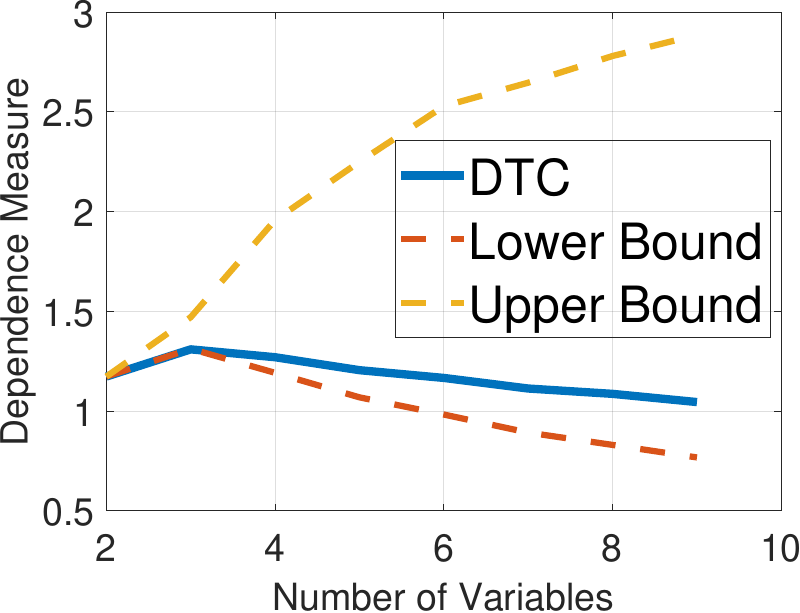}%
  \label{fig:dataB}}
\hfil
\subfloat[FMCA vs. InfoNCE.]{%
  \includegraphics[width=0.45\linewidth]{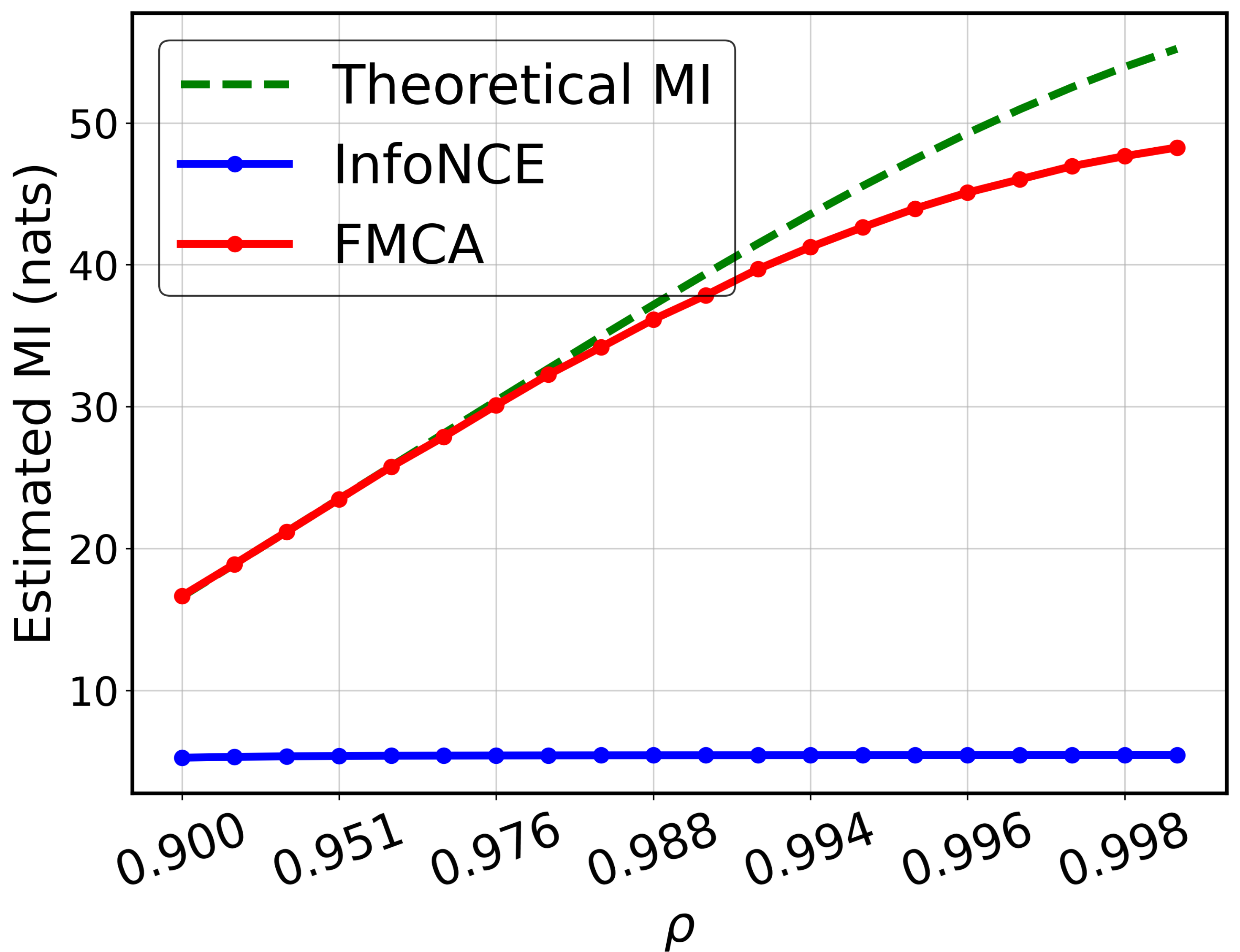}%
  \label{fig:fmcainfonce}}
\caption{Simulation results.}
\label{fig:synthetic}
\end{figure}

\section{Experiment Details}

\subsection{Window Size and Modality (Signal) Selection}

Selecting an appropriate temporal window length is crucial in
multimodal emotion recognition. Recent surveys report that most
studies operate with windows between 1-20 s
\cite{lopez2024,saganowski2022emotion}. Because the
optimal length is data- and model-dependent, we conducted a dedicated
\emph{pre-experiment} on the \textbf{DEAP} corpus to identify the best
setting for our pipeline.

\subsubsection{Dataset}
We used the official \texttt{Data\_preprocessed\_python.zip} release of
DEAP, which already contains the 40 synchronous channels
(32 EEG, 2 EOG, 2 EMG, 1 GSR, 1 respiration belt,
1 plethysmograph, 1 skin temperature) down-sampled to 128 Hz,
band-pass filtered to 4-45 Hz, and baseline-corrected
(\textit{-3 to 0 s} removed).

\subsubsection{Protocol}
Each trial was further segmented into windows of
$\{0.5, 1, 2, \dots, 20\}$ s.
To keep the number of training samples comparable across scales we
used non-overlapping windows when $L\le5$ s and an
adaptive overlap when $L>5$ s.
For every window length we trained exactly the same encoder
architecture and lightweight classifier employed in our main
experiments; window length was the only
variable factor.
Performance was assessed with five fold
subject-independent cross-validation, reporting balanced
accuracy (Acc) and macro-$F_{1}$.

\subsubsection{Results}
Both metrics rise steeply up to 10 s, where they attain their global
maximum, plateau between 10-14 s with no significant difference
($p>0.05$), and then decline.
Very short windows ($<3$ s) yield a drop of \(\ge8\) pp in macro-$F_{1}$,
indicating insufficient temporal context.
Figure 3 illustrates the trend.

\subsubsection{Adopted setting}
We therefore fix the window length to
\textbf{10 s} in all subsequent experiments—a choice that yields the
best empirical performance on DEAP and lies comfortably within the
range recommended by prior work.

The three benchmarks supply more physiological signals than are strictly
necessary for emotion recognition. To keep the pipeline lightweight and
interpretable we performed a second, modality-selection
pre-experiment.

\subsubsection{Method}
Using the 10 s window length identified earlier, we trained the same encoder as in all
main experiments, but appended a learnable modality-attention
mask. Concretely, after global spatial-temporal pooling the model
produces a $d$-dimensional embedding for every modality~$m$; a
trainable scalar gate~$a_m\in[0,1]$ (softmax normalized) multiplies
that embedding before the classifier. The gates are optimized jointly
with the network via cross-entropy. After convergence we rank
modalities by the final attention weights $\{a_m\}$ and retain the
highest-scoring subset for each dataset.

More specifically, the selection protocol is implemented at the modality group level rather than at the individual channel level. For each dataset, all available candidate modality groups are first encoded by the same encoder family used in the main experiments, producing one 128-dimensional embedding \(e_m\) for each modality group \(m\). A trainable scalar logit \(\alpha_m\) is assigned to each modality group. The logits are normalized across all candidate groups by a softmax operation, \(w_m=\exp(\alpha_m)/\sum_j\exp(\alpha_j)\). The weighted embedding \(w_m e_m\) is then passed to the classifier. All logits are initialized equally, i.e., \(\alpha_m=0\) for all \(m\), and are optimized jointly with the classifier using cross entropy loss. No manual modality weights or hard coded thresholds are used.

After convergence, candidate modality groups are ranked by their learned softmax weights. We do not use an absolute value threshold such as \(w_m>\tau\). Instead, because the downstream MFMC experiments are conducted with three active modalities, we use a top 3 ranking rule and retain the three modality groups with the largest learned weights. The modality attention model is used only to determine the active modality set and is discarded before MFMC self supervised pretraining and downstream evaluation. The same protocol is applied to DEAP, CEAP-360VR, and MAHNOB-HCI; only the candidate modality groups differ across datasets.

\begin{table}[t]
\centering
\small
\caption{Parameter settings of the modality selection protocol.}
\label{tab:modality_selection_settings}
\begin{tabular}{ll}
\hline
Item & Setting \\
\hline
Selection unit & Modality group \\
Gate parameter & One trainable scalar logit \(\alpha_m\) per modality group \\
Gate normalization & Softmax over all candidate modality groups \\
Logit initialization & \(\alpha_m=0\) for all modality groups \\
Modality embedding dimension & 128 \\
Optimization objective & Cross entropy classification loss \\
Manual absolute threshold & Not used \\
Decision rule & Select the top 3 groups ranked by learned weights \\
Reason for top 3 & Downstream MFMC experiments use three active modalities \\
Dataset scope & Same protocol for DEAP, CEAP-360VR, and MAHNOB-HCI \\
Use after selection & Gate model discarded; selected modalities used for MFMC \\
\hline
\end{tabular}
\end{table}

\subsubsection{DEAP \cite{deap}}
Seven signals (32-channel EEG, horizontal and vertical EOG, Zygomaticus and Trapezius EMG, GSR, respiration belt, plethysmograph, skin temperature) are available.
The mask assigns the largest weights to all 32 EEG leads, both
horizontal and vertical EOG, and the skin-temperature trace, resulting
in 35 retained inputs. Peripheral EMG, GSR, plethysmograph and
respiration receive substantially lower weights and are discarded.

\subsubsection{CEAP-360VR \cite{ceap}}
This corpus contains only peripheral physiology sampled at
approximately 30 Hz: three-axis accelerometer, BVP, EDA, SKT, heart rate (HR)
and inter-beat interval (IBI). The attention mask ranks
BVP, EDA and SKT highest; those three signals are therefore
kept for downstream experiments.

\subsubsection{MAHNOB-HCI \cite{Soleymani2012MAHNOB}}
MAHNOB logs 32-channel EEG plus four peripheral modalities: ECG, GSR, respiration amplitude and skin temperature.
The mask highlights EEG, GSR and
ECG, so we restrict the input to these three.

This automatic gating reduces the input dimensionality by 12-60 %
(depending on the dataset) while preserving or slightly improving
validation performance and yields a clean, interpretable set of
modalities for all subsequent analyses.

9) Preprocessing sensitivity on DEAP:

To assess the robustness of the released DEAP preprocessing pipeline, we conducted a supplementary one-factor sensitivity study that varies one preprocessing component at a time relative to the default DEAP configuration: window length \{2, 5, 10, 15\} s, stride \{0.4, 1, 2, 10\} s, normalization \{legacy constant, first-sample, per-window z-score\}, and outlier threshold \{none, ±3, ±5, ±7\}. All other settings were fixed to the released DEAP preprocessing pipeline (3 s offset, EEG/EOG/SKT inputs, current encoder, and current training budget). For computational efficiency, we used one fixed subject-dependent 80/20 split and one fixed subject-independent split, and report the highest test accuracy and macro-F1 observed during training. These results are intended to characterize qualitative robustness trends rather than to replace the official 5-fold cross-validation results reported in the main manuscript.

Overall, the DEAP sensitivity study indicates that the released preprocessing pipeline is a reasonable operating point rather than a narrowly tuned choice. The results show that very short windows are clearly suboptimal, that coarse strides degrade performance more noticeably than short strides, and that both normalization and outlier filtering materially affect retained sample count and downstream accuracy. Given the use of single fixed splits, especially under the more variable subject-independent protocol, these results should be interpreted as qualitative robustness trends rather than fold-averaged estimates.

\begin{table*}[!t]
\centering
\caption{DEAP preprocessing sensitivity study using one fixed subject-dependent split and one fixed subject-independent split. Results report the highest test accuracy / macro-F1 observed during training and are intended only to summarize qualitative robustness trends rather than to replace the official 5-fold cross-validation results.}
\label{tab:deap_preprocessing_sensitivity}
\begin{tabular}{lllll}
\toprule
Factor & Setting & Kept windows & Subj.-Dep. (Acc / F1) & Subj.-Indep. (Acc / F1) \\
\midrule
Baseline & 10 s / 0.4 s / legacy / ±5 & 20,097 & 0.987 / 0.987 & 0.353 / 0.289 \\
\midrule
Window length & 2 s & 16,527 & 0.527 / 0.517 & 0.356 / 0.262 \\
 & 5 s & 24,036 & 0.841 / 0.841 & 0.377 / 0.275 \\
 & 10 s & 20,097 & 0.993 / 0.994 & 0.351 / 0.263 \\
 & 15 s & 13,770 & 1.000 / 1.000 & 0.306 / 0.205 \\
\midrule
Stride & 0.4 s & 20,097 & 0.988 / 0.987 & 0.409 / 0.266 \\
 & 1 s & 8,111 & 0.922 / 0.923 & 0.417 / 0.309 \\
 & 2 s & 4,120 & 0.647 / 0.644 & 0.337 / 0.203 \\
 & 10 s & 946 & 0.442 / 0.429 & 0.336 / 0.246 \\
\midrule
Normalization & Legacy constant & 20,097 & 0.989 / 0.989 & 0.320 / 0.251 \\
 & First-sample & 311 & 0.333 / 0.270 & 0.361 / 0.280 \\
 & Per-window z-score & 7,579 & 0.971 / 0.970 & 0.240 / 0.190 \\
\midrule
Outlier threshold & None & 161,280 & 0.568 / 0.557 & 0.318 / 0.274 \\
 & ±3 & 5,576 & 0.997 / 0.998 & 0.382 / 0.286 \\
 & ±5 & 20,097 & 0.987 / 0.987 & 0.316 / 0.225 \\
 & ±7 & 38,679 & 0.735 / 0.730 & 0.394 / 0.288 \\
\bottomrule
\end{tabular}

Note: The baseline corresponds to the released DEAP preprocessing pipeline: 10 s windows, 3 s offset, 0.4 s stride, legacy constant normalization, and an outlier threshold of ±5.
\end{table*}

\subsection{Comparison Implementation Details}

This section details the data-augmentation recipes used in the self-supervised (SSL) experiments on the \textbf{CEAP}, \textbf{MAHNOB-HCI}, and \textbf{DEAP} corpora under three popular frameworks: \emph{SimCLR}, \emph{VICReg}, and \emph{Barlow Twins}. For clarity we first describe the per-dataset pipelines, then summarize all settings in Table~\ref{tab:ssl_aug}. Unless otherwise noted, augmentations are applied sequentially to create two transformed views of every input window.

\subsubsection{CEAP-360VR}
\begin{itemize}
  \item \textbf{SimCLR.} Physio Augmentation = Gaussian noise ($\sigma = 0.05$) \(\rightarrow\) random shift ($\pm10$ samples) \(\rightarrow\) channel dropout ($p = 0.1$)\,[1].
  \item \textbf{VICReg.} Identical pipeline to SimCLR (noise $0.05$, shift $\pm10$, dropout $0.1$).
  \item \textbf{Barlow Twins.} Noise with higher magnitude ($\sigma = 0.1$) + shift ($\pm10$); no channel dropout. Optional per-window $z$-scoring with outlier clipping at $\pm3\sigma$ is available but disabled.
\end{itemize}

\subsubsection{MAHNOB-HCI}
\begin{itemize}
  \item \textbf{SimCLR.} EEG Augmentation: noise ($\sigma = 0.05$), shift ($\pm10$), dropout ($p = 0.1$).
  \item \textbf{VICReg.} Matches SimCLR exactly (noise $0.05$, shift $\pm10$, dropout $0.1$).
  \item \textbf{Barlow Twins.} Noise ($\sigma = 0.1$) and shift ($\pm10$) only; channel dropout omitted. Optional clipped $z$-score normalisation (off by default).
\end{itemize}

\subsubsection{DEAP}
\begin{itemize}
  \item \textbf{SimCLR.} EEG Augmentation: Gaussian noise ($\sigma = 0.05$), shift ($\pm10$), dropout ($p = 0.1$).
  \item \textbf{VICReg.} Same settings as SimCLR (noise $0.05$, shift $\pm10$, dropout $0.1$).
  \item \textbf{Barlow Twins.} Noise ($\sigma = 0.1$) + shift ($\pm10$); no dropout. Optional per-window $z$-score with $\pm3\sigma$ clipping (disabled).
\end{itemize}

\begin{table}[t]
\centering
\caption{Self-supervised data-augmentation settings. “Opt.\ $z$-score” denotes an optional per-window $z$-normalisation with outlier clipping at $\pm3\sigma$ (disabled in all reported experiments).}
\label{tab:ssl_aug}
\setlength{\tabcolsep}{5pt}
\begin{tabular}{@{}lll@{}}
\toprule
\textbf{Dataset} & \textbf{Framework} & \textbf{Augmentations (parameters)} \\
\midrule
\multirow{2}{*}{CEAP}        & SimCLR / VICReg & Noise ($\sigma=0.05$), shift ($\pm10$), dropout ($p=0.1$)\\
                             & Barlow Twins    & Noise ($\sigma=0.1$), shift ($\pm10$), opt.\ $z$-score, no dropout\\
\midrule
\multirow{2}{*}{MAHNOB-HCI}  & SimCLR / VICReg & Noise ($\sigma=0.05$), shift ($\pm10$), dropout ($p=0.1$)\\
                             & Barlow Twins    & Noise ($\sigma=0.1$), shift ($\pm10$), opt.\ $z$-score, no dropout\\
\midrule
\multirow{2}{*}{DEAP}        & SimCLR / VICReg & Noise ($\sigma=0.05$), shift ($\pm10$), dropout ($p=0.1$)\\
                             & Barlow Twins    & Noise ($\sigma=0.1$), shift ($\pm10$), opt.\ $z$-score, no dropout\\
\bottomrule
\end{tabular}
\end{table}

\noindent\textbf{Self-supervised baseline hyperparameters.}
In addition to the data augmentation settings above, Table~\ref{tab:ssl_baseline_hyperparams} summarizes the training hyperparameters used by the self-supervised and multimodal self-supervised baselines that are reported in the main comparison table. The table reports the method specific settings used within each evaluation fold. Supervised baselines follow separate supervised training protocols and are not included in this SSL pretraining hyperparameter table.

\begin{table*}[t]
\centering
\caption{Training hyperparameters of the self-supervised and multimodal self-supervised baselines reported in the main comparison table.}
\label{tab:ssl_baseline_hyperparams}
\scriptsize
\setlength{\tabcolsep}{4.2pt}
\renewcommand{\arraystretch}{1.15}
\resizebox{\textwidth}{!}{%
\begin{tabular}{p{1.8cm} p{2.6cm} p{2.5cm}@{\hspace{2pt}}p{1.6cm} p{2.5cm} p{2.9cm} p{3.2cm} p{3.0cm}}
\hline
Baseline & Datasets / target modality & Active modalities & Batch size & Optimizer / LR & Training budget & Loss specific parameters & Augmentation \\
\hline
MFMC &
DEAP: EEG/EOG; CEAP-360VR: EDA/BVP; MAHNOB-HCI: EEG/ECG &
Three modalities &
256 &
Adam, LR $=2\times10^{-4}$ to $3\times10^{-4}$ &
70,001 to 80,001 total iterations &
Trace based higher order dependence objective with covariance tracking coefficient $\beta=0.5$ &
No explicit augmentation beyond the shared preprocessing \\
\hline
SymILE &
DEAP: EEG/EOG; CEAP-360VR: EDA/BVP; MAHNOB-HCI: EEG/ECG &
Three modalities &
256 &
Adam, LR $=3\times10^{-4}$ to $5\times10^{-4}$ &
70,001 to 100,001 total iterations &
Learnable logit scale initialized from temperature $0.07$ &
No explicit augmentation beyond the shared preprocessing \\
\hline
CLIP++ &
DEAP: EEG/EOG; CEAP-360VR: EDA/BVP; MAHNOB-HCI: EEG/ECG &
Three modalities &
256 &
Adam, LR $=3\times10^{-4}$ &
70,001 total iterations &
Average or sum of three pairwise CLIP losses; learnable logit scale initialized from temperature $0.07$ &
No explicit augmentation beyond the shared preprocessing \\
\hline
FMCA &
DEAP: EEG/EOG; CEAP-360VR: EDA/BVP; MAHNOB-HCI: EEG/ECG &
Two modalities &
256 &
Adam, LR $=2\times10^{-4}$ to $3\times10^{-4}$ &
70,001 to 80,001 total iterations &
Trace based dependence objective; covariance tracking coefficient $\beta=0.5$ when used &
No explicit augmentation beyond the shared preprocessing \\
\hline
CLIP &
DEAP: EEG/EOG; CEAP-360VR: EDA/BVP; MAHNOB-HCI: EEG/ECG &
Two modalities &
256 &
Adam, LR $=2\times10^{-4}$ to $3\times10^{-4}$; MAHNOB-HCI inner script uses LR $=5\times10^{-3}$ &
70,001 to 80,001 total iterations &
Learnable logit scale initialized from temperature $0.07$ &
No explicit augmentation beyond the shared preprocessing \\
\hline
Barlow Twins &
DEAP: EEG/EOG; CEAP-360VR: EDA/BVP; MAHNOB-HCI: EEG/ECG &
Single modality &
256 &
Adam, LR $=1.5\times10^{-4}$ to $5\times10^{-4}$ &
70,001 total iterations; 7k to 8k SSL pretraining iterations depending on dataset and script &
Off diagonal weight $\lambda=0.005$ &
Gaussian noise $\sigma=0.1$, temporal shift $\pm 10$; optional per window z-score disabled \\
\hline
SimCLR &
DEAP: EEG/EOG; CEAP-360VR: EDA/BVP; MAHNOB-HCI: EEG/ECG &
Single modality &
256 &
Adam, LR $=3\times10^{-4}$ &
70,001 total iterations; 7k to 8k SSL pretraining iterations depending on dataset and script &
Fixed temperature $\tau=0.5$ &
Gaussian noise $\sigma=0.05$, temporal shift $\pm 10$, channel dropout $p=0.1$ \\
\hline
VICReg &
DEAP: EEG/EOG; CEAP-360VR: EDA/BVP; MAHNOB-HCI: EEG/ECG &
Single modality &
256 &
Adam, LR $=3\times10^{-4}$ &
70,001 total iterations; 7k to 8k SSL pretraining iterations depending on dataset and script &
Invariance, variance, and covariance weights $25/25/1$; variance numerical term $10^{-4}$ &
Gaussian noise $\sigma=0.05$, temporal shift $\pm 10$, channel dropout $p=0.1$ \\
\hline
\end{tabular}%
}
\end{table*}

Supervised baseline hyperparameters. The supervised baselines, including EEGNet, the vanilla CNN, and HyperFuseNet, were trained end-to-end under the same preprocessing pipeline, modality configurations, train--test splits, five-fold evaluation protocols, and supervised training budget used in the main experiments. These supervised baselines do not include the self-supervised pretraining stage. The vanilla CNN uses the shared encoder backbone described in Section IV-B, whereas EEGNet and HyperFuseNet follow their original supervised architectures. This clarification is added to distinguish the supervised baselines from the self-supervised and multimodal self-supervised baselines summarized in Table~\ref{tab:ssl_baseline_hyperparams}.

\subsection{Ablation on Fusion Network Design}
\label{sec:fusion_ablation}

In the main manuscript, MFMC uses a lightweight MLP fusion head to construct paired modality representations for cyclic pair to third dependence estimation. To examine whether this design is sufficient, we compare the default MLP fusion head with two attention based fusion variants. This ablation evaluates whether increasing the capacity of the fusion module improves downstream performance, while keeping the MFMC objective and the cyclic projection logic unchanged.

All fusion heads take two 128 dimensional modality embeddings as input and output a 128 dimensional paired representation. The default MLP fusion first concatenates two modality embeddings into a 256 dimensional vector and then maps it to a 128 dimensional representation using a two layer MLP. The three cyclic representations, denoted as $e_{12}$, $e_{13}$, and $e_{23}$, are generated by three projection heads with the same architecture but independent parameters.

We additionally evaluate two attention based variants. Attention Fusion v1 reshapes the 256 dimensional input into two 128 dimensional tokens, maps each token to a 512 dimensional hidden space, applies 4 head self attention and a feed forward block, and then uses mean pooling over the two tokens before the final projection. Attention Fusion v2 is a lighter attention design. It maps each token to a 256 dimensional hidden space, adds a learnable modality embedding, applies 4 head self attention and a feed forward block, and then flattens the two output tokens before the final projection. Unlike v1, v2 preserves token identity through learnable modality embeddings and avoids mean pooling.

\begin{table}[!t]
\centering
\caption{Comparison of fusion head architectures. Each head maps two 128 dimensional modality embeddings to one 128 dimensional paired representation.}
\label{tab:fusion_head_architecture}
\resizebox{\linewidth}{!}{
\begin{tabular}{lcccccc}
\toprule
\textbf{Fusion head} & \textbf{Input} & \textbf{Hidden dim.} & \textbf{Main operation} & \textbf{Dropout} & \textbf{Token aggregation} & \textbf{Params, one / three heads} \\
\midrule
MLP Fusion
& 256D concat
& 512
& Two layer MLP with BatchNorm and ReLU
& 0
& None
& 198,528 / 595,584 \\

Attention Fusion v1
& Two 128D tokens
& 512
& 4 head self attention with FFN
& 0.1
& Mean pooling
& 1,710,976 / 5,132,928 \\

Attention Fusion v2
& Two 128D tokens
& 256
& 4 head self attention with FFN and modality embedding
& 0
& Flatten
& 495,744 / 1,487,232 \\
\bottomrule
\end{tabular}
}
\end{table}

\begin{figure}[!t]
\centering
\includegraphics[width=0.95\linewidth]{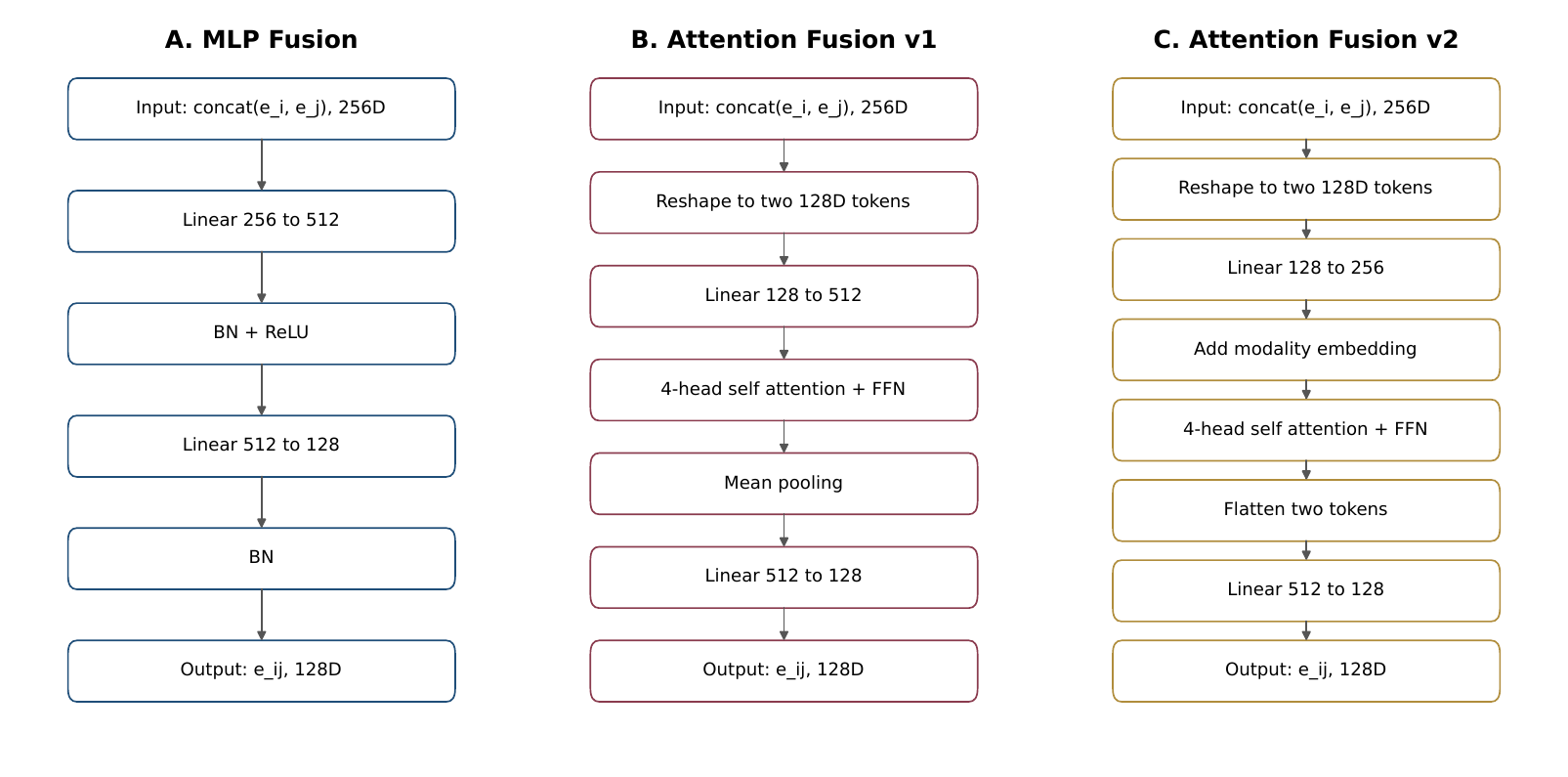}
\caption{Architecture comparison of the default MLP fusion head and two attention based fusion variants.}
\label{fig:fusion_head_architecture}
\end{figure}

We conduct this ablation on CEAP and DEAP under both subject dependent and subject independent protocols. For CEAP, the three input modalities are EDA, BVP, and SKT, and downstream performance is reported using the EDA encoder. For DEAP, the three input modalities are EEG, EOG, and temperature, and downstream performance is reported using the EEG encoder. All experiments use five folds, a batch size of 256, Adam optimizers with a learning rate of $3\times 10^{-4}$, and 20,001 training iterations. Accuracy is used as the evaluation metric.

\begin{table}[!t]
\centering
\caption{Ablation results of MLP fusion and attention based fusion heads. Accuracy is reported as mean $\pm$ standard deviation across five folds.}
\label{tab:fusion_ablation_results}
\resizebox{\linewidth}{!}{
\begin{tabular}{llcccc}
\toprule
\textbf{Dataset} & \textbf{Protocol} & \textbf{Modality} & \textbf{MLP Fusion} & \textbf{Attention Fusion v1} & \textbf{Attention Fusion v2} \\
\midrule
CEAP & Subject dependent & EDA
& $0.868 \pm 0.003$
& $0.3227 \pm 0.0082$
& $0.4992 \pm 0.0218$ \\

CEAP & Subject independent & EDA
& $0.331 \pm 0.020$
& $0.2485 \pm 0.0229$
& $0.2535 \pm 0.0163$ \\

DEAP & Subject dependent & EEG
& $0.987 \pm 0.009$
& $0.9227 \pm 0.0200$
& $0.9244 \pm 0.0215$ \\

DEAP & Subject independent & EEG
& $0.346 \pm 0.030$
& $0.2219 \pm 0.0274$
& $0.3140 \pm 0.0220$ \\
\bottomrule
\end{tabular}
}
\end{table}

For CEAP subject independent evaluation, Attention Fusion v1 uses a different subject split seed from MLP Fusion and Attention Fusion v2. Therefore, Attention Fusion v1 is included mainly as a diagnostic variant in that setting.

\begin{figure}[!t]
\centering
\subfloat[Parameter count comparison of the three fusion head designs. The reported values correspond to the total parameters of the three cyclic projection heads.]{%
  \includegraphics[width=0.48\linewidth]{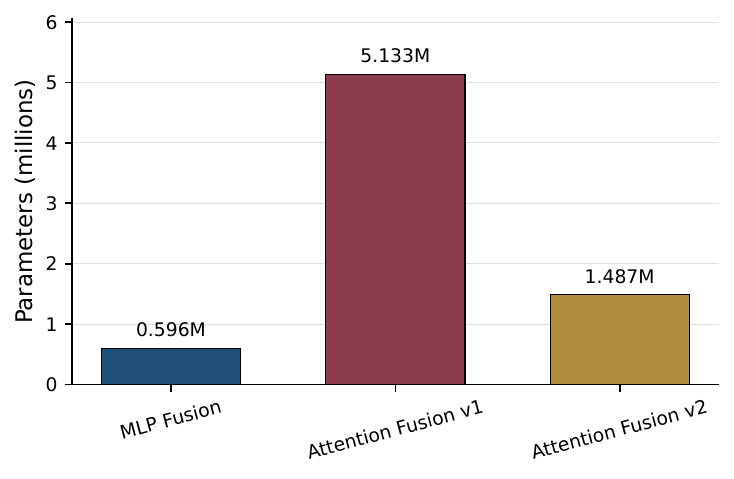}%
  \label{fig:fusion_parameter_count}}
\hfil
\subfloat[Accuracy comparison between MLP fusion and attention based fusion variants across datasets and evaluation protocols.]{%
  \includegraphics[width=0.48\linewidth]{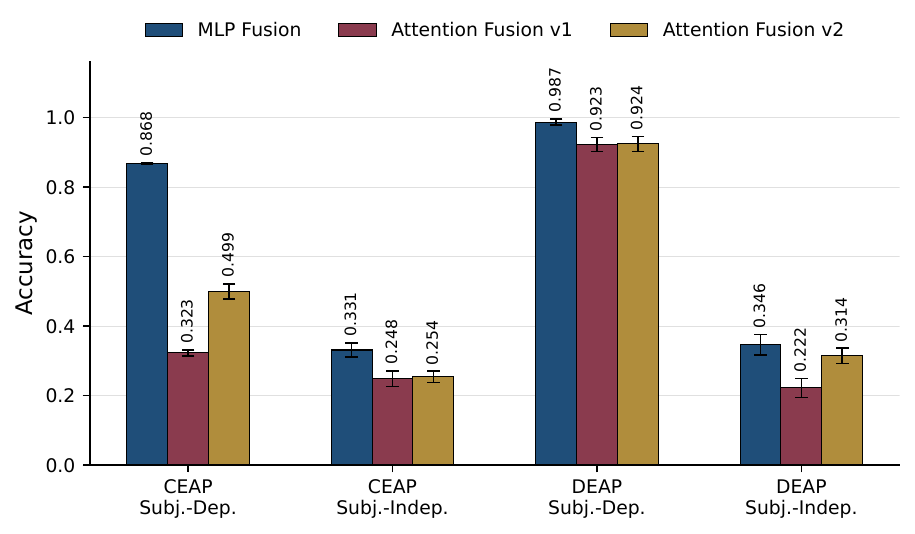}%
  \label{fig:fusion_ablation_accuracy}}
\caption{Fusion head parameter count and accuracy comparison.}
\label{fig:fusion_ablation_summary}
\end{figure}

The results show that increasing the complexity of the fusion head does not lead to consistent improvements. The default MLP fusion achieves the best performance in all four evaluated settings while using substantially fewer parameters than both attention based variants. In particular, Attention Fusion v1 contains about 8.6 times more parameters than the MLP fusion head, but performs worse across both datasets. This suggests that the larger attention head with mean pooling is not beneficial in this setting, possibly because mean pooling weakens the identity and ordering of the two modality tokens.

Attention Fusion v2 reduces the hidden dimension, removes dropout, introduces learnable modality embeddings, and preserves the two token outputs by flattening rather than mean pooling. This design improves over v1 in several settings, especially under DEAP subject independent evaluation. However, it still does not consistently outperform the default MLP fusion and uses about 2.5 times more parameters. These results indicate that MFMC does not rely on a high capacity fusion module. The fusion head mainly provides a compact parameterization of paired modality embeddings, whereas the higher order multimodal learning is driven by the cyclic DTC grounded FMCA objective.

Overall, this ablation supports the use of the default MLP fusion head in MFMC. Compared with attention based fusion, the MLP fusion head is simpler, more parameter efficient, and empirically stronger in the evaluated settings. This is consistent with the role of the fusion head in MFMC: it is used to construct paired representations for dependence estimation, while the main cross modality learning signal is provided by the cyclic DTC grounded objective.

\FloatBarrier

\subsection{Additional Stability Analysis on CEAP-360VR}

To complement the DEAP ablation in Fig.~5 of the main text, we repeated the same objective-level comparison on CEAP-360VR using the trimodal setting (EDA, BVP, and SKT). Fig.~\ref{fig:ceap_stability} shows the learning curves from a representative subject-dependent fold for MFMC, High-order InfoNCE, and FMCA-LogDet. Consistent with the DEAP results, MFMC achieves the highest downstream accuracy, while High-order InfoNCE remains competitive but plateaus below MFMC and FMCA-LogDet shows less stable optimization with clear post-peak degradation. We note that Table III in the main text reports the official CEAP-360VR results as mean $\pm$ standard deviation over five folds, whereas Fig.~\ref{fig:ceap_stability} is intended only as a qualitative stability illustration from a single representative fold and is therefore not numerically comparable one-to-one with Table III.

\begin{figure}[t]
    \centering
    \includegraphics[width=\linewidth]{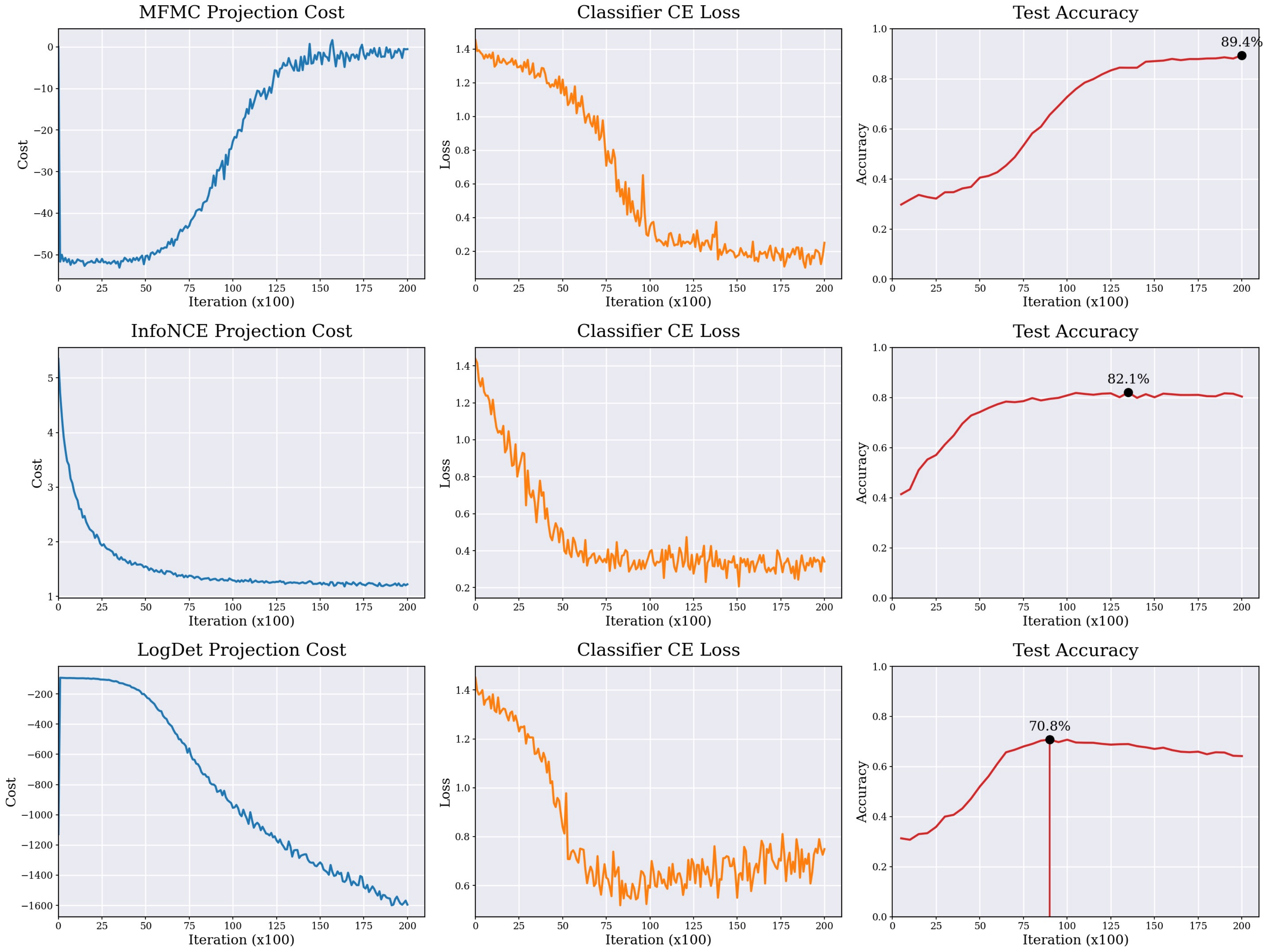}
    \caption{Representative CEAP-360VR subject-dependent stability comparison for MFMC, High-order InfoNCE, and FMCA-LogDet using the trimodal setting (EDA, BVP, and SKT). MFMC achieves the highest peak accuracy (89.4\%), while High-order InfoNCE and FMCA-LogDet peak at 82.1\% and 70.8\%, respectively. This figure is provided only as a qualitative single-fold learning-curve illustration; the official CEAP-360VR results are reported in Table III of the main text as mean $\pm$ standard deviation over five folds.}
    \label{fig:ceap_stability}
\end{figure}

\FloatBarrier

\subsection{Scalability to More Than Three Modalities on DEAP}
\label{sec:appendix_scalability}

To complement Theorem 2 in the main manuscript, we evaluate the practical extension of MFMC to larger modality sets on DEAP. The purpose of this supplementary experiment is to quantify how accuracy and memory change as the number of active physiological modalities increases, while keeping the core MFMC objective unchanged.

\subsubsection{Generalized objective and implementation}

Let \(e_i \in \mathbb{R}^{K}\) denote the embedding of modality \(X_i\). For each target modality \(X_i\), the remaining \(M-1\) modality embeddings are fused into a leave-one-out representation \(e_{[M]\setminus\{i\}} \in \mathbb{R}^{K}\). The full-sum \(M\)-modal objective is
\begin{equation}
\mathcal{L}^{(M)}_{\mathrm{full}} = \sum_{i=1}^{M} \mathrm{tr}\!\left(R_{[M]\setminus\{i\}}^{-1} P_{[M]\setminus\{i\},i} R_i^{-1} P_{[M]\setminus\{i\},i}^{\top}\right).
\end{equation}

For the 5-modality sampled variant, we sample a subset \(\mathcal{S} \subset \{1,\dots,5\}\) at each mini-batch and use the unbiased estimator
\begin{equation}
\mathcal{L}^{(5)}_{\mathrm{sample}} = \frac{5}{|\mathcal{S}|} \sum_{i \in \mathcal{S}} \mathrm{tr}\!\left(R_{[5]\setminus\{i\}}^{-1} P_{[5]\setminus\{i\},i} R_i^{-1} P_{[5]\setminus\{i\},i}^{\top}\right).
\end{equation}

In all settings, the modality encoders retain the same shared embedding dimension \(K = 128\). The leave-one-out fusion block follows the same MLP style as the trimodal implementation. Only the input dimension changes with the number of concatenated modality embeddings.

\subsubsection{Experimental protocol}

We consider four DEAP settings: (i) 3-modality full-sum MFMC using EEG, EOG, and SKT; (ii) 4-modality full-sum MFMC using EEG, EOG, SKT, and GSR; (iii) 5-modality full-sum MFMC using EEG, EOG, SKT, GSR, and respiration; and (iv) 5-modality sampled MFMC using the same five modalities. All settings use 5-fold cross-validation. We report mean best test accuracy \(\pm\) standard deviation across folds. We additionally record peak GPU memory, the number of leave-one-out terms evaluated per mini-batch, and the parameter count of the MFMC pretraining model. The sampled setting uses 2 leave-one-out terms per mini-batch.

\subsubsection{Complexity analysis}

Let \(B\) denote the mini-batch size and \(K\) the shared embedding dimension. For one leave-one-out term, forming the sample auto-covariance and cross-covariance matrices costs \(O(BK^2)\), and the matrix inverse or linear solve in the MFMC trace objective in the main manuscript costs \(O(K^3)\). Therefore, full-sum MFMC scales as \(O(M(BK^2 + K^3))\) with respect to the number of modalities \(M\). The dominant memory cost scales as \(O(MBK + MK^2)\), which accounts for modality embeddings and covariance matrices. For sampled MFMC, the expected per-mini-batch cost becomes \(O(|\mathcal{S}|(BK^2 + K^3))\), while the encoder-side memory footprint remains unchanged and the number of active leave-one-out terms is reduced.

\begin{figure}[t]
\centering
\includegraphics[width=0.95\columnwidth]{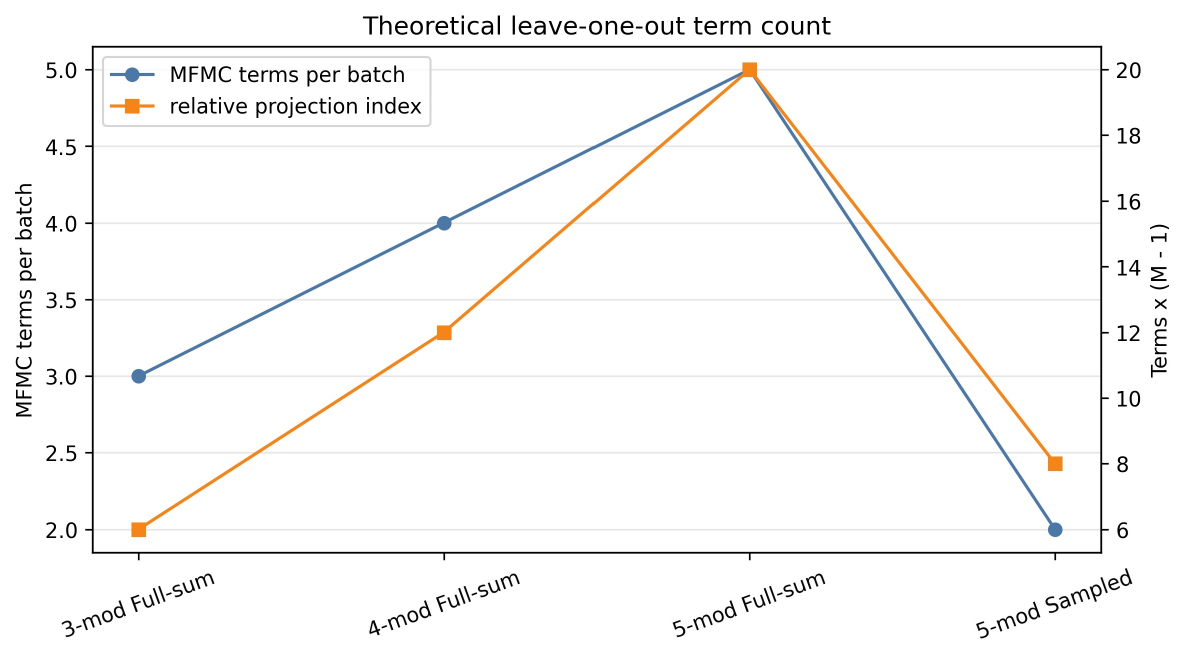}
\caption{Theoretical comparison of the number of leave-one-out terms and the relative projection-cost trend in the DEAP scalability study. Full-sum MFMC scales linearly with the number of modalities through \(M\) leave-one-out terms, whereas the sampled 5-modality variant evaluates only two terms per mini-batch.}
\label{fig:deap_scalability_theory}
\end{figure}

\subsubsection{Results}

\begin{table*}[t]
\caption{DEAP scalability comparison of MFMC across 3-, 4-, and 5-modality settings. Results are reported as mean best test accuracy \(\pm\) standard deviation across five folds. Peak GPU memory denotes the mean peak allocated GPU memory.}
\label{tab:deap_scalability}
\centering
\begin{tabular}{lccccc}
\toprule
Setting & Terms / batch & Best test acc. & Peak GPU mem. (MB) & Params (M) & Objective type \\
\midrule
3-mod Full-sum & 3 & 0.9803 \(\pm\) 0.0163 & 6321 & 27.53 & Full-sum \\
4-mod Full-sum & 4 & 0.9961 \(\pm\) 0.0016 & 6557 & 31.40 & Full-sum \\
5-mod Full-sum & 5 & 0.9925 \(\pm\) 0.0056 & 6796 & 35.41 & Full-sum \\
5-mod Sampled (\(|S|=2\)) & 2 & 0.9953 \(\pm\) 0.0038 & 6793 & 35.41 & Sampled \\
\bottomrule
\end{tabular}
\end{table*}

\begin{figure}[t]
\centering
\includegraphics[width=0.95\columnwidth]{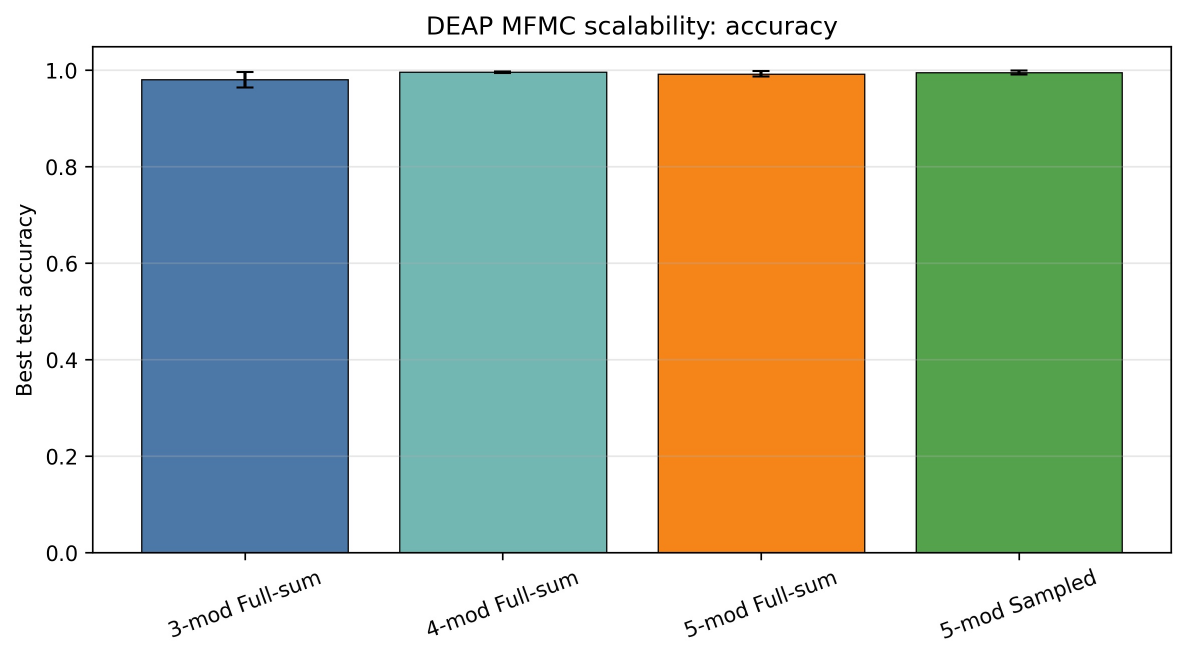}
\caption{Mean best test accuracy across five folds for the 3-, 4-, and 5-modality DEAP scalability settings. MFMC maintains strong performance across all settings, and the 5-modality sampled variant remains comparable to the corresponding full-sum setting.}
\label{fig:deap_scalability_acc}
\end{figure}

\begin{figure}[t]
\centering
\includegraphics[width=0.95\columnwidth]{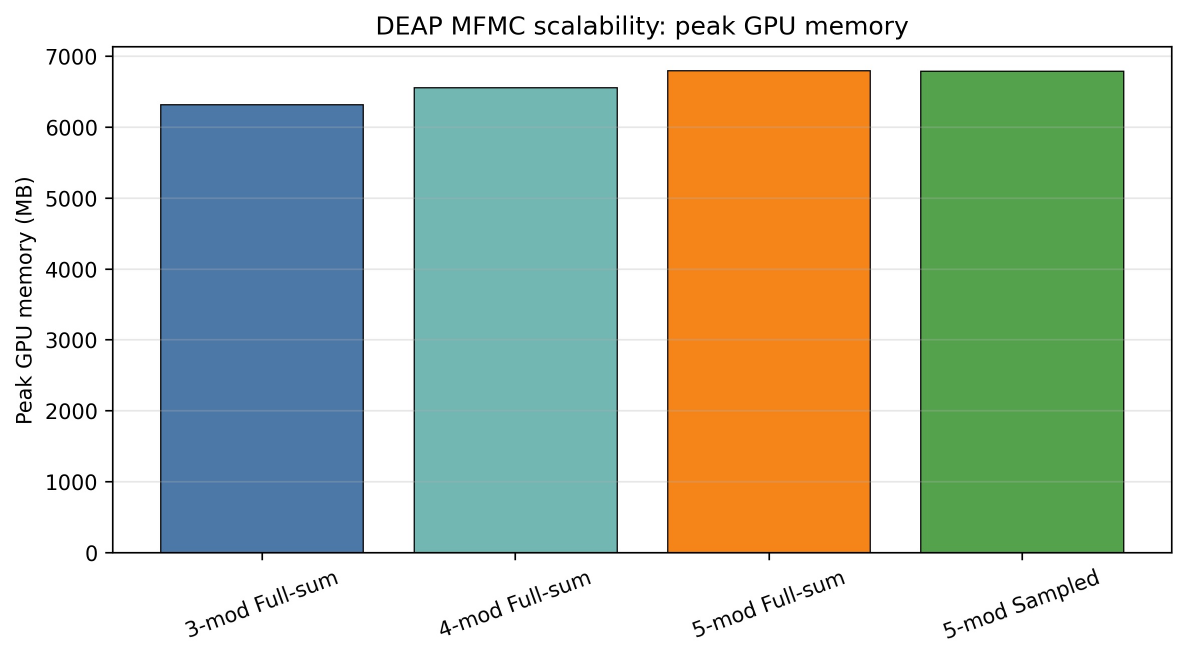}
\caption{Mean peak allocated GPU memory for the 3-, 4-, and 5-modality DEAP scalability settings. Peak memory increases gradually from the 3-modality to the 5-modality full-sum settings, while the 5-modality sampled variant shows a similar memory footprint to the 5-modality full-sum variant.}
\label{fig:deap_scalability_mem}
\end{figure}

Across the 3-, 4-, and 5-modality settings, MFMC maintains strong performance on DEAP, with mean best test accuracy ranging from 0.9803 \(\pm\) 0.0163 to 0.9961 \(\pm\) 0.0016 and 0.9925 \(\pm\) 0.0056. The measured peak GPU memory increases gradually from 6321 MB in the 3-modality full-sum setting to 6557 MB and 6796 MB in the 4- and 5-modality full-sum settings, respectively. The 5-modality sampled variant evaluates only two leave-one-out terms per mini-batch and achieves comparable accuracy at 0.9953 \(\pm\) 0.0038, while using a similar peak GPU memory footprint of 6793 MB. These results support the practical extensibility of MFMC beyond three modalities and show that the generalized leave-one-out formulation remains effective up to five modalities on DEAP.

\subsection{Ridge Coefficient Sensitivity in Trace-Based FMCA}
\label{app:epsilon_sensitivity}

To clarify the numerical-stability handling of the trace-based FMCA objective in Eq. (22), we conducted a single-split diagnostic experiment on DEAP. This experiment is intended to examine the effect of the ridge coefficient \(\varepsilon\) on matrix conditioning and downstream performance, rather than to replace the official five-fold cross-validation results. We used one fixed subject-dependent 80/20 split, the EEG/EOG/SKT trimodal setting, batch size 200, and a training budget of 10,000 iterations. All settings were kept fixed except for \(\varepsilon\), which was varied over
\[
\{0,10^{-6},10^{-5},10^{-4},10^{-3},10^{-2}\}.
\]
For each value, we evaluated the downstream EEG encoder using accuracy and macro-F1, and recorded the condition numbers of the ridge-regularized autocorrelation matrices.

\begin{table}[t]
\centering
\caption{Sensitivity of trace-based FMCA to the ridge coefficient \(\varepsilon\) on a fixed DEAP subject-dependent split. This experiment uses EEG/EOG/SKT, batch size 200, and 10,000 training iterations. It is intended as a numerical-stability diagnostic rather than a replacement for the official five-fold results.}
\label{tab:epsilon_sensitivity}
\begin{tabular}{c c c c c}
\hline
\(\varepsilon\) & Best Acc. & Best Macro-F1 & Mean Cond. & Max Cond. \\
\hline
0 & 0.9861 & 0.9862 & \(1.03{\times}10^6\) & \(9.17{\times}10^6\) \\
\(10^{-6}\) & 0.9819 & 0.9818 & \(8.62{\times}10^5\) & \(6.57{\times}10^6\) \\
\(10^{-5}\) & 0.9831 & 0.9834 & \(7.19{\times}10^5\) & \(5.08{\times}10^6\) \\
\(10^{-4}\) & 0.9814 & 0.9814 & \(2.56{\times}10^5\) & \(9.50{\times}10^5\) \\
\(10^{-3}\) & 0.9826 & 0.9826 & \(4.22{\times}10^4\) & \(9.57{\times}10^4\) \\
\(10^{-2}\) & 0.9896 & 0.9898 & \(4.38{\times}10^3\) & \(9.19{\times}10^3\) \\
\hline
\end{tabular}
\end{table}

All tested settings completed without NaN, Inf, or divergence. The best accuracy remains within \(0.9814\)--\(0.9896\), and the best macro-F1 remains within \(0.9814\)--\(0.9898\), indicating that \(\varepsilon\) is not a performance-critical hyperparameter in this diagnostic setting. In contrast, the conditioning of the autocorrelation matrices improves substantially as \(\varepsilon\) increases. Compared with no ridge regularization, \(\varepsilon=10^{-6}\) reduces the maximum condition number from \(9.17{\times}10^6\) to \(6.57{\times}10^6\), and larger values further reduce the condition number. We therefore use \(\varepsilon=10^{-6}\) in the main experiments as a small numerical-stabilization coefficient that minimally perturbs the trace objective, while the sensitivity analysis confirms that the downstream performance is robust over a broad tested range.

\section{GitHub Demo Details for Reproducibility}

For reproducibility, we provide the demonstration code via the GitHub link \\
\url{https://github.com/DY9910/MFMC}.

\subsection{Purpose \& Quick-Start}
This appendix provides three executable notebooks: \texttt{MFMC\_DEAP\_main1\_subject\_dep.ipynb} for the subject-dependent DEAP experiment, \texttt{MFMC\_DEAP\_main2\_subject\_indep.ipynb} for the subject-independent DEAP experiment, and \texttt{DEAP\_MFMC\_scalability\_experiment.ipynb} for the supplementary 3-, 4-, and 5-modality scalability study reported in Appendix~\ref{sec:appendix_scalability}.

We additionally provide \texttt{Supplement/DEAP\_signal\_selection/signal\_selection.ipynb}, a DEAP-only notebook that implements the learnable modality-attention selection protocol. It loads all DEAP modality groups, trains the softmax-normalized modality gates, and displays the learned modality ranking and top 3 decision inside the notebook. The same gate definition, optimization objective, and top 3 rule are used for CEAP-360VR and MAHNOB-HCI, with dataset specific candidate modality groups.

\subsubsection{Why these notebooks?}
We provide \emph{one} notebook per experiment %
\texttt{MFMC\_DEAP\_main1\_subject\_dep.ipynb} for the \textbf{subject-dependent} split and
\texttt{MFMC\_DEAP\_main2\_subject\_indep.ipynb} for the \textbf{subject-independent} split.
Each notebook contains the full \emph{five-fold} cross-validation loop: the code trains
folds 1-5 sequentially, writes every fold's metrics to
\texttt{results/fold\_k/}, then the last code cell loads the temporary files and creates visualisations.
Reviewers can run each notebook top-to-bottom to reproduce the complete experiment in one session.
The scalability notebook runs the supplementary DEAP experiment for 3-, 4-, and 5-modality MFMC settings, exports per-fold and aggregate CSV summaries, and saves comparison figures to the \texttt{scalability\_experiment/} folder.

2) Prerequisite and preprocessing: The notebooks expect NumPy tensors produced from the official \texttt{data\_preprocessed\_python.zip} release. Run the following once:\\
\textbullet\ Unzip raw data: Ensure that \texttt{s01.dat}\,\textendash\,\texttt{s32.dat} are accessible through \texttt{DEAP\_RAW\_DATA\_DIR} (default: \texttt{<BASE\_PATH>/DEAP\_data/}).\\
\textbullet\ Activate the environment: (see next subsection) and run \texttt{python DEAP\_Preprocess.py}. The script\\
(i) skips the first 3 s of each trial, extracts 10 s (1,280-sample) windows with a 0.4 s stride,\\
(ii) computes fixed normalization constants for EEG, EOG, and temperature from the first available DEAP window and rescales these three legacy modalities with those constants,\\
(iii) discards a window if any normalized EEG, EOG, or temperature sample exceeds $\pm 5$,\\
(iv) saves \texttt{eeg\_data.npy}, \texttt{eog\_data.npy}, \texttt{temp\_data.npy}, \texttt{emotion\_labels.npy}, \texttt{subject.npy}, \texttt{valence.npy}, and \texttt{arousal.npy} to the processed-data folder, creates \texttt{skt\_data.npy} as an alias of \texttt{temp\_data.npy}, and can additionally generate \texttt{gsr\_data.npy} and \texttt{resp\_data.npy} for the DEAP scalability notebook; the added GSR and respiration modalities are normalized per window with z-score normalization and clipped to $\pm 8$.

\subsubsection{Software setup}
Create a Conda environment with Python 3.9 (any \(\ge\)3.8 works) and install:

\begin{itemize}[nosep,leftmargin=*]
  \item \textbf{PyTorch stack} (CUDA 11.3): \texttt{torch 1.12.1},
        \texttt{torchvision 0.13.1}.
  \item \textbf{Core scientific libraries}: \texttt{numpy 1.21.6},
        \texttt{scipy 1.9.3}, \texttt{scikit-learn 1.1.3},
        \texttt{matplotlib 3.5.3}, \texttt{joblib 1.2.0}.
\end{itemize}

\subsection{Data and Model Pipeline}

\subsubsection{Data flow}
The two main DEAP notebooks use five NumPy tensors generated in preprocessing: \texttt{eeg\_data.npy} \((N, 32, 1280)\), \texttt{eog\_data.npy} \((N, 2, 1280)\), \texttt{temp\_data.npy} \((N, 1, 1280)\), \texttt{emotion\_labels.npy} \((N)\), and \texttt{subject.npy} \((N)\), all stored in \texttt{<BASE\_PATH>/Data\_processed/}. The scalability notebook uses the same retained windows and may additionally load \texttt{gsr\_data.npy} and \texttt{resp\_data.npy} for the 4- and 5-modality settings. Here \(N = 20{,}097\), reflecting 10 s windows extracted after skipping the initial 3 s and advancing with a 0.4 s stride. A pointer table maps each window to its subject and trial, enabling a subject-balanced split generator that assembles five folds (approximately 4 subjects per fold). The \texttt{DataLoader} reads these preprocessed arrays directly from disk. EEG, EOG, and temperature use fixed legacy normalization constants computed once from the first available DEAP window; GSR and respiration, when used in the scalability notebook, are per-window z-score normalized and clipped to $\pm 8$ during preprocessing. No further normalization or augmentation is applied in the \texttt{DataLoader}. Mini-batches therefore contain three tensors (EEG, EOG, temperature) plus one label vector in the main notebooks, with GSR and respiration optionally added in the scalability setting.

\subsubsection{Model architecture}
Each modality encoder has two stages: a temporal network that processes each sensor channel separately, followed by a channel network that fuses the channels into a 128-D vector.

\begin{itemize}
  \item \textbf{Temporal network}  
        Shared four-block Conv \(\rightarrow\) Batch Norm \(\rightarrow\) ReLU \(\rightarrow\) Max Pool backbone with filter widths 1 → 32 → 64 → 128 → 256, kernel size 11, pooling factor 4 at each block. Three fully connected layers reduce the output to 128 units. The same weights are reused for EEG, EOG and temperature channels because batch and channel dimensions are flattened before the first convolution.
  \item \textbf{Channel network}  
        The 128-D vectors are reshaped to \((\text{batch},\text{channels},128)\), flattened across channels, and passed through a one-hidden-layer MLP \((128\times C)\rightarrow4000\rightarrow128\) where \(C\) is 32, 2, or 1 for EEG, EOG, and temperature.
  \item \textbf{MFMC projection heads}  
        Three identical 2-layer MLPs \(128\rightarrow512\rightarrow128\) with ReLU and Batch Norm project pairs of modality embeddings into the common space used by the trace-correlation loss.
  \item \textbf{Classifier}  
        During supervised training the EEG encoder output feeds a four-layer MLP \(128\rightarrow256\rightarrow512\rightarrow256\rightarrow4\) with Batch Norm and ReLU after the first three layers; the final linear layer outputs logits for the four valence-arousal quadrants.
\end{itemize}

The model produces one 128-D feature per 10 s window for both MFMC and emotion classification. The total parameter count depends on the number of active modalities and projection heads; the exact counts used in the DEAP scalability study are reported in Table~\ref{tab:deap_scalability}.

\subsection{Training Protocol and Results}

\subsubsection{Loop schedule}
Each fold is trained for \textbf{20 000 iterations}.  
Every iteration runs two optimisers back-to-back:


\begin{itemize}
  \item \textbf{Unsupervised step:} Draw a mini-batch (\(B=256\)) of EEG, EOG, and temperature windows,
        compute the MFMC loss, and update \texttt{all\_feature\_params} with Adam
        (lr \(=3\times10^{-4}\), \(\beta_1=0.5\), \(\beta_2=0.9\)).
  \item \textbf{Supervised step:} Draw a fresh EEG batch, obtain 128-D features from the frozen EEG encoder,
        and update the classifier with cross-entropy using the same Adam settings.
\end{itemize}

No scheduler, early stopping, or augmentation is used; training always reaches iteration 20 000.

\subsubsection{Cross-validation outputs}
Five subject-balanced folds are trained sequentially.  
Each fold produces a JSON log and a loss-accuracy curve figure in \texttt{results/fold\_k/}.  
The final notebook cell merges these files into \texttt{cv\_summary.csv}, a combined learning-curve plot, and a confusion-matrix image.

\subsubsection{Accuracy}
Subject-dependent folds achieve
\(\mathbf{98.3\,\% \pm 1.4}\) best accuracy and
\(96.4\,\% \pm 1.8\) final accuracy.
Subject-independent folds reach
\(\mathbf{34.6\,\% \pm 3.7}\) best accuracy and
\(26.9\,\% \pm 3.2\) final accuracy.
Most remaining errors involve high- versus low-arousal quadrants; temperature contributes the least discriminative power.

\subsubsection{Runtime}
One fold trains in roughly one hour on a single NVIDIA L40s GPU (48 GB) with 16 logical CPU cores, so the full five-fold run finishes in about five hours wall-clock.

\section{LEARNED HIGHER ORDER DEPENDENCE AND CROSS SUBJECT DEGRADATION ANALYSIS}

To better understand what higher-order dependencies are learned by MFMC, we conduct a diagnostic dependence analysis on DEAP and CEAP-360VR using representative 80/20 train-test splits following the subject-dependent and subject-independent protocols. This analysis is intended to interpret the learned MFMC representations, rather than to provide an additional performance benchmark.

We first compute the normalized contribution of the three cyclic pair-to-third trace terms in Eq. (23). Let \(\ell_{12\rightarrow3}\), \(\ell_{13\rightarrow2}\), and \(\ell_{23\rightarrow1}\) denote the trace-based dependence scores corresponding to \((X_1,X_2)\rightarrow X_3\), \((X_1,X_3)\rightarrow X_2\), and \((X_2,X_3)\rightarrow X_1\), respectively. The normalized contribution is defined as

\[
c_{ij\rightarrow k}
=
\frac{\ell_{ij\rightarrow k}}
{\ell_{12\rightarrow3}+\ell_{13\rightarrow2}+\ell_{23\rightarrow1}}.
\]

This quantity indicates whether the cyclic MFMC objective is dominated by a single pair-to-third direction or distributed across multiple cyclic interactions.

Because a balanced cyclic contribution profile does not by itself indicate whether a fused pair representation captures additional joint dependence beyond individual modalities, we further compute a pair-fusion gain. For each direction \((X_i,X_j)\rightarrow X_k\), we compare the fused pair-to-third dependence score with the stronger of the two corresponding single-modality-to-third dependence scores:

\[
g_{ij\rightarrow k}
=
\frac{
\ell_{ij\rightarrow k}
-
\max(\ell_{i\rightarrow k},\ell_{j\rightarrow k})
}{
\max(\ell_{i\rightarrow k},\ell_{j\rightarrow k})+\delta
},
\]

where \(\delta=10^{-12}\) is used only for numerical stability. This gain is an empirical diagnostic of joint-over-single dependence and should not be interpreted as a strict information-theoretic synergy measure.

\begin{table*}[t]
\centering
\caption{Dependence contribution and pair-fusion gain analysis on DEAP and CEAP-360VR using representative 80/20 splits. The contribution columns report normalized cyclic trace contributions. The gain columns report pair-fusion gain over the stronger corresponding single-modality dependence baseline.}
\label{tab:dependence_contribution_gain}
\begin{tabular}{llcccccc}
\hline
Dataset & Protocol &
\(c_{12\rightarrow3}\) &
\(c_{13\rightarrow2}\) &
\(c_{23\rightarrow1}\) &
\(g_{12\rightarrow3}\) &
\(g_{13\rightarrow2}\) &
\(g_{23\rightarrow1}\) \\
\hline
DEAP & Subject-dependent & 0.323 & 0.338 & 0.339 & 173.5\% & 11.5\% & 11.5\% \\
DEAP & Subject-independent & 0.239 & 0.385 & 0.376 & -3.3\% & 0.9\% & -1.3\% \\
CEAP-360VR & Subject-dependent & 0.326 & 0.343 & 0.331 & 105.6\% & 116.4\% & 126.7\% \\
CEAP-360VR & Subject-independent & 0.324 & 0.360 & 0.316 & -0.3\% & 7.0\% & -6.0\% \\
\hline
\end{tabular}
\end{table*}

For DEAP, \(X_1\), \(X_2\), and \(X_3\) correspond to EEG, EOG, and skin temperature (SKT), respectively. For CEAP-360VR, they correspond to EDA, BVP, and SKT, respectively.

Table~\ref{tab:dependence_contribution_gain} provides a quantitative explanation for the performance degradation observed under subject independent evaluation. In the subject dependent setting, all pair fusion gains are positive. On DEAP, the gains are 173.5\%, 11.5\%, and 11.5\%, while on CEAP-360VR all three gains exceed 100\%. This indicates that the fused pair representations capture joint dependence beyond the stronger corresponding single modality. In contrast, under subject independent evaluation, the gains become much smaller or partially negative. On DEAP, they change to negative 3.3\%, 0.9\%, and negative 1.3\%; on CEAP-360VR, they become negative 0.3\%, 7.0\%, and negative 6.0\%. These results suggest that cross subject generalization weakens the transferable joint over single dependence, which explains why subject independent accuracy drops even though MFMC remains competitive with or better than the baselines.

The cyclic contribution results show that MFMC does not concentrate its DTC-grounded objective on a single pair-to-third interaction. In the subject-dependent setting, both DEAP and CEAP-360VR exhibit relatively balanced cyclic contribution profiles. More importantly, all three pair-fusion gains are positive. On DEAP, the largest gain is observed for \((\mathrm{EEG},\mathrm{EOG})\rightarrow \mathrm{SKT}\), indicating that the joint central-ocular representation has substantially stronger dependence with thermal dynamics than either EEG or EOG alone. On CEAP-360VR, all three peripheral pair-fusion gains exceed 100\%, suggesting strong joint-over-single dependence among EDA, BVP, and SKT.

In the subject-independent setting, the learned dependence becomes more direction-specific and partially redundant. On DEAP, the two terms involving SKT as part of the fused pair, \((\mathrm{EEG},\mathrm{SKT})\rightarrow \mathrm{EOG}\) and \((\mathrm{EOG},\mathrm{SKT})\rightarrow \mathrm{EEG}\), contribute more than \((\mathrm{EEG},\mathrm{EOG})\rightarrow \mathrm{SKT}\), but the pair-fusion gains are modest or slightly negative. On CEAP-360VR, \((\mathrm{EDA},\mathrm{SKT})\rightarrow \mathrm{BVP}\) is the only direction with a clearly positive pair-fusion gain. These results suggest that MFMC learns distributed higher-order dependencies in subject-dependent settings, while cross-subject generalization emphasizes more direction-specific and partially redundant physiological coupling.

\begin{figure*}[t]
\centering
\begin{minipage}[t]{0.48\linewidth}
\centering
\includegraphics[width=\linewidth]{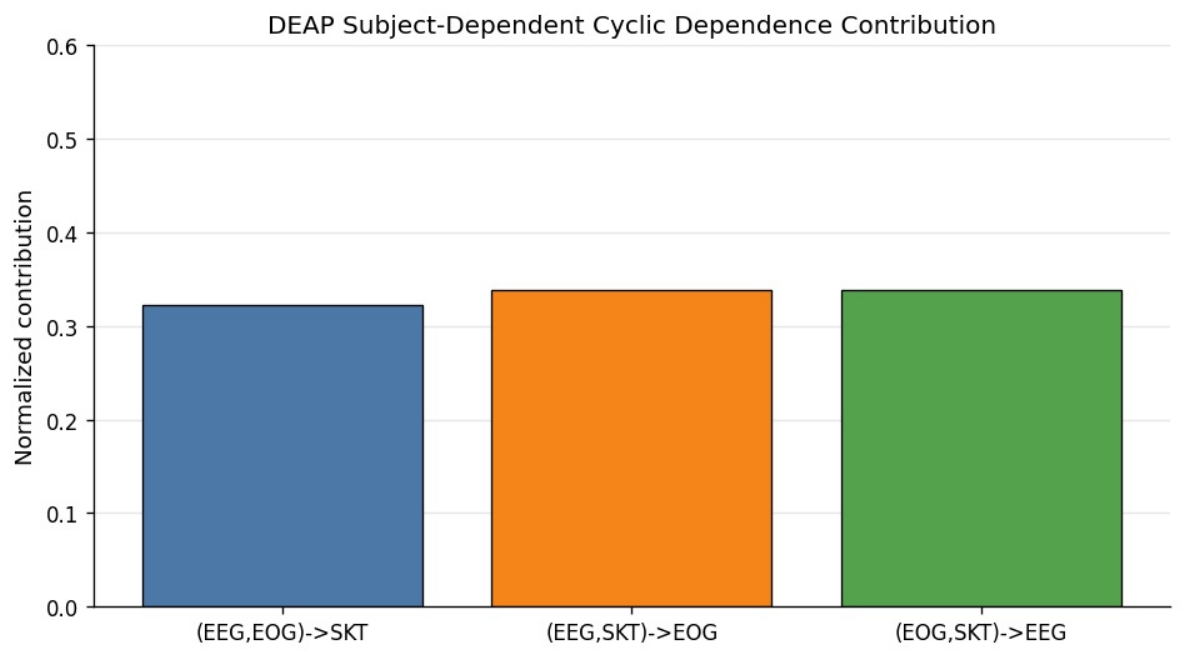}
\\[2pt]
(a) DEAP subject-dependent
\end{minipage}
\hfill
\begin{minipage}[t]{0.48\linewidth}
\centering
\includegraphics[width=\linewidth]{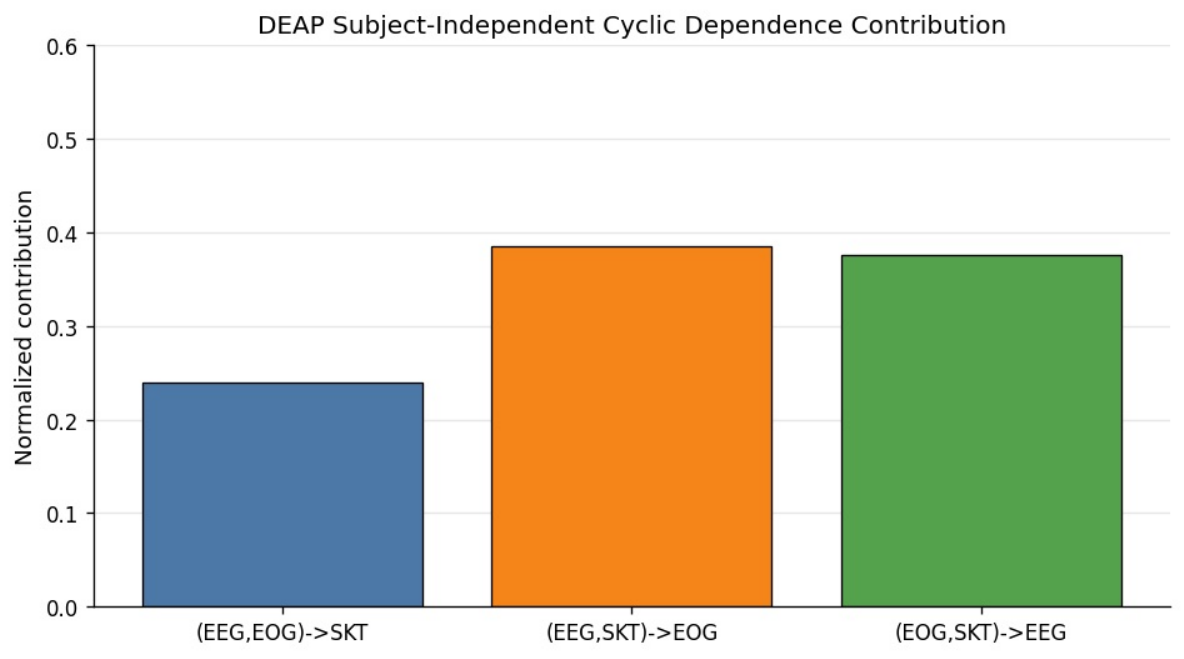}
\\[2pt]
(b) DEAP subject-independent
\end{minipage}

\vspace{0.75em}

\begin{minipage}[t]{0.48\linewidth}
\centering
\includegraphics[width=\linewidth]{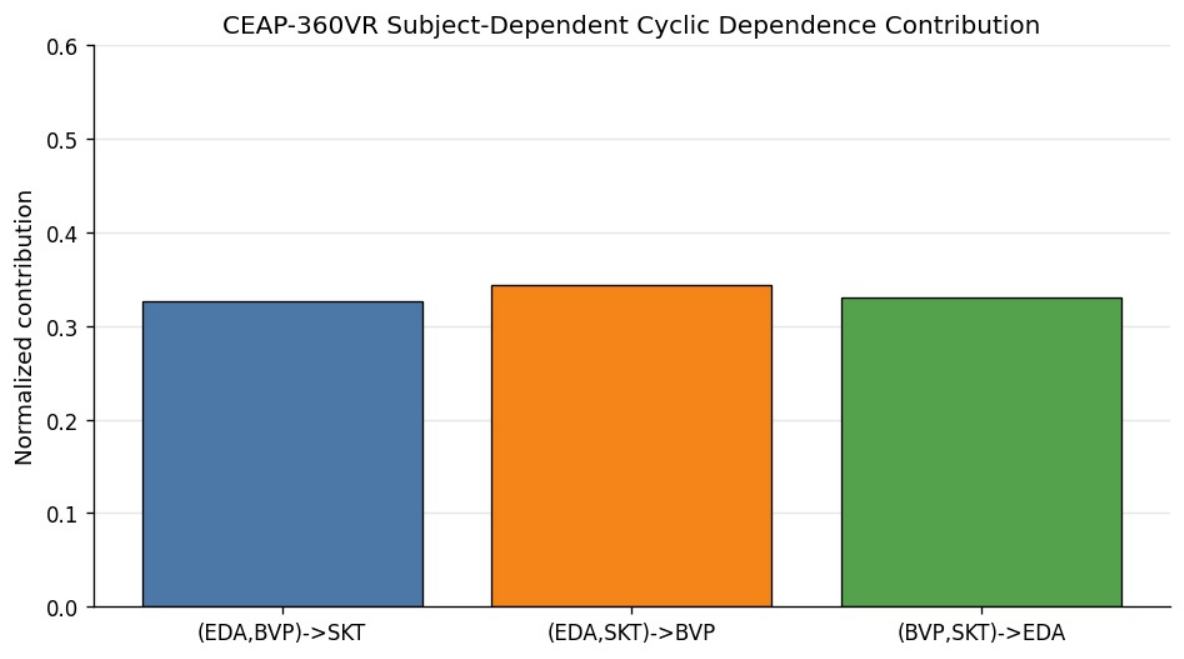}
\\[2pt]
(c) CEAP-360VR subject-dependent
\end{minipage}
\hfill
\begin{minipage}[t]{0.48\linewidth}
\centering
\includegraphics[width=\linewidth]{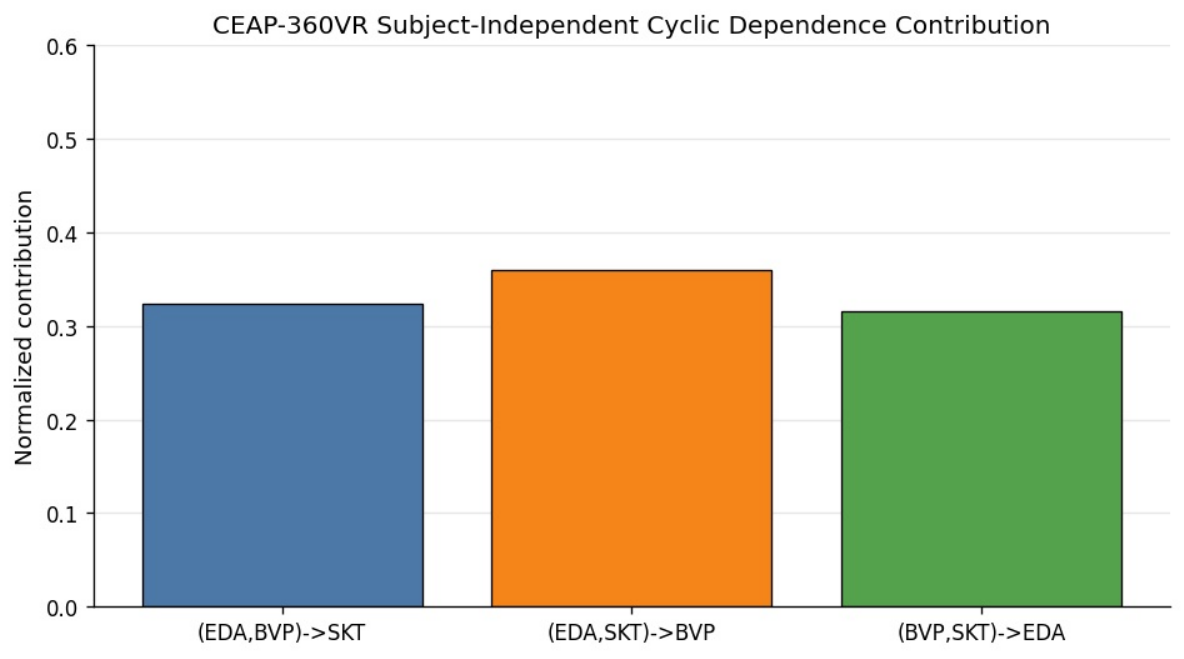}
\\[2pt]
(d) CEAP-360VR subject-independent
\end{minipage}
\caption{Normalized cyclic dependence contributions of MFMC on DEAP and CEAP-360VR. The three bars in each panel correspond to the three pair-to-third terms in the cyclic DTC-grounded objective.}
\label{fig:higher_order_contribution_bars}
\end{figure*}

\begin{figure*}[t]
\centering
\begin{minipage}[t]{0.48\linewidth}
\centering
\includegraphics[width=\linewidth]{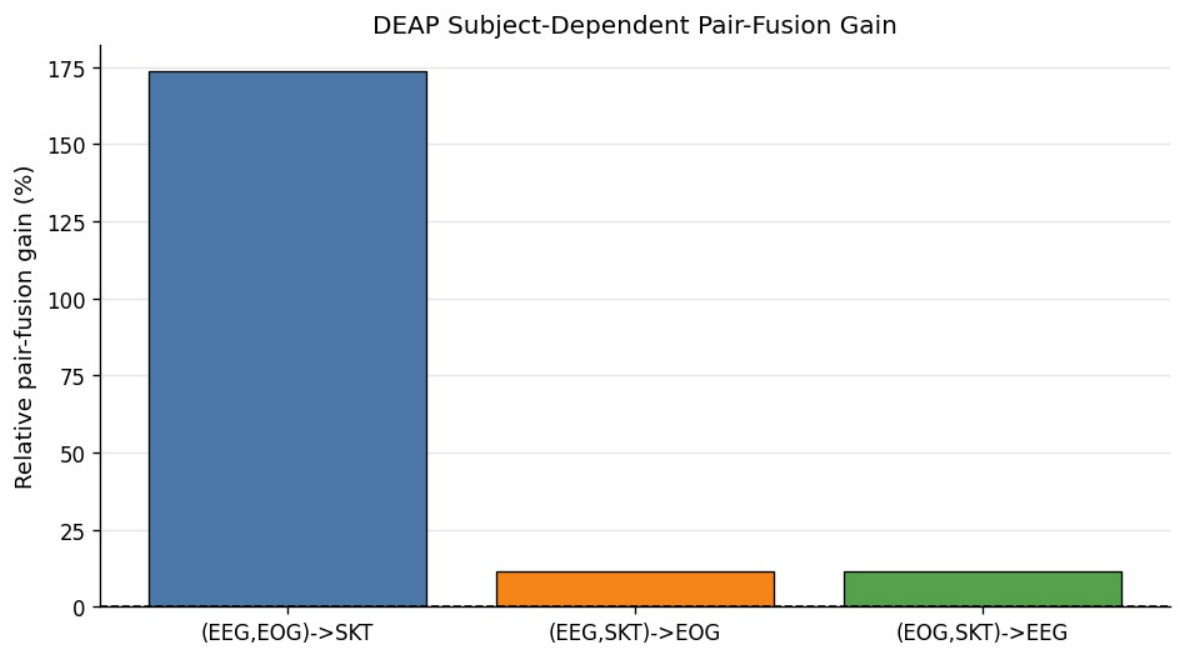}
\\[2pt]
(a) DEAP subject-dependent
\end{minipage}
\hfill
\begin{minipage}[t]{0.48\linewidth}
\centering
\includegraphics[width=\linewidth]{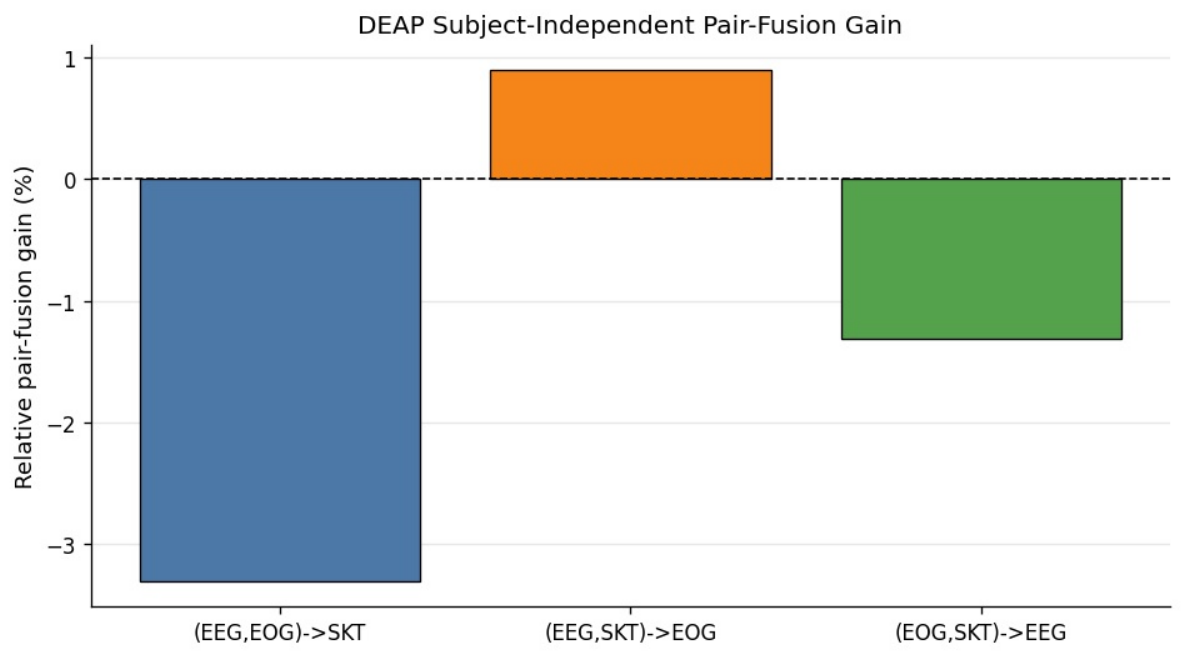}
\\[2pt]
(b) DEAP subject-independent
\end{minipage}

\vspace{0.75em}

\begin{minipage}[t]{0.48\linewidth}
\centering
\includegraphics[width=\linewidth]{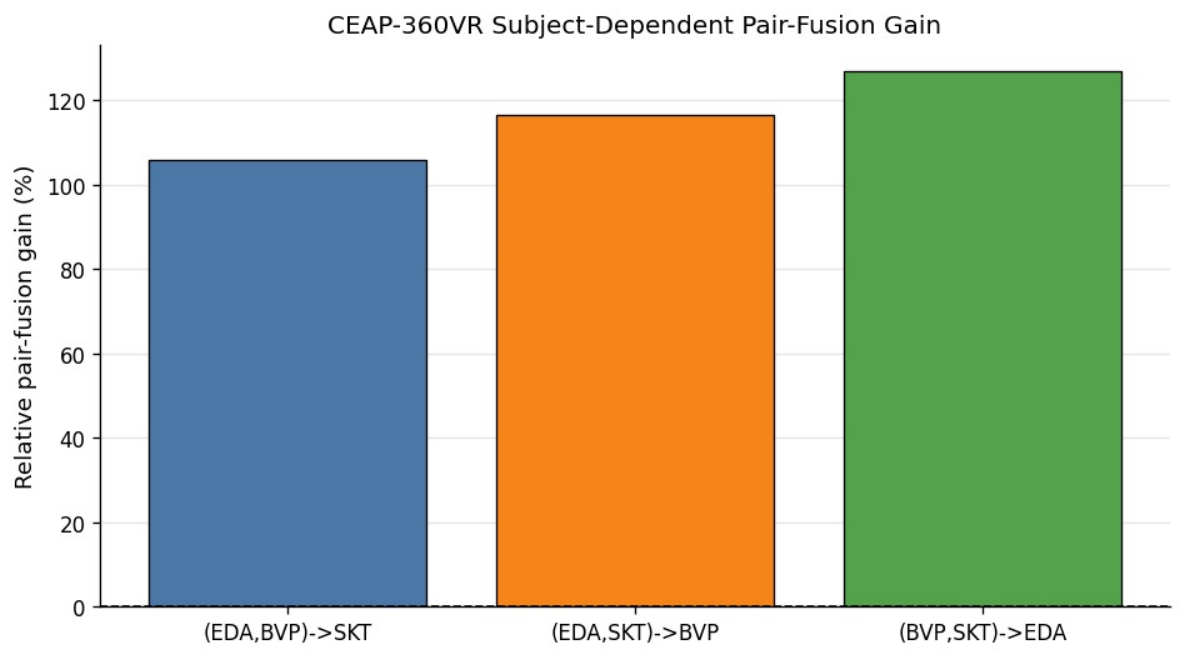}
\\[2pt]
(c) CEAP-360VR subject-dependent
\end{minipage}
\hfill
\begin{minipage}[t]{0.48\linewidth}
\centering
\includegraphics[width=\linewidth]{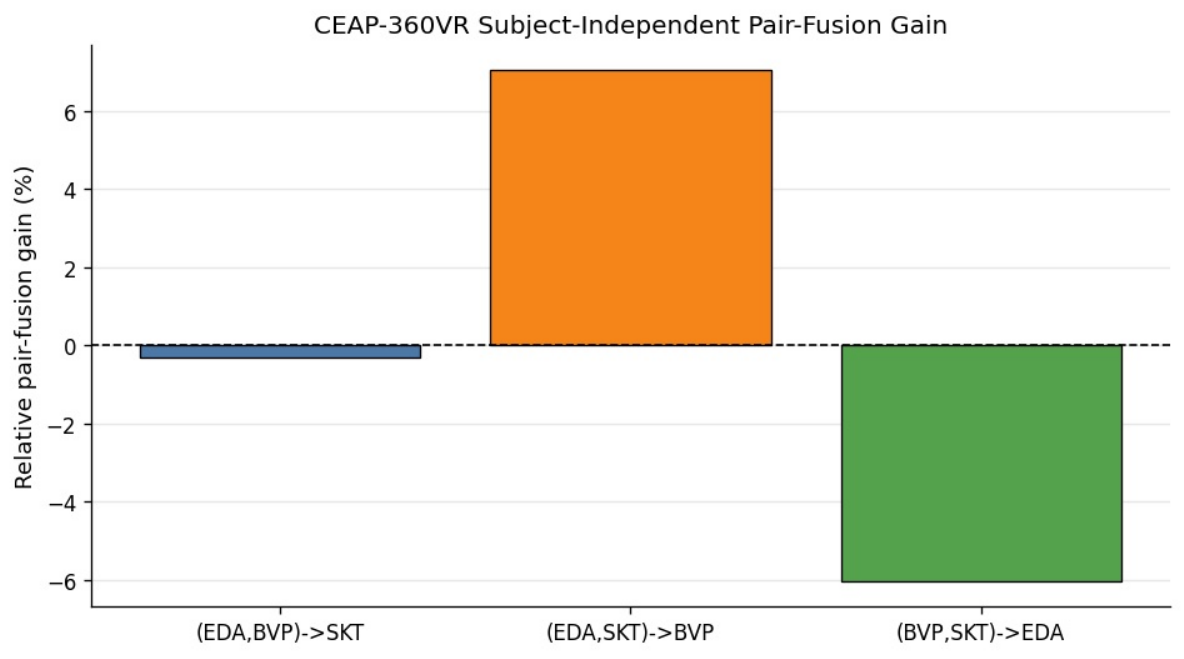}
\\[2pt]
(d) CEAP-360VR subject-independent
\end{minipage}
\caption{Pair-fusion gain over the strongest single-modality dependence baseline. Positive values indicate that the fused pair representation has stronger dependence with the third modality than either individual modality alone.}
\label{fig:pair_fusion_gain_bars}
\end{figure*}

\section{CONTROLLED DTC ABLATION ON DEAP}

To directly examine the contribution of the tri modal DTC grounded objective, we conduct a controlled diagnostic ablation on DEAP using one fixed subject dependent 80\% training and 20\% testing split. This experiment is designed to isolate the role of the cyclic pair to third dependence structure rather than to replace the official five fold evaluation.

We compare three objectives. MFMC DTC is the original synchronized cyclic pair to third objective. Pairwise FMCA uses the same three modalities and the same trace based FMCA estimator, but maximizes only the sum of pairwise dependence terms. Shuffled DTC keeps the original pair to third architecture and fusion heads, but shifts the target embeddings within each mini batch, thereby destroying sample level synchronized tri modal correspondence.

As shown in Table~\ref{tab:controlled_dtc_ablation} and Fig.~\ref{fig:deap_dtc_ablation}, MFMC DTC achieves the highest best accuracy of 0.9870, outperforming Pairwise FMCA by 1.06 percentage points. This indicates that the cyclic pair to third DTC structure provides additional benefit beyond pairwise dependence, even when the estimator and input modalities are kept the same. More importantly, Shuffled DTC drops sharply to 0.3545, although it keeps the same pair to third architecture. This confirms that the gain is not caused merely by additional fusion parameters or by using three modalities, but relies on synchronized tri modal dependence.

\begin{table}[t]
\centering
\caption{Controlled DTC ablation on DEAP using one fixed subject dependent 80\% training and 20\% testing split.}
\label{tab:controlled_dtc_ablation}
\begin{tabular}{p{0.19\linewidth}p{0.55\linewidth}c}
\toprule
Objective & Description & Best accuracy \\
\midrule
MFMC DTC & Synchronized cyclic pair to third DTC grounded objective & 0.9870 \\
Pairwise FMCA & Same three modalities and same FMCA estimator, pairwise dependence only & 0.9764 \\
Shuffled DTC & Same pair to third architecture, target shifted within mini batch & 0.3545 \\
\bottomrule
\end{tabular}

\vspace{2pt}
This experiment is a controlled diagnostic ablation and is not intended to replace the official five fold cross validation results in the main manuscript.
\end{table}

\begin{figure}[t]
    \centering
    \includegraphics[width=\linewidth]{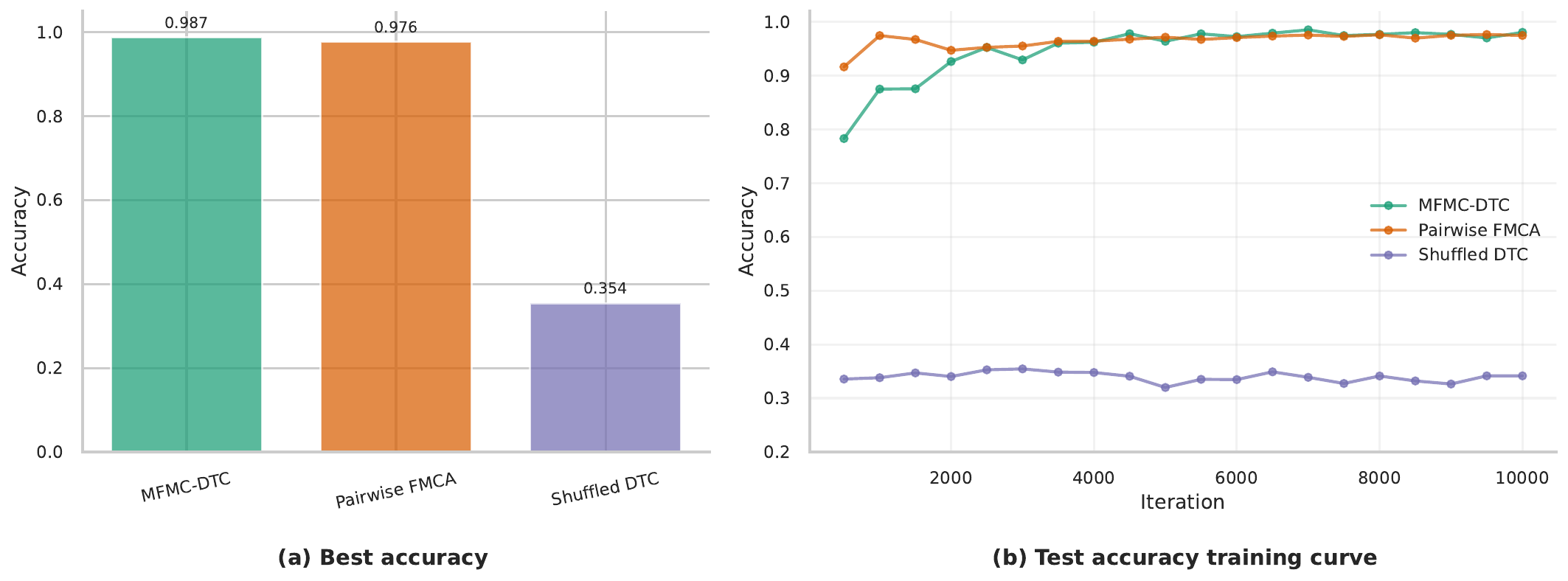}
    \caption{Controlled DTC ablation on DEAP using one fixed subject dependent 80\% training and 20\% testing split. (a) Best test accuracy. (b) Test accuracy curves during training. Pairwise FMCA uses the same three modalities and the same trace based FMCA estimator as MFMC, but removes the cyclic pair to third DTC structure and maximizes only pairwise dependence. Shuffled DTC keeps the original pair to third architecture and fusion heads, but shifts the target embeddings within each mini batch to destroy synchronized tri modal correspondence. This experiment is intended as a diagnostic ablation and does not replace the official five fold cross validation results.}
    \label{fig:deap_dtc_ablation}
\end{figure}

\clearpage

\bibliographystyle{IEEEtran}
\bibliography{IEEEabrv,main}

\end{document}